\documentclass[journal,10pt]{IEEEtran}

\usepackage{tabularx} 

\usepackage{cite}

\usepackage{subcaption} 
\usepackage[export]{adjustbox} 

\usepackage[utf8]{inputenc}
\usepackage{times}  
\usepackage{graphicx}

\usepackage[T1]{fontenc}
\usepackage{float}
\usepackage{amsmath}
\usepackage{amsthm}
\usepackage{amssymb}
\usepackage{stackrel}
\usepackage{setspace}
\usepackage{url}
\usepackage{bm}
\usepackage{caption}
\usepackage{xcolor}
\usepackage{cases} 
\usepackage{algorithm}
\usepackage{algpseudocode}
\usepackage{longtable} 
\usepackage{array}

\makeatletter

\floatstyle{ruled}
\newfloat{algorithm}{tbp}{loa}
\providecommand{\algorithmname}{Algorithm}
\floatname{algorithm}{\protect\algorithmname}

\theoremstyle{plain}

\theoremstyle{definition}

\theoremstyle{plain}

\theoremstyle{definition}

\theoremstyle{plain}

\newtheorem{remark}{Remark}

\IEEEoverridecommandlockouts
\usepackage{ifpdf}

\ifCLASSOPTIONcompsoc
\usepackage[caption=false,font=normalsize,labelfont=sf,textfont=sf]{subfig}
\else
\fi

\usepackage{mathtools}
\usepackage{booktabs}
\usepackage{multirow}
\usepackage{setspace}
\usepackage{lettrine}
\usepackage{array}
\@ifundefined{showcaptionsetup}{}{
	\PassOptionsToPackage{caption=false}{subfig}}
\usepackage{subfig}
\makeatother
\allowdisplaybreaks[4]
\usepackage{diagbox}
\usepackage{babel}
\usepackage{pifont}

\newtheorem{lemma}{Lemma}

\begin{document}
\captionsetup[figure]{font={small}, name={Fig.}, labelsep=period}
	
\title{
Joint Routing and Model Pruning for Decentralized Federated Learning in Bandwidth-Constrained Multi-Hop Wireless Networks}

\author{Xiaoyu He, Weicai Li, \textit{Graduate Student Member, IEEE}, Tiejun Lv, \textit{Senior Member, IEEE}, and  Xi Yu

 \thanks{Manuscript received 13 September 2025; revised 12 January 2026, and 26 February 2026; accepted 14 March 2026. This paper was supported in part by the National Natural Science Foundation of China under No. 62271068. (\emph{corresponding author: Tiejun Lv}.)}

\thanks{X. He, T. Lv, and X. Yu are with the School of Information and Communication Engineering, Beijing University of Posts and Telecommunications, Beijing 100876, China.  (e-mail: \{xiaoyuhhh, lvtiejun, yusy\}@bupt.edu.cn). 

W. Li is with the Center for Target Cognition Information Processing Science and Technology, and the Key Laboratory of Modern Measurement and Control Technology, Ministry of Education, both at Beijing Information Science and Technology University, Beijing, China (e-mail: liweicai@bupt.edu.cn).}}
\maketitle
\begin{abstract}
Decentralized federated learning (D-FL) enables privacy-preserving training without a central server, but multi-hop model exchanges and aggregation are often bottlenecked by communication resource constraints.
To address this issue, we propose a joint routing-and-pruning framework that optimizes routing paths and pruning rates to keep communication latency within prescribed limits.
We analyze how the sum of model biases across all clients affects the convergence bound of D-FL, and formulate an optimization that maximizes the model retention rate to reduce these biases under communication constraints. Further analysis shows that each client’s model retention rate is path-dependent, which reduces the original problem to a routing optimization. Leveraging this insight, we develop a routing algorithm that selects latency-efficient transmission paths, allowing more parameters to be delivered within the time budget and thereby improving D-FL convergence.
Simulations show that, compared with unpruned systems, the proposed framework reduces average transmission latency by 27.8\% and improves testing accuracy by about 12\%; relative to standard benchmark routing algorithms, the proposed routing method further improves accuracy by roughly 8\%.
\end{abstract}


\section{Introduction}
Decentralized federated learning (D-FL) removes the dependency on a central parameter server and adopts a point-to-point architecture that distributes computation and communication loads more evenly across nodes, thereby improving system scalability, reducing communication overhead, and avoiding single points of failure~\cite{11130884,10431578}.

\subsubsection{Motivation}
In D-FL, clients exchange model updates using gossip protocols~\cite{9563232,10061474,10970239}, communicating only with neighboring nodes. Although fully decentralized, this restricts training to local interactions, resulting in slower convergence and potential model inconsistencies. To overcome these limitations, some works employ multi-hop routing, forwarding updates through relay nodes~\cite{10965802,10391064}, which increases information coverage and enables more effective model aggregation than one-hop approaches~\cite{10766357,9912186}. Additionally, multi-hop routing allows efficient aggregation under constrained communication resources without requiring multiple iterations, surpassing the performance of edge aggregation methods~\cite{10456324}.

However, D-FL is still constrained by the limited computation and communication capabilities of edge devices~\cite{10518048,10870116,9849007}. Sending full models across multiple hops incurs high latency and energy costs. Centralized FL has studied compression methods such as quantization~\cite{10001205,9915794}, sparsification~\cite{10530219}, and pruning~\cite{10643325}, with pruning being especially effective by removing less important parameters without extra overhead, reducing communication cost and improving flexibility. Optimizing client routing and pruning in D-FL is therefore crucial to enhance efficiency, reduce latency, and maintain robustness in resource-limited networks.

\subsubsection{Challenges}

Cross-routing communication can improve model coverage in D-FL, but it also brings challenges in heterogeneous networks characterized by large bandwidth disparities and strict latency requirements. Fixed transmission schemes, such as applying a uniform pruning ratio to all clients~\cite{9713700}, often result in suboptimal resource utilization. Bandwidth-limited clients may be unable to finish transmissions on time, leading to synchronization delays and slower convergence. This issue becomes more critical in large-scale D-FL systems, where communication heterogeneity has a pronounced effect on overall performance.


Popular FL algorithms, such as FedAVG~\cite{2016Communication}, QAFeL~\cite{10705319}, and FedVKD~\cite{11063265}, typically assume that all clients have identical transmission capabilities. However, in practical scenarios, clients often exhibit significant heterogeneity in computation and communication capabilities. Some resource-constrained clients may experience substantial delays or even fail to participate in training, as they are unable to transmit complete model updates. The on-demand routing protocol proposed in \cite{8585453} offers new insights into reliable data transmission in highly heterogeneous networks. Accordingly, adaptively adjusting model pruning ratios and optimizing routing strategies based on clients’ transmission capacities and network conditions has become a crucial research direction in D-FL.

While adaptive pruning techniques have been explored in centralized FL to enhance communication efficiency~\cite{9762360}, their application in D-FL with multi-hop routing remains under-investigated. Some approaches consider adaptive pruning under gossip-based point-to-point communication~\cite{11027589,10683955}, but fail to account for the efficiency of multi-hop transmissions. Consequently, clients may employ overly aggressive pruning to satisfy latency requirements, which can result in substantial loss of critical model information and degrade the accuracy and convergence of the global model. 

\subsubsection{Contributions}
This paper tackles latency and bandwidth constraints in multi-hop D-FL by proposing a joint routing and model pruning framework. We analyze the effects of pruning and routing on communication latency and establish convergence guarantees. Then, routing paths and model retention rates are jointly optimized under latency and bandwidth constraints to improve convergence performance.

The main contributions of this work are summarized as follows:
\begin{itemize}
    \item To cope with bandwidth and latency constraints, we propose a novel joint routing and model pruning framework. This framework enables each client to determine its model retention rate by optimizing multi-hop transmission paths based on the maximum tolerable latency, model size, and wireless link transmission rates. To the best of our knowledge, this is the first study in D-FL that simultaneously considers multi-hop routing and model pruning.
    
    \item We formulate the number of model parameters based on the model retention rate and further derive the communication latency model of D-FL under a given retention rate and routing path. In addition, we provide an upper bound on the $\ell_2$-norm of client model deviation under joint routing and pruning. Based on this, we construct an optimization problem to maximize the model retention rate and improve convergence speed under latency and bandwidth constraints.
    
    \item By analyzing the dependency of model retention rate on the transmission path, we show that the original optimization problem can be equivalently transformed into a routing optimization problem. Accordingly, we propose a new routing algorithm to improve communication efficiency and maximize model retention under given latency and bandwidth constraints.

    \item   The proposed routing algorithm combines node priorities and client-aware link weights to effectively improve multi-hop model transmission efficiency, reduce client model transmission time, and enable the transmission of more model parameters within the given time.
\end{itemize}

\textbf{Extensive simulation results demonstrate that, compared with an unpruned system, the proposed optimal pruning strategy \textbf{reduces the average model transmission latency of clients by approximately 13\%} while \textbf{improving model testing accuracy by about 12\%}. Moreover, the optimized routing algorithm further \textbf{enhances testing accuracy by around 8\%} compared with conventional routing methods.}

The rest of this paper is organized as follows. Section II reviews the related works, followed by the system model in Section III. The convergence analysis and problem formulation are presented in Section IV. The new routing algorithm developed to maximize the model retention rate and subsequently optimize routing and pruning jointly is described in Section V, followed by the simulations and conclusions in Sections VI and VII, respectively. Notation and definitions are summerized in Table \ref{notations}

\begin{table}[t]
\renewcommand{\arraystretch}{1.0}
\small
\centering
\caption{Notation and definitions}
\label{notations}
\begin{tabular}{p{2.3cm}|p{13cm}}
\hline
\textbf{Notation} & \textbf{Definition}\\
\hline
$m, n, N$ & Client indices and total number of clients \\
$\mathcal{V}, \mathcal{E}$ & Sets of clients and edges \\

$\mathcal{N}_n$ & Neighboring nodes of client \(n\) \\
$\alpha$ & Index of training rounds \\

$p_n$ & Ideal aggregation coefficient of client $n$ \\
$p_{m,n}$ & Aggregation coefficient of client $m$ at client $n$ \\

$F_n, F_n^*$ & Local loss and its minimum for client $n$ \\
$F, F^*$ & Global loss and its minimum \\

$\boldsymbol{\omega}_{\alpha,n}$ & Updated local model of client $n$ at round $\alpha$ \\
$\boldsymbol{\hat{\omega}}_{\alpha,n}$ & Locally aggregated model at client $n$ \\

$\boldsymbol{\bar{\omega}}_{\alpha}$ & Ideal global model \\
$\mathcal{T}_m$ & Spanning tree rooted at client \(m\) \\

$\mathcal{I}_m$ & Ordered nodes in client \(m\)'s transmission path \\
$\mathcal{N}_{\mathcal{T}_m}(i)$ & Neighbors inside one-hop range of node $i$ \\

$t_{m,i}^\alpha$ & One-hop latency constraint at node \(i\) \\
$t_{m}^\alpha$ & Total latency constraint for client \(m\) \\

$t_{\max}$ & Max latency threshold \\
$\tilde{v}_{m,i}$ & Slowest link transmission rate at node \(i\) \\

$v_{(i,j)}, \chi_{(i,j)}$ & Transmission rate and weight of edge \((i,j)\) \\

$e_{\alpha,(m,n),k}$ & Success indicator of $k$-th model element from $m$ to $n$ \\

$\eta_m$ & Channel retention rate of client \(m\) \\
$r_m$ & Model retention rate, $r_m = \eta_m^2$ \\

$k, K$ & Model element index and size \\
$U_m$ & Model size for client \(m\) \\

$\zeta$ & Neural network layer index \\
$\delta_\zeta$ & Number of channels in layer \(\zeta\) \\

$l$ & Channel index \\
$\mathbf{W}_{\zeta, \alpha, m}$ & Model parameters of layer \(\zeta\) for client \(m\) \\

$Q_i$ & Priority of node $i$ \\
$\theta$ & Weight difference between new and old links \\

$h_{\alpha,(m,n)}$ & Channel gain between clients $m$ and $n$ \\

\hline
\end{tabular}
\end{table}

\section{Related Work}

As D-FL becomes increasingly used in multi-client, serverless settings, there is growing interest in enhancing its communication and computational efficiency, particularly in heterogeneous networks where bottlenecks and latency are critical. Strategies such as optimized model transmission and aggregation, multi-hop routing, and model pruning have emerged as essential tools for improving FL efficiency, boosting transmission performance, and enhancing overall system operation.

\subsection{Model Transmission and Aggregation}

In D-FL, clients exchange model updates through gossip protocols, enabling random point-to-point communication. The work in \cite{9996127} improved network utilization by favoring high-bandwidth peers to ensure convergence. However, due to the randomness of gossip propagation, multiple iterations are often needed for updates to reach all nodes, lowering efficiency and increasing errors for sparsely connected clients.

Recent studies have introduced routing mechanisms to enhance coverage and reduce errors using multi-hop aggregation~\cite{10454716,10965802}. For example,  the authors of \cite{9705093} developed device-to-device aggregation. Nevertheless, existing methods still lack full optimization of multi-hop routing.
Dynamic routing and model forwarding were proposed in \cite{10454716} to reduce transmissions and improve performance. Similarly, the R\&A D-FL approach~\cite{10965802} allows efficient model exchange via predefined paths. Yet, these works do not fully analyze the joint impact of routing and pruning, nor provide tailored routing optimization for resource-limited wireless networks.

\subsection{Multi-hop Routing Optimization}

Multi-hop routing optimization has evolved along four main directions: classical graph algorithms, intelligent heuristic methods, and reinforcement learning~\cite{Kheradmand2022ClusterBasedRS,article,HOSSEINZADEH2025101657,2024An}. Classical algorithms, such as Kruskal~\cite{10726020} and Bellman-Ford~\cite{10760521}, perform well in point-to-point routing but are unsuitable for the broadcast nature of federated learning. Flooding achieves full coverage but incurs high latency~\cite{9734208}, while heuristic methods (e.g., ant colony~\cite{HOSSEINZADEH2025101657} and genetic algorithms~\cite{10816321}) provide adaptability but often ignore link parallelism and bottlenecks caused by the slowest link.

Reinforcement learning approaches exhibit strong robustness under interference and competition~\cite{2024An}, but their focus on single-path latency limits their effectiveness in broadcast-based multi-hop scenarios. Therefore, a new routing strategy is needed that simultaneously considers link parallelism and bottleneck delays, while maintaining system adaptability and reducing computational complexity, enabling efficient operation in high-density, large-scale urban networks~\cite{Alotaibi2025OptimizingDR}.

\subsection{Model Pruning}
Recently, various model compression techniques have been proposed to reduce communication overhead and accelerate FL training, including quantization~\cite{10001205}, sparsification~\cite{10530219}, and model pruning. For example, AdaQuantFL~\cite{DBLP:journals/corr/abs-2102-04487} reduces parameter precision from 32-bit floating-point to 8-bit integers to cut transmission size. However, its fixed uniform quantization lacks flexibility and may degrade important parameters’ precision, limiting communication efficiency. Another method uses sparsification by zeroing weights to reduce communication, but requires extra transmission of sparse indices and struggles with device heterogeneity~\cite{10535199}.

Model pruning directly improves communication efficiency by reducing the number of transmitted parameters, effectively addressing the core challenge in FL~\cite{6020937}. Pruning can be unstructured or structured. Unstructured pruning removes individual weights based on importance, but leads to irregular sparse patterns that lower compression efficiency and require specialized hardware~\cite{10350508}. Structured pruning removes filters and channels in convolutional layers, preserving the original structure for acceleration without special hardware~\cite{9842506}. The authors of \cite{DBLP:journals/corr/abs-2010-01264} highlighted that pruning input and output channels in heterogeneous client environments significantly reduces model size while maintaining consistency between local and global models, thus improving aggregation stability.

\subsection{Efficient Federated Learning}

Model pruning is widely used in FL to reduce communication and speed up training. For example, FedGroup-Prune reduces parameters in fully connected layers with a fixed pruning rate, which harms the performance of high-ability clients in ~\cite{10677496}. A few works explore adaptive pruning, such as adjusting channel widths based on client capabilities~\cite{DBLP:journals/corr/abs-2010-01264}. 
However, most methods are designed for centralized FL and are not suitable for D-FL. Pruning in D-FL is still underexplored. Although recent studies such as AdapCom-DFL~\cite{10304208} and DF2-MPC~\cite{10683955} have introduced personalized model pruning for individual clients in D-FL to enhance computational and communication efficiency, their underlying communication mechanisms remain fundamentally based on gossip-style peer-to-peer exchanges between directly connected neighbors. This makes them unsuitable for real-world network environments with complex multi-hop routing topologies and heterogeneous bandwidth and latency conditions across routes.

Compared with existing reinforcement learning-based routing optimization~\cite{10976414} and the ASDQR client selection scheme~\cite{10.1145/3716870}, our framework supports real-time adaptive routing paths and model sizes, demonstrating lower latency and stronger convergence performance in large-scale multi-hop heterogeneous scenarios.
Table \ref{tab:comparison} compares the proposed approach with existing related studies.

\begin{table}[t]
\centering
\scriptsize 
\captionsetup{font=scriptsize}
\caption{Comparison with Related Works in Decentralized Federated Learning}
\label{tab:comparison}
\begin{tabular}{|p{1.2cm}|p{1.2cm}|p{0.8cm}|p{1cm}|p{0.8cm}|p{1.5cm}|}
\hline
\textbf{Reference} 
& \textbf{Convergence} 
& \textbf{Multi-hop} 
& \textbf{Comm. Constraints} 
& \textbf{Model Pruning} 
& \textbf{Joint Pruning \& Routing} \\
\hline
\cite{10965802} & $\checkmark$ & $\checkmark$ & $\times$ & $\checkmark$ & $\times$ \\
\hline
\cite{10304208} & $\checkmark$ & $\times$ & $\checkmark$ & $\checkmark$ & $\times$ \\
\hline
\cite{10683955} & $\times$ & $\times$ & $\checkmark$ & $\checkmark$ & $\times$ \\
\hline
\cite{10535199} & $\checkmark$ & $\times$ & $\checkmark$ & $\times$ & $\times$ \\
\hline
\textbf{This paper} & \textbf{$\checkmark$} & \textbf{$\checkmark$} & \textbf{$\checkmark$} & \textbf{$\checkmark$} & \textbf{$\checkmark$} \\
\hline
\end{tabular}
\end{table}

\begin{figure}[t]
\centering
\includegraphics[width=0.8\linewidth]{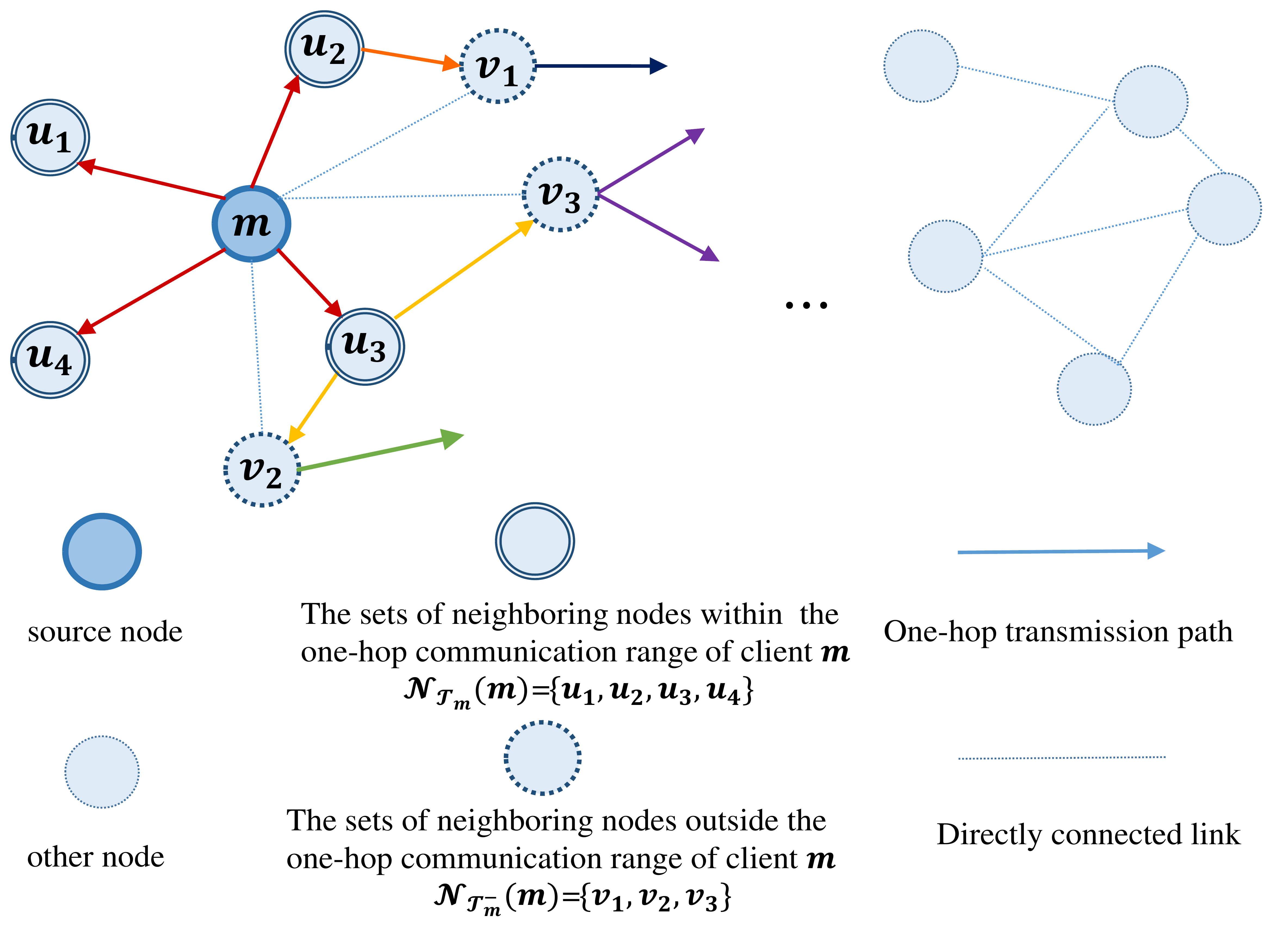}
\caption{An illustration of the multi-hop route for transmitting the local model from the source client to all target clients in the $\alpha$-th round.}\label{system model}
\end{figure}
\section{System model}\label{sec:system_model}

We consider a D-FL system with $N$ edge clients and no central point. The topology of the system is characterized by an undirected graph $\mathcal{G}(\mathcal{V}, \mathcal{E})$, where $\mathcal{V}=\left \{ 1,\dots, N\right \}$ is the index set of all clients and $\mathcal{E}$ is the set of edges between these clients. Let $\mathcal{N}_{n}\subseteq  \mathcal{V}$ denote the neighbor set of client $n\in\mathcal{V}$, collecting the clients within its one-hop communication range.

The communication links between neighboring clients are bandwidth-limited and subject to delay constraints, which are critical factors in many time-sensitive machine learning tasks~\cite{10302355,10835163}.
In the proposed D-FL framework, clients transmit their local models to all other clients via multi-hop routing, as shown in Fig. \ref{system model}. When the multi-hop routing causes model transmission latency to exceed the deadline, potentially affecting timely delivery, model pruning is employed to ensure that local models are delivered on time. The pruning strategy, optimized based on routing, effectively enhances the convergence performance of the D-FL system.
An overview of the proposed framework is illustrated in Fig.~\ref{fig:framework}.

\begin{figure}[t]
\centering
\includegraphics[width=0.99\linewidth]{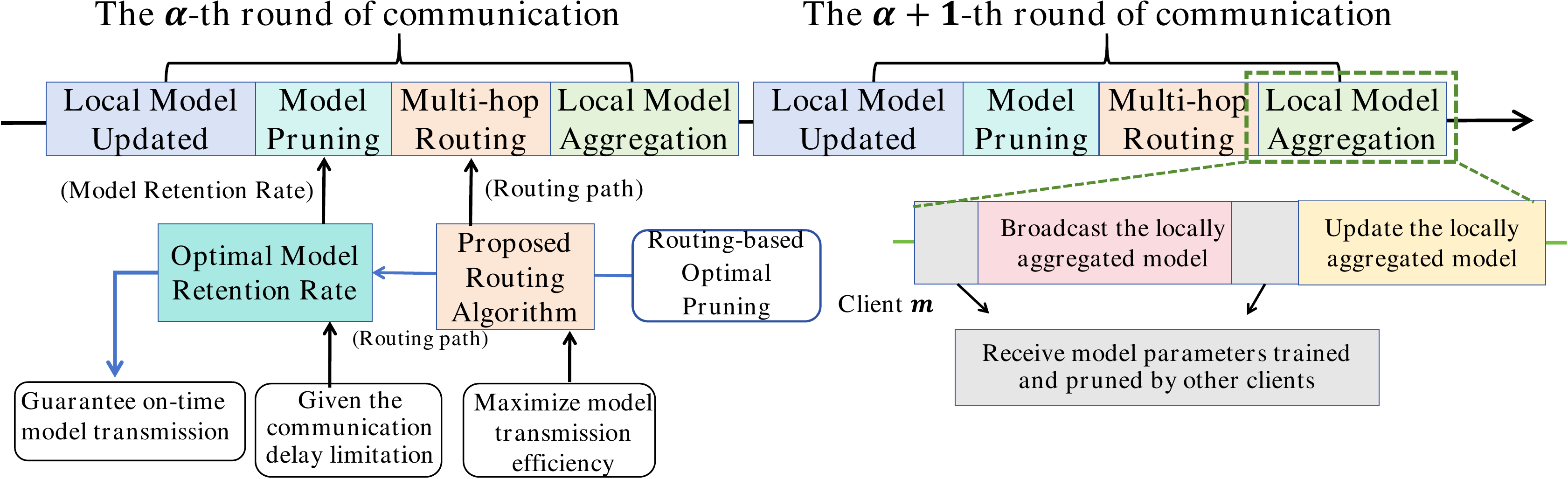}
\caption{Workflow of the D-FL rounds with local model update, pruning, multi-hop routing, and aggregation.}\label{fig:framework}
\end{figure}
\subsection{D-FL model}
The goal of D-FL is to minimize the global loss function $\underset{\boldsymbol{\omega}\in \mathbb{R}^{K}}{ \min}\, F  ( \boldsymbol{\omega}  ) $ and find the optimal model parameter $ \boldsymbol{\omega}^{*} =\underset{ \boldsymbol{\omega}\in \mathbb{R}^{K}}{\arg\min}\, F(\boldsymbol{\omega})$, where $F  ( \boldsymbol{\omega} )$ is the weighted sum of the local loss functions of all clients, as given by
\begin{equation}
    F(\boldsymbol{\omega})={\sum}_{\forall n \in \mathcal{V} }p_{n}F_{n}(\boldsymbol{\omega}),
\end{equation}
\noindent where $ p_{n} =\frac{D_{n}}{{\sum}_{\forall n\in \mathcal{V} } D_{n}}$ is the ideal aggregation weight coefficient under the classical FL frameworks, e.g., FedAvg.

Each client $n \in \mathcal{V}$ participates in model training and has a local dataset $\mathcal{D}_{n}$ with $ D_n=\left |\mathcal{D}_{n} \right|$ data points. 
The local loss function of client $n\in \mathcal{V}$ is  $    F_{n}(\boldsymbol{\omega})=\frac{1}{D_{n}}  {\textstyle \sum_{d\in \mathcal{D}_{n}}} f(d,\boldsymbol{\omega}),$
where $\boldsymbol{\omega}\in \mathbb{R}^{K}$ denotes the $K$-dimensional model parameters, $d\in \mathcal{D}_{n}$ is a data sample, and $f(d,\boldsymbol{\omega}) $ is the loss function of the data sample $d$ with respect to the model parameters $\boldsymbol{\omega}$.

Several training rounds are needed to find the optimal model parameter $ \boldsymbol{\omega}^{*}$.
In the $\alpha $-th training round, each client $n\in\mathcal{V}$ locally updates model parameters using stochastic gradient descent (SGD) with a mini-batch $\xi \subseteq \mathcal{D}_n$ randomly selected from its local dataset $\mathcal{D}_{n}$ and the locally aggregated model of the $(\alpha -1)$-th round, i.e., $\hat{\boldsymbol{\omega}}_{\alpha -1,n}$. The updated local model is 
\begin{align}\label{eq:epoch_imperfect}
\boldsymbol{\omega}_{\alpha ,n}=\hat{\boldsymbol{\omega}}_{\alpha \!-\!1,n}-\eta\nabla F_{n}(\hat{\boldsymbol{\omega}}_{\alpha \!-\!1,n},\xi), 
\end{align}
where $\nabla F_n({\mathbf{\cdot}},\xi)=\frac{1}{|\xi|}  \sum_{d\in \xi} \nabla f(d,{\mathbf{\cdot}})$ is the stochastic gradient, and $\eta$ is the learning rate.

\subsection{Model Pruning}\label{section:model pruning}

After each client \( m \in \mathcal{V} \) completes the \( \alpha \)-th training round, each client cannot transmit the complete model to other clients within the given time constraints. To address this issue, we adopt structured pruning by adjusting the number of input and output channels to effectively reduce model parameters, thereby ensuring timely model transmission.

\begin{figure}[t]
\centering
\includegraphics[width=0.8\linewidth]{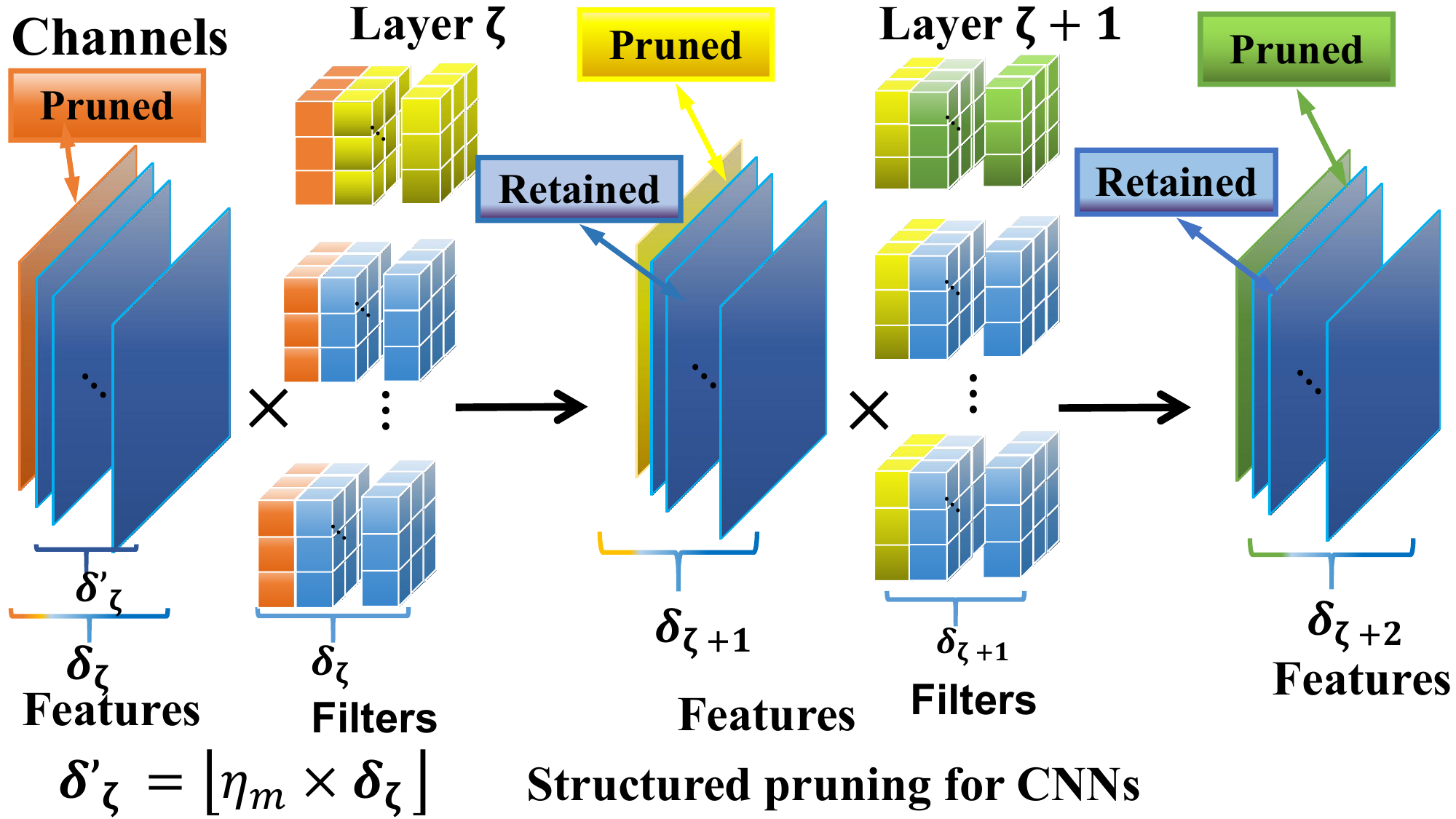}
\caption{
Schematic of structured pruning. This figure shows the structured pruning process applied to input and output channels in a hidden layer. Blue indicates retained channels, while gray and yellow represent pruned input and output channels, respectively. Each small square denotes a model parameter.
}

\label{fig:CNN_SEGMNETATION}
\end{figure}

Fig.~\ref{fig:CNN_SEGMNETATION} illustrates the process of applying structured pruning in a convolutional neural network (CNN). For the $\zeta$-th layer, the weight matrix is denoted as $\mathbf{W}_{\zeta,\alpha,m} \in \mathbb{R}^{\delta_\zeta \times \delta_{\zeta+1}}$, where $\delta_\zeta$ and $\delta_{\zeta+1}$ represent the number of input and output channels, respectively.
Let $\eta_m$ denote the channel retention ratio for client $m$, i.e., the proportion of channels retained in each layer. According to the predefined $\eta_m$, client $m$ retains the first $\lfloor \eta_m \delta_\zeta \rfloor$ input channels and $\lfloor \eta_m \delta_{\zeta+1} \rfloor$ output channels. Based on the above rules, we construct a binary pruning mask to indicate which channels are retained (denoted by 1) and which are pruned (denoted by 0), as given by 

\begin{subequations}

\label{eqn:mask_total} 
\noindent
\begin{minipage}{0.48\textwidth}
\begin{equation}
    g^{\text{input}}_{\zeta, l} = 
    \begin{cases} 
    1, & l \leq \lfloor \eta_m \delta_\zeta \rfloor \\ 
    0, & \text{otherwise}
    \end{cases}
    \label{eqn:mask_in} 
\end{equation}
\end{minipage}
\hfill
\begin{minipage}{0.48\textwidth}
\begin{equation}
    g^{\text{output}}_{\zeta+1, l} = 
    \begin{cases} 
    1, & l \leq \lfloor \eta_m \delta_{\zeta+1} \rfloor \\ 
    0, & \text{otherwise}
    \end{cases}
    \label{eqn:mask_out} 
\end{equation}
\end{minipage}
\end{subequations}

\noindent where $l$ denotes the channel index. This mask guides the model to perform computations only on the retained channels during both the forward and backward passes, thereby enabling efficient channel pruning and reducing computational overhead.

After pruning, the actual number of channels in the layer becomes $\delta_\zeta^m = \sum_{l=1}^{\delta_\zeta}  g^{\text{input}}_{\zeta, l} = \eta_m \delta_\zeta, \quad
\delta_{\zeta+1}^m = \sum_{l=1}^{\delta_{\zeta+1}}  g^{\text{output}}_{\zeta+1, l} = \eta_m \delta_{\zeta+1}$. 
Accordingly, the number of model parameters retained by client $m$ in this layer is $|\mathbf{W}'_{\zeta, \alpha, m}| = \eta_m^2 |\mathbf{W}_{\zeta, \alpha, m}|$. The corresponding \textbf{model retention rate} is $r_m = \frac{|\mathbf{W}'_{\zeta, \alpha, m}|}{|\mathbf{W}_{\zeta, \alpha, m}|} = \eta_m^2$. With this mechanism, the system can adaptively adjust the model size based on each client's communication capacity, ensuring that the model can be uploaded within a given time budget. The actual number of parameters uploaded by client $m$ is 
\begin{equation}
    U_m = |\boldsymbol{\omega}'_{\alpha ,m}|=\left\lfloor r_m K \right\rfloor\label{eq:U_m}
\end{equation}
Here, $K = |\boldsymbol{\omega}_{\alpha ,m}|$ is the total number of parameters in the full model; $r_m$ is the proportion of parameters uploaded,  $\boldsymbol{\omega}_{\alpha,m} \in \mathbb{R}^{K}$ denotes the full model parameter vector of client $m$ at the $\alpha$-th training round.
 During transmission, each client uploads only its pruned model parameters $\boldsymbol{\omega}'_{\alpha ,m}$. Since all clients pre-share the model structure, the parameter locations can be accurately reconstructed locally. This avoids aggregation errors caused by inconsistent pruning, and ensures alignment and stable convergence.

\subsection{Routing and Model Transmission}
The pruned models are transmitted to all other clients via multi-hop routing, and each client synchronously collects and aggregates the locally trained models from others. To optimize communication efficiency, we construct a collision-free multi-hop transmission model and mathematically analyze the joint impact of model pruning and routing selection on communication latency. 

\subsubsection{Collision-Free Transmission Protocol}

\begin{figure}[t]
\centering
\includegraphics[width=0.8\linewidth]{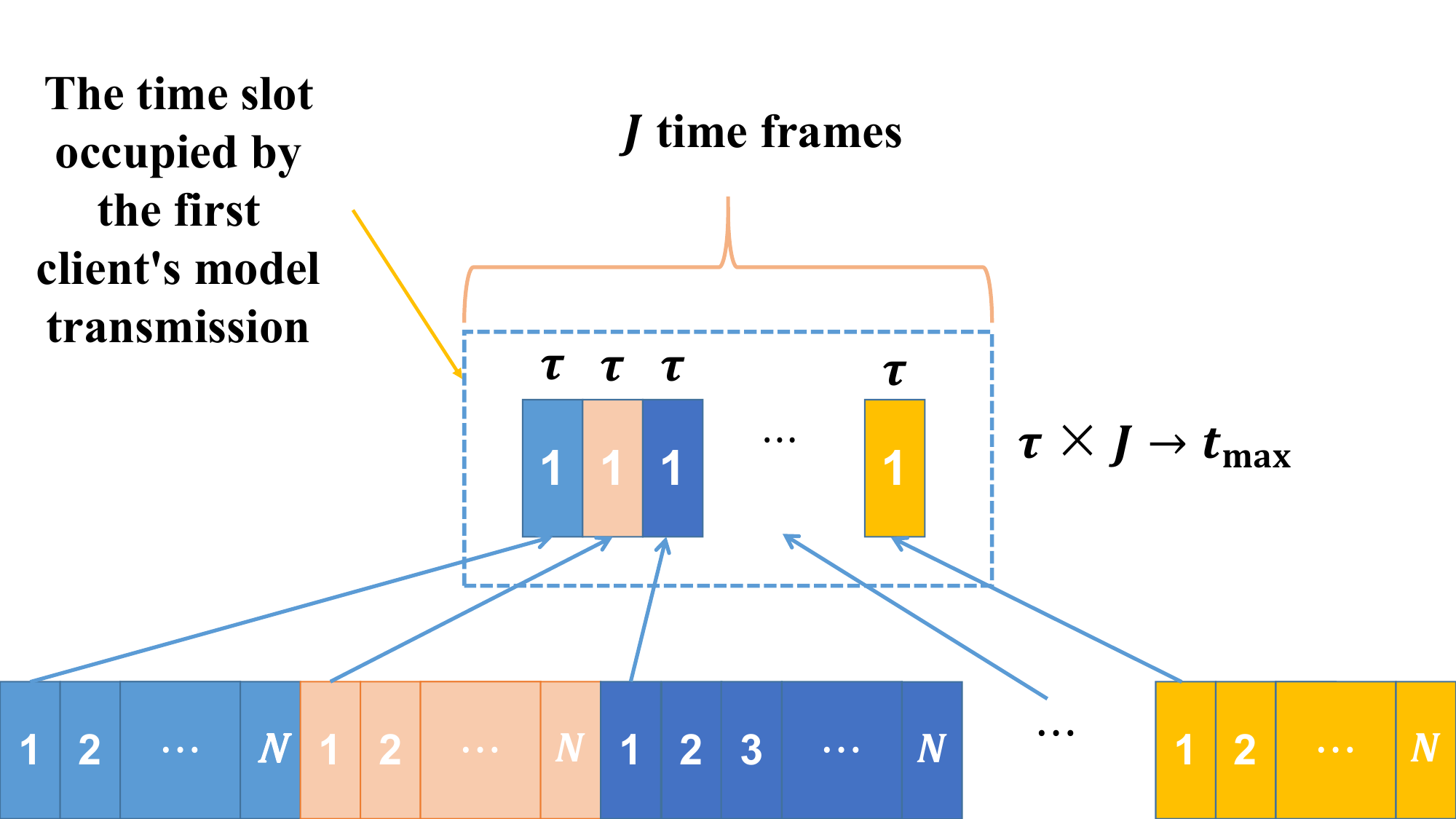}
\caption{An illustration of the TDMA-based collision-free protocol for model exchange in a decentralized network.}\label{fig:TDMA}
\end{figure}

To avoid transmission collisions, the model exchange follows a time division multiple access (TDMA) protocol, in which each client is assigned a dedicated time slot within each frame, as illustrated in Fig.~\ref{fig:TDMA}.
Each round comprises $J$ frames. Each frame consists of $N$ time slots, with each time slot having a duration of $\tau$. The $n$-th slot is exclusively allocated for the transmission of client $n$'s local model. Thus, each client has a maximum latency constraint of $t_{\max} = \tau J$ per round through its designated route, and the total latency constraint for all clients in a round is $N t_{\max}$.

According to graph theory~\cite{wilson10}, the total number of time slots required per training round is \( JN \), where \( N \) denotes the total number of clients in the network. Since the network topology remains unchanged within a training round, the connectivity among clients is also static during the \( N \) time slots. Therefore, the \( JN \) time slots can be precisely allocated to all clients using the edge coloring algorithm~\cite{wilson10}.

\subsubsection{Routing under Latency Constraints}This work primarily focuses on model transmission latency. It is worth noting that although node load and queuing effects may impact the overall communication latency, these factors can be effectively mitigated through proper resource scheduling. Therefore, their impact is ignored in this study.
 Due to different channel conditions and transmission powers of clients, the transmission rates of direct links between neighboring clients vary. The one-hop latency constraint of the client's local model is constrained by the minimum transmission rate among the links within the one-hop range. Let \( \mathcal{ T}_m \) be the spanning tree constructed with client \( m \) as the root, which defines the transmission paths from client \( m \) to other target nodes. Let \( \mathcal{ I}_m  \) be the set of nodes in the spanning tree \( \mathcal{ T}_m  \) responsible for transmitting client \( m \)'s local model. The nodes in \( \mathcal{I}_m \) are arranged according to the transmission path of client \( m \).

In the \( \alpha \)-th training round, the one-hop latency constraint for node \( i \) forwarding client \( m \)'s model is given by:
\begin{equation}
t_{m,i}^{\alpha} = \frac{U_m}{\tilde{v}_{m,i}}, \quad i\in\mathcal{I}_m, m\in\mathcal{V}.\label{eq:t_{m,i}}
\end{equation}
Here, \(U_m =\left\lfloor r_m K \right\rfloor\) defines the size of the model parameters transmitted by client \(m\), and \(\tilde{v}_{m,i}\) denotes the transmission rate of the slowest link within the one-hop communication range of node \( i \) when forwarding client \( m \)'s model, as given by:
\begin{equation}
\tilde{v}_{m,i} = \min_{j \in \mathcal{N}_{\mathcal{T}_m}(i)}\; v_{(i,j)}, \quad i\in\mathcal{I}_m, m\in\mathcal{V},\label{slowest transmission rate}
\end{equation}
where \(\mathcal{N}_{\mathcal{T}_m}(i)\) (and \(\mathcal{N}^{-}_{\mathcal{T}_m}(i)\)) denote the sets of neighboring nodes within (and outside) the one-hop communication range of node \(i\) when client \(m\)'s model is transmitted along the specified route \(\mathcal{T}_m\), such that \(\mathcal{N}_{\mathcal{T}_m}(i) \cup \mathcal{N}^{-}_{\mathcal{T}_m}(i) = \mathcal{N}_i\). Additionally, \(v_{(i,j)}\)\footnote{
For neighboring clients $(i,j)\in \mathcal{E}$, they receive the pilot and calculate the channel information of the direct link between $(i,j)$, and the signal-to-noise ratio in the communication channel is denoted by $\gamma_{(i,j)}$.
Based on the Shannon-Hartley theorem, we obtain the maximum theoretical data rate $v_{(i,j)}$ (in bits per second (bps)) of a directly connected link given its bandwidth $B$ (in Hz) and SNR $\gamma_{(i,j)}$ as given by
$v_{(i,j)} = B \log_2(1 + \gamma_{(i,j)})$.} is the transmission rate of the link \((i, j) \in \mathcal{E}\).

In the $\alpha$-th training round, client $m$ transmits its model with a total latency $t_m^\alpha$, which is the sum of the one-hop latencies along its route. This total latency must not exceed the maximum allowed time $t_{\max}$, expressed as:
\begin{equation}
t_{m}^{\alpha}=\sum_{i \in \mathcal{I}_{m}}t_{m,i}^{\alpha} = \sum_{i \in \mathcal{I}_{m}} \frac{\left\lfloor r_m K \right\rfloor}{\tilde{v}_{m,i}}\leq t_{\max}, \quad m\in\mathcal{V}.\label{eq:6}
\end{equation}

\subsection{Local Model Aggregation}\label{sec:Local Model Aggregation}
In the D-FL system, each client collects the local model parameters from other clients and aggregates them locally. A weighted averaging method is applied, aggregating only the retained model parameters. Specifically, the locally aggregated model for client \( n \) in the \( \alpha \)-th round is 
\begin{equation}
\boldsymbol{\hat{\omega}}_{\alpha ,n} = \left[\hat{{\omega}}_{\alpha,n,1}, \cdots, \hat{{\omega}}_{\alpha ,n,K}\right],\label{eq:hat}
\end{equation}
where the \( k \)-th model parameter in $\boldsymbol{\hat{\omega}}_{\alpha ,n}$ is given by
\begin{equation}
\hat{{\omega}}_{\alpha,n,k}=\sum_{\forall m\in\mathcal{V}}\frac{p_{m}{e}_{\alpha,(m,n),k}}{\sum_{\forall m'\in\mathcal{V}}p_{m'}{e}_{\alpha,(m',n),k}}{\omega}_{\alpha,m,k},\label{eq:omega}
\end{equation}
where \( e_{\alpha ,(m,n),k} \) indicates whether the \( k \)-th element of ${\boldsymbol{\omega}}_{\alpha, m}$ is pruned for transmission or not. Apparently, the model is not pruned and \(e_{\alpha ,(n,n),k} =1\) if $m=n$. Therefore,  \begin{align}
    \sum_{k=1}^K e_{\alpha ,(m,n),k}=\begin{cases} 
\lfloor r_m K\rfloor,\, & \text{if } m\neq n ;\\
K, & \text{ if }m=n.
\end{cases}\label{eq:prune_indicator}
\end{align}%
Here, \( p_{\alpha ,(m,n),k} \triangleq \frac{p_{m}{e}_{\alpha ,(m,n),k}}{\sum_{\forall m' \in \mathcal{V}} p_{m'}{e}_{\alpha ,(m',n),k}} \) is client $n$'s local aggregation coefficient for the \( k \)-th element of ${\boldsymbol{\omega}}_{\alpha ,m}$. Obviously, $\sum_{m \in \mathcal{V}} p_{\alpha ,(m,n),k} = 1$, ensuring the local aggregation coefficients are normalized at any client $n,\,\forall n\in \mathcal{V}$. The detailed steps of the D-FL framework with routing and pruning are provided \textbf{Algorithm~\ref{alg:adaptrout-D-FL}}.

\begin{algorithm}[ht]\footnotesize
\caption{D-FL With Pruning and Routing}\label{alg:adaptrout-D-FL}
\begin{algorithmic}[1]
\State \textbf{Input:} Data $\mathcal{D}_{n}$ distributed on $N$ clients, the local minibatch size $\xi$, the learning rate $\eta$, the model transmission path $\mathcal{T}_m$ and the corresponding set of relay nodes $\mathcal{I}_m$ , model retention rate $r_m$, $m \in \mathcal{V}$ .
\For{each round $\alpha=1,2,\cdots$}
    \For{each client $n \in \mathcal{V}$ in parallel}
        \State $\boldsymbol{\omega}_{\alpha ,n} \gets $ \textbf{ClientUpdate}($\xi,\hat{\boldsymbol{\omega}}_{\alpha-1,n}$);
    \EndFor
    \For{each client $n \in \mathcal{V}$}
        \State $\hat{\boldsymbol{\omega}}_{\alpha,n} \gets$ \textbf{Aggregate\&Prune}($ \boldsymbol{\omega}_{\alpha ,m}, r_m, \mathcal{T}_m$, $\mathcal{I}_m ; \forall m \in \mathcal{V}$);
    \EndFor
\EndFor   
\end{algorithmic}

\vspace{2 mm}
\underline{\textbf{ClientUpdate }$ (\xi,\hat{\boldsymbol{\omega}}_{\alpha-1,n})$:}\nonumber
\begin{algorithmic}[1]
            \State Sample batch $\xi \in \mathcal{D}_n$;
            \State Perform local training with \eqref{eq:epoch_imperfect} and obtain $\boldsymbol{\omega}_{\alpha ,n}$.
\end{algorithmic}

\vspace{2mm}

\underline{\textbf{Aggregate\&Prune }$(\boldsymbol{\omega}_{\alpha ,m}, r_m, \mathcal{T}_m$, $\mathcal{I}_m; \forall m \in \mathcal{V})$:}\nonumber
\begin{algorithmic}[1]
\State \textbf{Initialize:} The updated model of each client, $\boldsymbol{\omega}_{\alpha,m}$, and the model retention rate $r_m$, $ \forall m \in \mathcal{V}$.

\State \textbf{Model Pruning:} The client prunes its local model $\boldsymbol{\omega}_{\alpha,m}$ according to the model retention rate $r_m $, resulting in the retained model $\boldsymbol{\omega}_{\alpha,m}'$.
\State \textbf{Multi-hop Transmission:} The client sends the pruned model \( \boldsymbol{\omega}_{\alpha,m}' \) to all other clients via multi-hop transmission routing $\mathcal{T}_m$, $\mathcal{I}_m$.
\State \textbf{Model Aggregation:} The client aggregates the received pruned model \( \boldsymbol{\omega}_{\alpha, m}' \), \( \forall m \in \mathcal{V} \), according to \eqref{eq:omega} and \eqref{eq:prune_indicator}, resulting in \( \hat{{\omega}}_{\alpha,n,k}\), where \( k = 1, \dots, K \).

\State Output the local
aggregated model $\boldsymbol{\hat{\omega}}_{\alpha ,n}$, by \eqref{eq:hat}.
\end{algorithmic}

\end{algorithm}

\section{Convergence Analysis and Formulation}\label{sec:model_retention}
This section first conducts a theoretical convergence analysis of the convergence behavior of model pruning in D-FL operating in multi-hop wireless network environments. We subsequently construct an optimization problem aimed at minimizing the obtained convergence upper bound.

The objective of the D-FL aggregation process is to ensure the locally aggregated models at each device $n$ converge to the ideal global model\footnote{Note that \(\boldsymbol{\bar{\omega}}_{\alpha}\) differs from the global model in C-FL. As given in \eqref{eq:epoch_imperfect}, the local training of each D-FL round begins with the models locally aggregated by individual clients, i.e., \(\hat{\boldsymbol{\omega}}_{\alpha - 1, n}\), rather than using the globally aggregated model from the previous round, \(\boldsymbol{\bar{\omega}}_{\alpha - 1}\), as in C-FL.
}, i.e.,
\begin{align}\label{accurate_model}
\boldsymbol{\hat{\omega}}_{\alpha ,n} \rightarrow\boldsymbol{\bar{\omega}}_{\alpha }={\sum}_{\forall  n \in \mathcal{V}}p_{n}\boldsymbol{\omega}_{\alpha ,n}.
\end{align}

Achieving this might not be feasible due to the effects of distributed and local model aggregation, as well as the constraints of time resources.

\subsection{Convergence Analysis} 

We analyze the convergence of \textbf{Algorithm~\ref{alg:adaptrout-D-FL}} and propose several key supporting lemmas. These lemmas establish the one-round convergence upper bound of D-FL under the integration of model pruning and multi-hop routing mechanisms. Unlike the lemmas introduced in \cite{9154332}, the proposed lemmas take into account the case where the model retention rates of each client are different.

\begin{lemma}\label{lemma1}

Under the common assumptions that the local FL objectives \(F_n(\cdot), \forall n\) are \(L\)-smooth and \(\mu\)-strongly convex \cite{10965802}, the difference between the global model of D-FL at the \(\alpha\)-th communication round, \(\boldsymbol{\bar{\omega}}_{\alpha}\), and the global optimum \(\boldsymbol{\omega}^*\), is bounded as follows:
{\small
\begin{equation}\label{delta_1}
\begin{aligned}
&\Vert\boldsymbol{\bar{\omega}}_{\alpha}-\boldsymbol{\omega}^*\Vert^2
\leq (1+\tau_{\varrho})\Big[
\left(1-2\mu\eta+\eta^{2}L^{2}\right)
\Vert\boldsymbol{\bar{\omega}}_{\alpha-1}-\boldsymbol{\omega}^*\Vert^2 \\
&\!\!\!+\frac{1+\eta L}{\tau_{\varrho}}
\Big(\sum_{n'\in \mathcal{V}}p_{n'}^{2}+\eta L p_{\max}\Big)
\sum_{n\in\mathcal{V}}
\Vert\hat{\boldsymbol{\omega}}_{\alpha-1,n}-\bar{\boldsymbol{\omega}}_{\alpha-1}\Vert^{2}
\Big]
\end{aligned}
\end{equation}
}
where $\tau_{\varrho}$ is a freely adjustable tuning parameter that balances the contributions of two error terms in the convergence analysis, aiming to obtain a tighter upper bound.
\end{lemma}%
\begin{proof}
    See \textbf{Appendix \ref{proof:lemma1}}.
\end{proof}

\textbf{Lemma \ref{lemma1}} reveals that model pruning and routing introduce significant aggregation errors, which are mainly dominated by the bias term in the second component of~\eqref{delta_1}. Specifically, the first component is independent of the communication and aggregation mechanisms, whereas the second component captures the deviation between the locally aggregated model under imperfect aggregation and the ideal global model. 
As referring to~\cite{10965802}, the convergence analysis can also be extended to non-convex FL objective functions, where the aggregation error term has the same influence on the convergence.

The one-round upper bound of D-FL  with joint model pruning and routing, i.e., the RHS of \eqref{delta_1} in \textbf{Lemma~\ref{lemma1}}, is dominated by the model biases between the ideal aggregated model $\boldsymbol{\bar{\omega}}_{\alpha }$ and the imperfect locally aggregated model $\hat{\boldsymbol{\omega}}_{\alpha-1,n},\forall n$, i.e., $\Vert\hat{\boldsymbol{\omega}}_{\alpha-1,n}-\bar{\boldsymbol{\omega}}_{\alpha-1}\Vert^{2}=\sum_{k=1}^{K}|\hat{{\omega}}_{\alpha-1,n,k}-\bar{{\omega}}_{\alpha-1,k}|^{2}$, where the $k$-th element of the model bias of client $n$'s local model can be rewritten as
\begin{align}
\hat{{\omega}}_{\alpha,n,k}-\bar{{\omega}}_{\alpha,k}={\sum}_{\forall m\in \mathcal{V}}\lambda_{\alpha,(m,n),k}\cdot{\omega}_{\alpha,m,k},
\end{align}
where $\lambda_{\alpha,(m,n),k}\triangleq {p}_{\alpha,(m,n),k}-p_{m}, \forall m,n \in \mathcal{V}$. 

Then, the sum of the $\ell_2$-norms of the model biases across all clients can be bounded by
\begin{equation}\small
\begin{aligned}
&\sum_{n\in\mathcal{V}}\Vert\hat{\boldsymbol{\omega}}_{\alpha-1,n}-\bar{\boldsymbol{\omega}}_{\alpha-1}\Vert^{2}=
\sum_{n\in\mathcal{V}}\sum_{k=1}^{K}(\hat{{\omega}}_{\alpha-1,n,k}-{\bar{\omega}}_{\alpha-1,k})^{2} \\
&= \sum_{ n\in\mathcal{V}}\sum_{k=1}^{K}\Big(\sum_{\forall m\in\mathcal{V}} p_{\alpha-1 ,(m,n),k} {\omega}_{\alpha-1,m,k}-\sum_{\forall  m \in \mathcal{V}}p_{m}\omega_{\alpha-1 ,m,k}\Big)^{2}\\
\leq &\sum_{ n\in\mathcal{V}} \sum_{k=1}^{K}\Big( \sum_{ m\in \mathcal{V}}\lambda_{\alpha-1,(m,n),k}^2\Big)\Big(\sum_{ m'\in \mathcal{V}}{\omega}_{\alpha-1,m',k}^{2}\Big) \\
&\leq \sum_{n\in\mathcal{V}}\Big(\sum_{k=1}^{K}\sum_{m\in\mathcal{V}}\lambda_{\alpha-1,(m,n),k}^{2}\Big)
\sqrt{\sum_{k'=1}^{K}\Big(\sum_{m'\in\mathcal{V}}{\omega}_{\alpha-1,m',k'}^{2}\Big)^{2}}.
\end{aligned}
\label{eq:sum}
\end{equation}

Here, $\sum_{k=1}^{K}\sum_{m\in\mathcal{V}}\lambda_{\alpha-1,(m,n),k}^{2}$ is the only part related to model pruning and multi-hop aggregation, where $\lambda_{\alpha,(m,n),k} = p_{\alpha,(m,n),k} - p_m$. According to Section~\ref{sec:Local Model Aggregation}, $p_{\alpha,(m,n),k}$ reflects the impact of multi-hop aggregation on the $k$-th parameter from client $m$.

The model bias can be upper-bounded as follows, which reveals the relationship between the retention rate and the convergence bound.

\begin{lemma}\label{lemma2}The upper bound of the dominant part of the model bias is given by
\begin{equation}\label{eq:lemma2}
\begin{aligned}
&{\sum}_{k=1}^{K} {\sum}_{m \in \mathcal{V}} \lambda_{\alpha-1,(m,n),k}^{2} \\
& \leq  {\sum}_{l \in \mathcal{V} \setminus \{n\}} (K -\lfloor r_l  K \rfloor) p_l^2  
+ \left( {\sum}_{l \in \mathcal{V} \setminus \{n\}} \!\!(K \! - \! \lfloor r_l K \rfloor) \right)^2
\end{aligned}
\end{equation}
where  \( p_l \) is the ideal aggregation weight coefficient under classical C-FL. 
\end{lemma}
\begin{proof}
    See \textbf{Appendix \ref{proof_lemma2}}.
\end{proof}

\subsection{Problem Formulation}\label{sec:pruning and routing}

Based on Lemma~\ref{lemma2}, the upper bound of the dominant term in the model bias is governed by the model retention rate under latency and bandwidth constraints. Hence, minimizing this upper bound contributes to reducing the global loss function in D-FL. Since $K$ and $p_l$ are constants and the collision-free transmission protocol renders the routing paths independent across clients, the problem can be formulated as maximizing the model retention rate of each client as
\begin{subequations} 
\begin{align} 
   & \textbf{P1:} \quad  \max_{\mathcal{I}_m, \mathcal{T}_m}\; r_m,\quad \forall m \in \mathcal{V} \label{object 1}\\
   \text{s.t.}\;  &  \sum_{i \in \mathcal{I}_{m}} \frac{\lfloor r_m K\rfloor}{\tilde{v}_{m,i}} \leq t_{\max}, \quad \forall m \in \mathcal{V}, \label{constratin 1a}   
   \\
    & \tilde{v}_{m,i} = \min_{j \in \mathcal{N}_{\mathcal{T}_m}(i)}\;  v_{(i,j)}\geq 0, \quad i\in\mathcal{I}_m, \label{constratin 1b} 
 \\
 &  0 < r_m  \leq 1,\quad\forall m \in \mathcal{V}. \label{constratin 1c} 
\end{align}\label{eq:P1}
\end{subequations}
Constraint~\eqref{constratin 1a} ensures timely model transmission within the latency limit $t_{\max}$, as given in \eqref{eq:6}. Constraint~\eqref{constratin 1b} specifies the transmission rate along the route \( \mathcal{I}_m \) based on prior channel state information; see \eqref{slowest transmission rate}. Constraint~\eqref{constratin 1c} restricts the feasible range of the model retention rate \( r_m \) for client \( m \).

Clearly,  \textbf{P1} is a mixed-integer nonlinear programming (MINLP) problem, which contains multiple parameters, i.e., $r_m, \mathcal{I}_{m}$, and $\mathcal{T}_{m}$, and is non-convex and challenging to solve optimally in practice. By examining the nested relationship between the model retention ratio $r_m$ and the routing variables $\mathcal{I}_{m}$ and $\mathcal{T}_{m}$, as reformulated from \eqref{constratin 1a} and \eqref{constratin 1b}, we have
\begin{equation}
  r_m \leq \frac{t_{\max}}{K \sum\limits_{i \in\mathcal{I}_{m}} \max\limits_{j \in \mathcal{N}_{\mathcal{T}_m}(i)} \chi_{(i,j)}}, \quad \forall m \in \mathcal{V},
  \label{eq:rm_v1}
\end{equation}
where $\chi_{(i,j)} \triangleq 1 / v_{(i,j)}$ denotes the edge weight between neighboring clients $\forall (i,j) \in \mathcal{E}$.  
In conjunction with Constraint~\eqref{constratin 1c}, the model retention rate is upper bounded as
\begin{equation}
    r_m \leq \min \bigg\{1, \frac{t_{\max}}{K \sum\limits_{i \in \mathcal{I}_m} \max\limits_{j \in \mathcal{N}_{\mathcal{T}_m}(i)} \chi_{(i,j)}}\bigg\}, \quad \forall m \in \mathcal{V}.
  \label{eq:rm*}
\end{equation}

Maximizing the model retention rate $r_m$ is equivalent to minimizing  $\sum_{i \in \mathcal{I}_m} \max_{j \in \mathcal{N}_{\mathcal{T}_m}(i)} \chi_{(i,j)}$. The original problem can be reformulated as a routing problem:
\begin{subequations}
\begin{align}
    \textbf{P2:} \quad & \min_{\mathcal{I}_m,\, \mathcal{T}_m} \sum_{i \in \mathcal{I}_{m}} \max_{j \in \mathcal{N}_{\mathcal{T}_m}(i)} \chi_{(i,j)}, \quad \forall m \in \mathcal{V}, \label{eq:P2a} \\
    \text{s.t.} \quad & 0 < \chi_{(i,j)} < \infty, \quad \forall (i,j) \in \mathcal{E}. \label{eq:P2b}
\end{align}\label{eq:P2}
\end{subequations}
After solving~\eqref{eq:P2} and obtaining the optimal routing solution for client $m$, we obtain
\begin{equation}
   \{\mathcal{I}_m^*,\, \mathcal{T}_m^* \}= \mathop{\arg\min}\limits_{\mathcal{I}_m, \mathcal{T}_m} \sum_{i \in \mathcal{I}_{m}} \max_{j \in \mathcal{N}_{\mathcal{T}_m}(i)} \chi_{(i,j)}, \label{eq:routing}
\end{equation}
The optimal model retention rate is given by
\begin{equation}
    r_m^* = \min \Big\{1, \frac{t_{\max}}{K \sum\limits_{i \in \mathcal{I}_m^*} \max\limits_{j \in \mathcal{N}_{\mathcal{T}_m^*}(i)} \chi_{(i,j)}}\Big\}, \quad \forall m \in \mathcal{V}.
    \label{eq:rm*}
\end{equation}

Traditional routing algorithms fail to efficiently obtain optimal solutions of the non-convex objective, owing to the high complexity of concurrent and distributed multi-user routing strategies inherent in D-FL systems. Model parameters must be transmitted via multi-hop broadcasting, where each hop may involve parallel transmissions to multiple neighboring nodes. Moreover, the single-hop latency constraint is constrained by the slowest communication link within the current broadcast range. This transmission mechanism is fundamentally different from traditional point-to-point communication, significantly increasing the complexity of routing optimization and rendering existing methods ineffective in addressing such localized and constrained model transmission challenges.

To address this challenge, we propose a new routing algorithm that optimizes the model transmission path to approximate the optimal solution of~\eqref{eq:P2a}. The complete algorithm design and theoretical analysis are provided in Section~\ref{sec:routing_optimization}.

\section{Proposed Routing Algorithm}\label{sec:routing_optimization}
In this section, we propose a routing optimization algorithm based on Node Priority and Client-Aware Link Threshold (\textbf{P\_CLT}), which aims to minimize the cumulative maximum link weights over all hops along the transmission path (i.e., problem \textbf{P2}) to enhance model transmission efficiency.

\begin{algorithm}\footnotesize
\caption{Routing Algorithms (P\_CLT)}\label{alg:Simplified Optimal Routing}
\begin{algorithmic}[1]
\State \textbf{Input:} Graph $\mathcal{G}(\mathcal{V},\mathcal{E})$, source $m$, iterations $\Psi$, $\theta$.
\State \textbf{Output:} Routing tree $\mathcal{T}_m^*$, transmission order $\mathcal{I}_m^*$

\vspace{2mm}

    \State Initialize  $\mathcal{G}(\mathcal{V},\mathcal{E})$ $\mathcal{T}_m''. \mathcal{E}_{\mathcal{T}_{m}''} \gets \text{MST}(\mathcal{G})$.
    \State $\mathcal{T}_m', \mathcal{E}_{\mathcal{T}_{m}'} \gets \textbf{ModifyLinks}(\eqref{eq:condition_theta},\mathcal{ T}_m'', \mathcal{E}_{\mathcal{T}_{m}''})$;
  
    \For{each iteration $\psi$ from 1 to $\Psi$}
        \State $\mathcal{T}_m, \mathcal{E}_{\mathcal{T}_{m}} \gets \textbf{ModifyLinks}(\eqref{eq:condition_w},\mathcal{ T}_m', \mathcal{E}_{\mathcal{T}_{m}'})$;
    \EndFor
\State $\mathcal{T}_m^* \gets \mathcal{T}_m$; $\mathcal{I}_m^* \gets \mathcal{T}_m^*$. The set of nodes in \( T^*_m \) that transmit client \( m \)'s local model, where the nodes in this set are arranged in the order of the multi-hop transmission path of client \( m \)'s local model.

\end{algorithmic}

\vspace{2mm}

\underline{\textbf{ModifyLinks($ f(\chi_{(c,u)} ,\chi_{(c,v)})$, $\mathcal{T}_m$, $\mathcal{E}_{\mathcal{T}_{m}}$)}:}\nonumber

\begin{algorithmic}[1]
 \State $Q = [Q_1, \ldots, Q_N]$, where $Q_i \gets$ via \eqref{eq:Q} and $\mathcal{T}_m$, $\forall i \in \mathcal{V}$
 \State \textbf{Initialize:} Source node set $\mathcal{V_{\rm s}} = \{ m \}$; processed node set $\mathcal{V}_{\rm queue} = \emptyset$; and non-selectable node set $\mathcal{V}_{\rm Q} = \emptyset$.
 
 \While{$\lvert \mathcal{V}_{\rm queue} \rvert \neq N$}
    \For{each node \( c \in \mathcal{V_{\rm s}} \)}
        \State $\mathcal{V}_{\rm queue} \gets \mathcal{V}_{\rm queue} \cup \{c \}$,  $\mathcal{V_{\rm Q}} \gets \mathcal{V_{\rm Q}} \cup \mathcal{N_{\rm\mathcal{ T}_m}}(c)$    
        \For{$v \in \mathcal{N_{\rm\mathcal{ T}_m^-}}(c)$ and $u \in \mathcal{N_{\rm\mathcal{ T}_m}}(c)$}
                \If{$ f(\chi_{(c,u)}, \chi_{(c,v)})$ holds}
                    \State Add the edge $(c, v)$ to $\mathcal{E}_{\mathcal{T}_{m}}$, and remove the edge that creates a cycle;
                    \State $\mathcal{T}_m' \gets\mathcal{ T}_m$; $\mathcal{V_{\rm Q}} \gets \mathcal{V_{\rm Q}} \cup \{ v \}$.       
                \EndIf
            \EndFor
    \EndFor
    \State \textbf{$\mathcal{V_{\rm Q}} \gets \text{Sort}(\mathcal{V_{\rm Q}}, \text{key} = Q[i], \text{reverse} = \text{True})$}, $\mathcal{V_{\rm s}} \gets \mathcal{V_{\rm Q}}$.
 \EndWhile
 \State \Return $\mathcal{T}_m'$ and $\mathcal{E}_m'$.
\end{algorithmic}
\end{algorithm}

The \textbf{P\_CLT} routing algorithm initially constructs a minimum spanning tree \(\mathcal{T}_m''\) for the model transmission of client \( m \) as the starting routing structure, where \( \mathcal{E}_{\mathcal{T}_m''} \) denotes the set of links in the spanning tree. By optimizing the spanning tree structure, the optimal routing solution is obtained, consisting of the optimal spanning tree \(\mathcal{T}_m^*\) and the optimal node set \(I_m^*\), where \( m \in \mathcal{V} \). The detailed process is illustrated in Algorithm \ref{alg:Simplified Optimal Routing}. The P\_CLT routing algorithm introduces the concept of ''link reciprocity'' and proposes two conditions for link addition based on this concept to adjust the spanning tree structure. Additionally, during the optimization process, we incorporate node priorities and client-aware link weights.

\subsection{Link Reciprocity}

To prevent the issue of increased cumulative maximum link weight caused by including both high-weight and low-weight links within the same hop in multi-hop communication, we introduce the concept of ``link reciprocity'' during the optimization of the spanning tree structure (i.e., the model transmission path). This mechanism ensures links with similar or smaller weights are assigned to the same hop. 
Following this principle, we propose two constraints for link addition:
\begin{equation}  
\lvert \chi_{(c,v)} - \chi_{(c,u)} \rvert \leq \theta, \label{eq:condition_theta}  
\end{equation}  
\begin{equation}  
\chi_{(c,v)} \leq \chi_{(c,u)}^{\max}. \label{eq:condition_w}  
\end{equation}  
For the current node \( c \), if the link \((c,v) \in \mathcal{E}_{T_{m}^-}\) satisfies either condition \eqref{eq:condition_theta} or \eqref{eq:condition_w}, it can be added to the one-hop communication of node \( c \).
Here, \(\chi_{(c,u)}^{\max}\) represents the maximum link weight within the one-hop communication range of node \(c\), where \( u \in \mathcal{N_{\rm \mathcal{T}_m}}(c) \) and \( v \in \mathcal{N_{\rm \mathcal{T}_m^-}}(c) \), with \( m \in \mathcal{V} \). The threshold \( \theta \) defines the maximum allowable weight difference between links within one hop. Adjusting \( \theta \) can control the range of one-hop communication; see Section~\ref{se:Client-Aware Link Weight Threshold}.

During link modification based on the link reciprocity principle, the algorithm sorts the adjacent links of all nodes by weight. The total sorting complexity is given by
\begin{equation}
\mathcal{O}\left(E \log \frac{E}{N}\right),\label{eq:O_link}
\end{equation}
where \(N\) is the number of nodes, and \(E\) is the number of edges. The detailed derivation is provided in Appendix~\ref{appendix:complexity_proof}.

\subsection{Node Priority and Client-Aware Link Threshold}\label{Generic Link Modification }
During routing optimization, inefficient link establishment with adjacent nodes or improper one-hop communication range configuration (where the weight difference between new and existing links is too large or too small) can result in suboptimal routing. To mitigate this, we introduce a node priority mechanism and a client-aware link threshold strategy to further enhance the spanning tree structure.

\subsubsection{Design of Node Priority}\label{Design of node priorities}
To avoid suboptimal routing caused by random link additions, we introduce the concept of \textit{node priority}. The priority of a node is defined as the number of neighboring nodes in the spanning tree, i.e.,
\begin{equation}
Q_{i} = |\mathcal{N_{\rm \mathcal{ T}_m}}(i)|,\label{eq:Q}
\end{equation}
where $Q_i$ is the priority of node $i$. The larger $Q_i$, the higher the priority of node $i$. The total time complexity for computing and sorting the degrees of all nodes is:
\begin{equation}
   \mathcal{O}(N \log N + E).\label{eq:O_node_pr}
\end{equation}
Specifically, the degree of each node can be obtained by traversing its adjacency list, which takes \(\mathcal{O}(N + E)\) time~\cite{enwiki:1289635162}. The sorting step can be performed using standard comparison-based algorithms, such as quicksort or mergesort, with a time complexity of \(\mathcal{O}(N \log N)\)~\cite{cormen2009introduction}.

\subsubsection{Client-Aware Link Weight Threshold}\label{se:Client-Aware Link Weight Threshold}

We control the one-hop communication range of nodes by dynamically adjusting the weight difference between the new and original links (\(\theta\)). A smaller \(\theta\) helps mitigate the negative impact of high-weight links, while a larger \(\theta\) can expand the one-hop communication range and improve transmission efficiency. The optimal value of \(\theta\) varies with network scale; therefore, in this paper, we determine its optimal value experimentally. Specifically, for a given number of clients \(N\), we conduct a grid search over \(\theta\) in the range of 0.1 to 0.9 to identify the value that minimizes the transmission time. As shown in Fig.~\ref{fig.alg2}(b), the parameter $\theta$ significantly affects the performance of Algorithm~\ref{alg:Simplified Optimal Routing}. Selecting an inappropriate value for $\theta$ not only results in increased transmission latency during the initial optimization stage but also negatively impacts the subsequent iterative refinement process. This sensitivity underscores the importance of carefully tuning $\theta$ to guide link adjustments and ensure the efficient operation of the routing algorithm.



\begin{remark}[Finite-step Convergence and Local Optimality]

At the beginning of each outer iteration, given the current spanning tree 
$\mathcal{T}_m$, the target routing tree $\mathcal{T}_m^*$ is exclusively determined 
by the inner-layer link optimization module \textbf{ModifyLinks}. During the inner-layer optimization, each link addition or replacement strictly avoids cycle formation and complies with the node-priority rules and link-addition 
constraints (\eqref{eq:condition_theta} or \eqref{eq:condition_w}). In the first stage, constraint \eqref{eq:condition_theta} is imposed. Existing 
studies have demonstrated that this mechanism effectively reduces communication 
latency while improving both model performance and system fairness~\cite{10171556}. 
Accordingly, under constraint \eqref{eq:condition_theta}, the objective function in \eqref{eq:P2a} is monotonically 
non-increasing. In the second stage, constraint \eqref{eq:condition_w} is incorporated and the outer 
loop is executed for $\Psi$ iterations, which further enforces a continuous decrease 
of the objective function. Since the objective function is non-negative and the set of feasible links is finite, 
the inner-layer optimization is guaranteed to converge in a finite number of steps. 
Across outer iterations, the algorithm converges to locally optimal routing trees 
under constraints \eqref{eq:condition_theta} and \eqref{eq:condition_w}, respectively. Given the finite number of feasible 
spanning trees, once a fixed point is reached, the overall algorithm is guaranteed 
to stabilize in finite steps Algorithm~\ref{alg:Simplified Optimal Routing} eventually converges.

\end{remark}

\subsection{Computational Complexity of Algorithm~\ref{alg:Simplified Optimal Routing}}

The main computational cost of Algorithm~\ref{alg:Simplified Optimal Routing} lies in the routing optimization process. Other steps, such as the calculation in Step 10 using ~\eqref{eq:rm*} and the constraints in ~\eqref{eq:condition_theta} and~\eqref{eq:condition_w}, involve only simple arithmetic operations with negligible overhead.
During initialization, the routing is constructed based on a minimum spanning tree with a time complexity of $\mathcal{O}(E \log E)$. This is followed by one link modification according to constraint~\eqref{eq:condition_theta}, and then $\Psi$ iterative link modifications based on constraint~\eqref{eq:condition_w}.
The link modification step includes node priority calculation and link updates, with complexity given by ~\eqref{eq:O_link} and~\eqref{eq:O_node_pr}. Thus, the overall time complexity is
$\mathcal{O}(E \log E) + \mathcal{O}\big((\Psi + 1) \times (E \log \tfrac{E}{N} + N \log N + E)\big).$
Here, $E$ is the number of links and $\Psi$ is the number of iterations. For sparse or medium-density networks (e.g., geometric graphs), $E = \mathcal{O}(N)$ typically holds, resulting in an overall complexity of $\mathcal{O}(\Psi \cdot N \log N)$. As shown in Fig.~\ref{fig.alg2}(b) for $N=20$, the optimized routing transmission time of Algorithm~\ref{alg:Simplified Optimal Routing} decreases with increasing $\Psi$ and stabilizes around $\Psi = 3$, indicating that only a few iterations are sufficient for convergence without significantly increasing computational complexity. Even for denser networks where $E = \mathcal{O}(N^2)$,  resulting in an overall complexity of $\mathcal{O}(\Psi \cdot N^2 \log N)$, the algorithm remains feasible for networks with several hundred clients.
Table~\ref{tab:complexity} compares the computational complexity of the proposed algorithm with existing standard routing methods in high-density networks.

\begin{table}[htbp]
\scriptsize 
\centering
\captionsetup{font=scriptsize} 
\caption{Comparison of Computational Complexity of Different Algorithms}
\label{tab:complexity}
\begin{tabular}{l|p{1cm}|p{0.9cm}|p{1.2cm}|p{1.5cm}}
\hline
\textbf{Algorithm} & Kruskal (MST) & Bellman-Ford & Flooding (Flood-Fill) & \textbf{Algorithm 2 (Proposed)} \\ \hline
\textbf{Complexity} & $N^2 \log N$ & $N^3$ & $N^2$ & $\Psi \times N^2 \log N$ \\ \hline
\end{tabular}
\end{table}

To illustrate the dynamic routing optimization of Algorithm~\ref{alg:Simplified Optimal Routing}, Fig.~\ref{fig.alg2}(a) shows a concrete example for client $m=13$. The process starts from the initial minimum spanning tree $\mathcal{T}_m''$, followed by the first link optimization under constraint \eqref{eq:condition_theta}  to obtain $\mathcal{T}_m'$, expanding coverage and reducing hops. Then, $\Psi$ iterations under constraint \eqref{eq:condition_w} refine the routing to the optimized tree $\mathcal{T}_m^*$.

\begin{figure}[t]
    \centering
    \begin{subfigure}[t]{0.5\textwidth} 
        \centering
        \includegraphics[width=\linewidth]{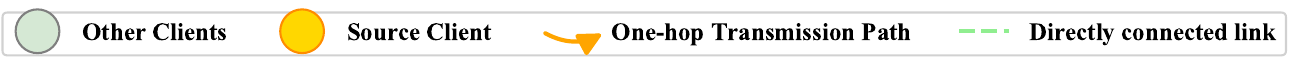}
        
        \vspace{2mm} 
        \begin{minipage}{0.32\linewidth}
            \centering
            \includegraphics[width=\linewidth]{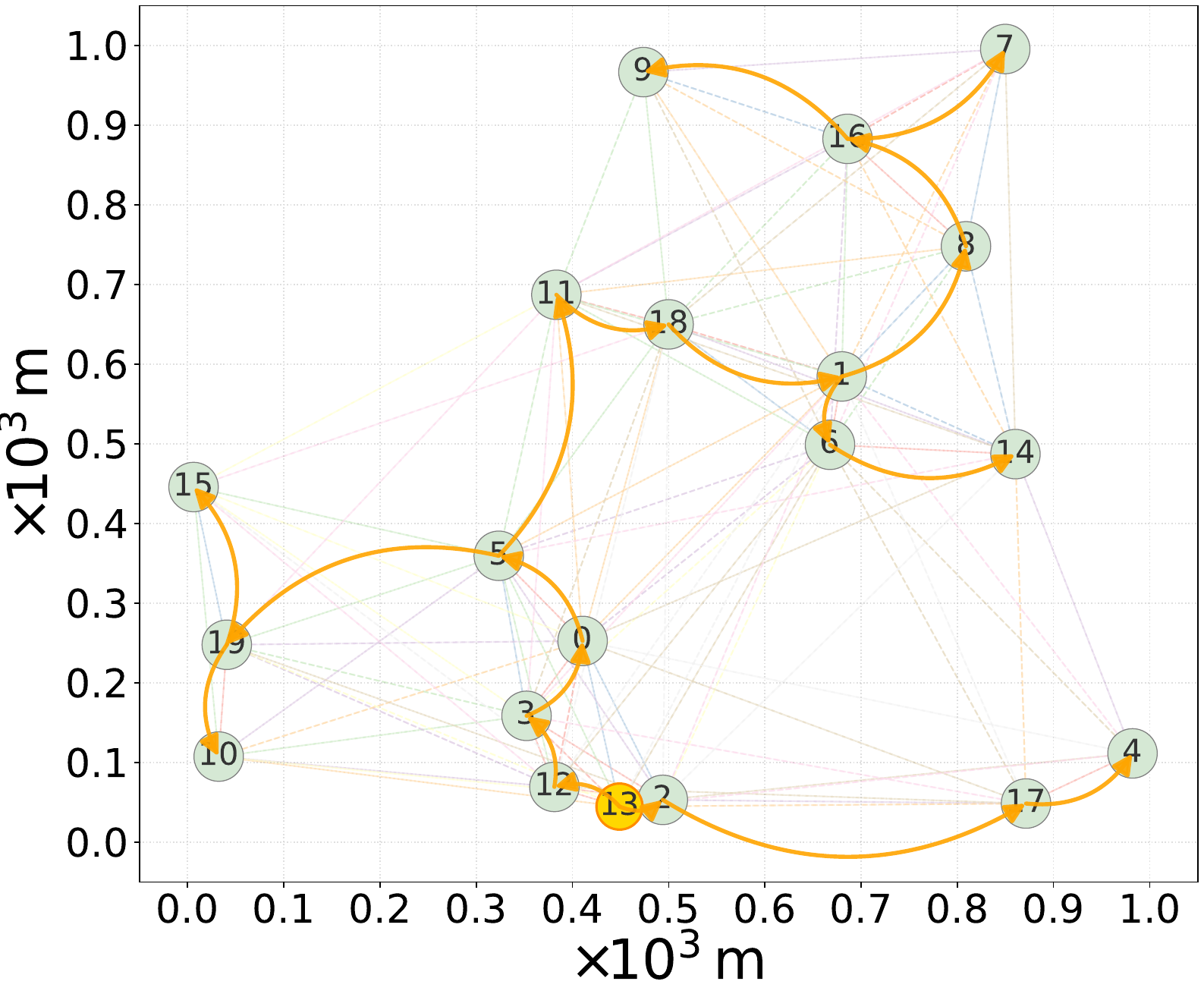}
            \caption*{$\mathcal{T}_m''$} 
        \end{minipage}
        \hfill
        \begin{minipage}{0.32\linewidth}
            \centering
            \includegraphics[width=\linewidth]{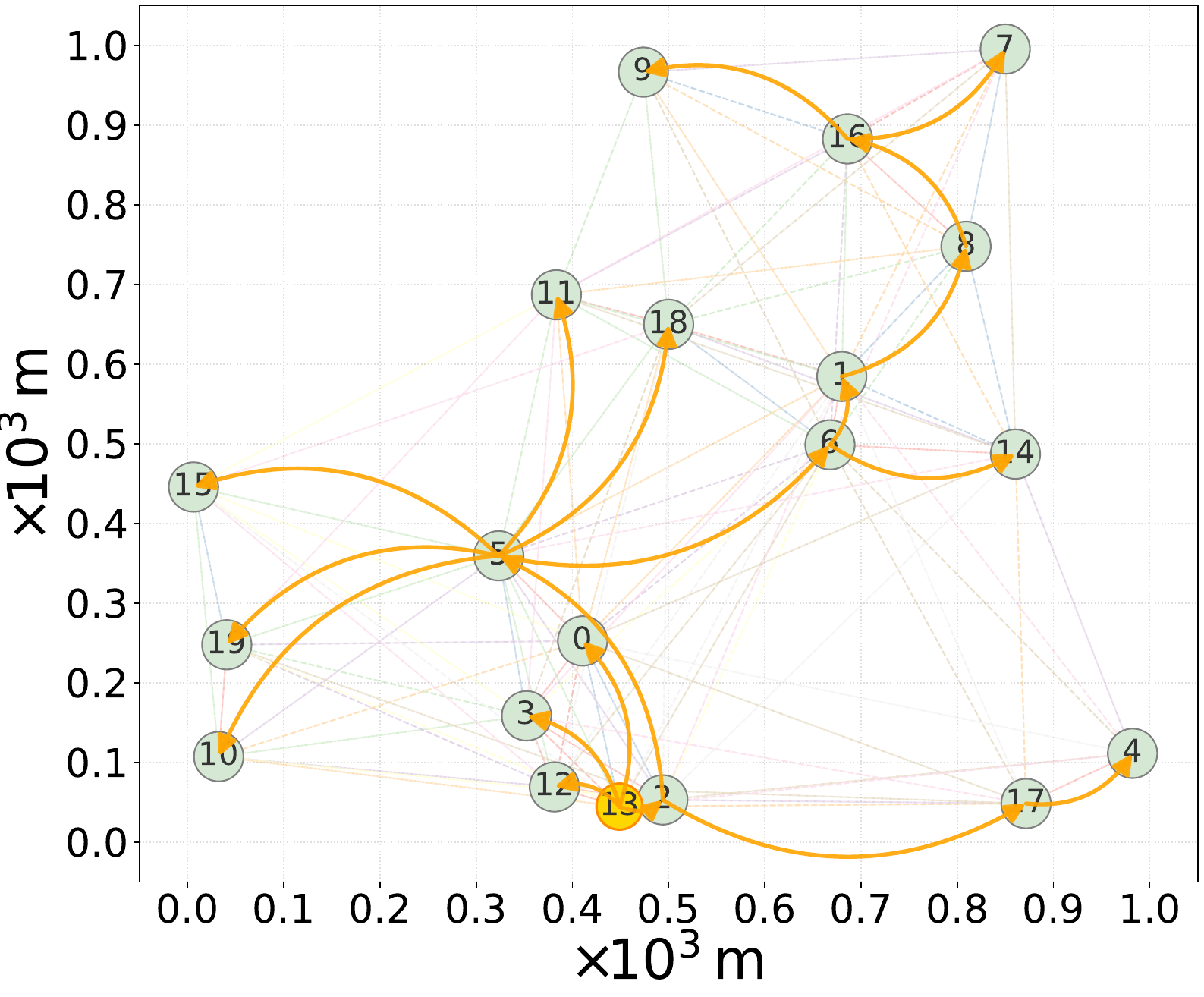}
            \caption*{$\mathcal{T}_m'$}
        \end{minipage}
        \hfill
        \begin{minipage}{0.32\linewidth}
            \centering
            \includegraphics[width=\linewidth]{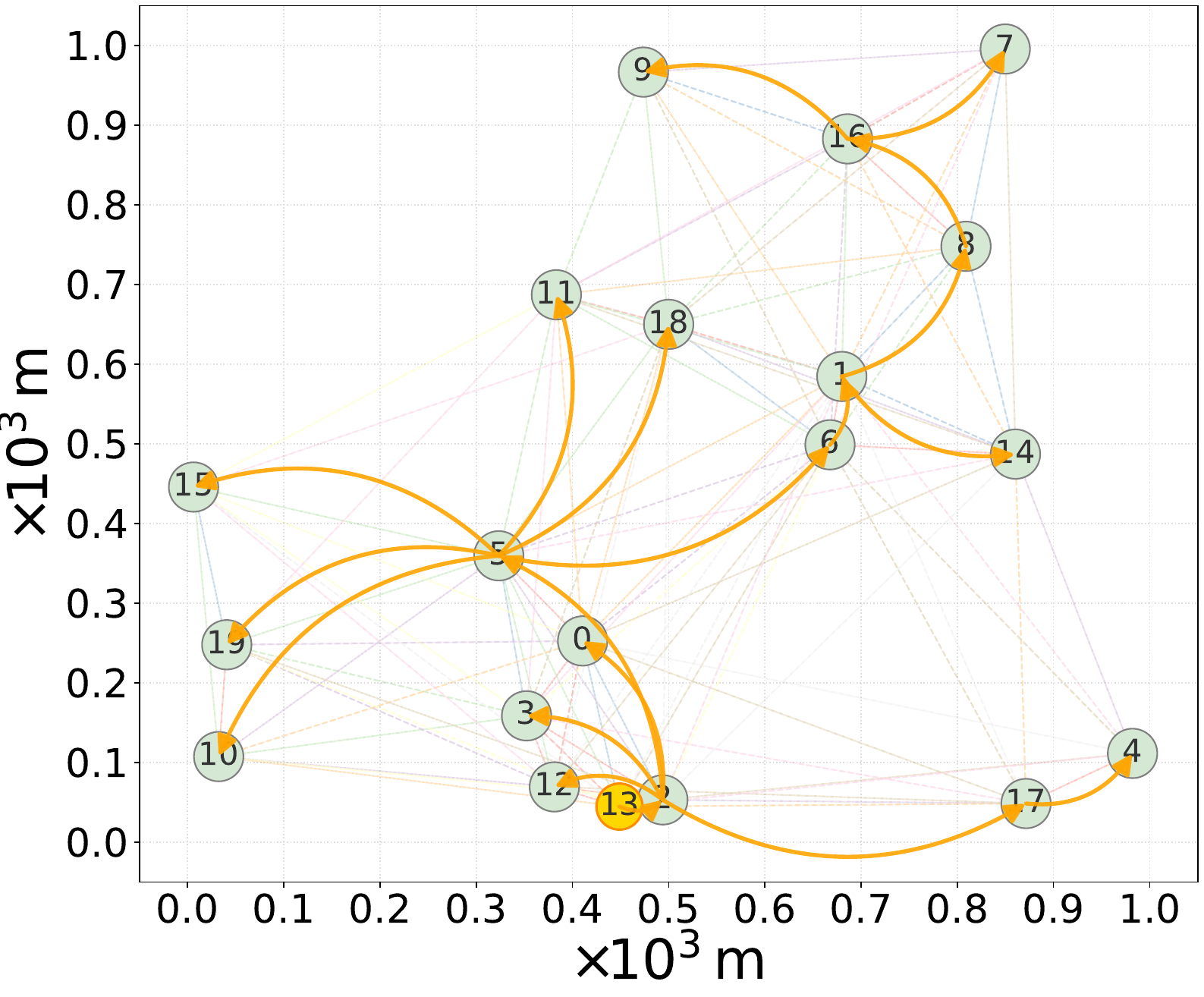}
            \caption*{$\mathcal{T}_m^*$}
        \end{minipage}

        \caption{An illustrative example of the model-transmission routing for client 13 with   $\theta = 0.1$ and $\Psi = 3$.} 
        \label{subfig:tree}
    \end{subfigure}
    \hfill
    \begin{subfigure}[t]{0.46\textwidth}
        \centering
        \includegraphics[width=\linewidth]{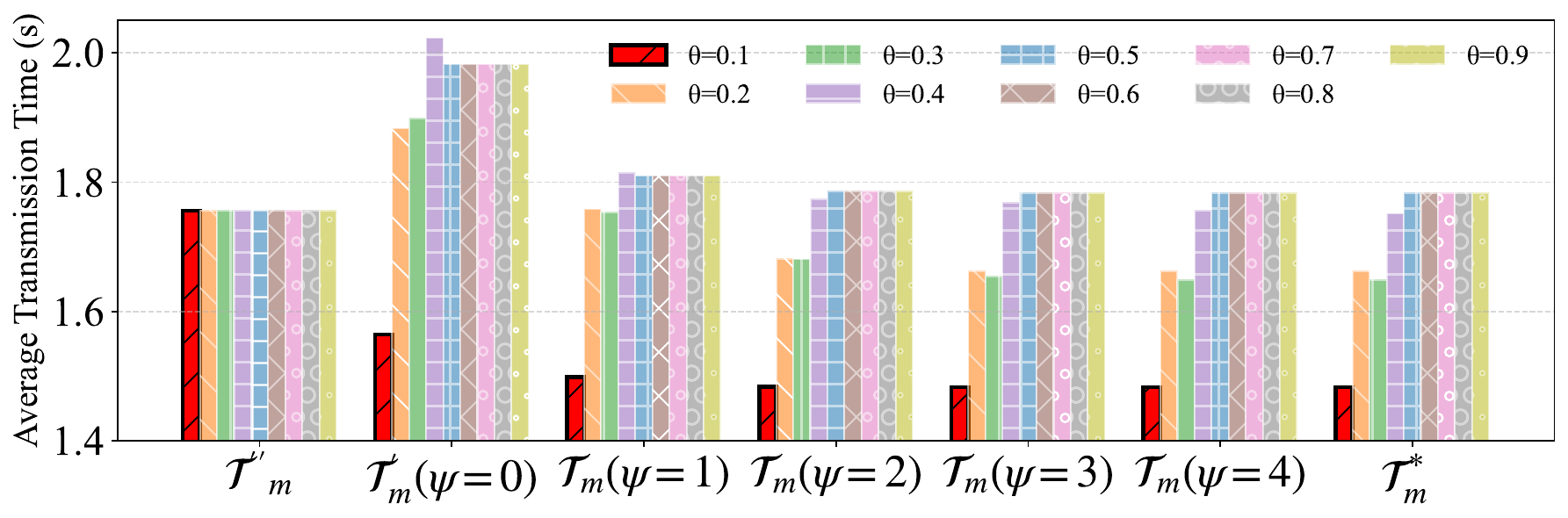}
        \caption{Average model transmission time under different values of $\theta$ and iteration numbers $\Psi$.}
        \label{subfig:parameters}
    \end{subfigure}

    \caption{Illustration of the proposed routing Algorithm \ref{alg:Simplified Optimal Routing} at different optimization stages and the impact of parameters $\theta$ and $\Psi$ on performance ($N = 20$). }\label{fig.alg2}
\end{figure}

\subsection{Enhanced Routing: Design for Addressing Real-Time Multi-Hop Transmission Bottlenecks}\label{sec:Bottlenecks}

Under strict real-time constraints, we introduce two mechanisms on top of P\_CLT to further alleviate bottlenecks in multi-hop communications:
 \paragraph{Congestion Avoidance Mechanism (CAM)} 
This mechanism avoids congestion at bandwidth-limited nodes through an adaptive decision-making process between \emph{model pruning} and \emph{path detouring}. Specifically, the system evaluates two alternative strategies: (i) directly traversing a congested node with an increased pruning ratio, and (ii) bypassing the congested node via a detour path. The strategy that yields a higher \emph{model retention rate} is selected, thereby ensuring stable and continuous model parameter transmission under given communication constraints.

\paragraph{Forwarding-Aware Priority Scheduling and Rerouting Strategy (FPSR)} Critical model parameters are transmitted with priority. When a node reaches its forwarding limit, lower-priority parameters are rerouted to meet real-time requirements while ensuring successful model transmission.

\section{ Numerical Results}\label{sec:experimental_design}
In this section, we evaluate the effectiveness of the proposed D-FL method with model pruning and multi-hop routing through simulations. All experiments use the implementation on GitHub\footnote{\url{[https://github.com/hxysobeautiful/D-FL.git]}}

\subsection{Experimental Setup}
Without loss of generality, we generate a network of $N = 20$ nodes following a random geometric graph model \cite{9895277}. The connectivity density of the edges in the network topology is denoted as $\rho$; in other words, the number of directly connected client pairs is $\rho \times \frac{N(N-1)}{2}$. By default, $\rho = 0.6$. All clients operate at a central radio frequency of $f_c = 2,500 \ \mathrm{MHz}$, with a bandwidth of $B = 30 \ \mathrm{MHz}$, and each client has a transmission power of $P = 20 \ \mathrm{dBm}$. The noise power spectral density  is $N_0 = -174 \ \mathrm{dBm/Hz}$. The maximum transmission delay is \( t_{\max} = 2 \, \text{s} \). The link-weight perception threshold of clients is set to $\theta = 0.1$, the number of iterations is set to $\Psi = 3$.
 Without affecting generality, the channel gain is defined as
\(
h_{\alpha,(i,j)}^2 = \left( \frac{\lambda}{4\pi d_{i,j} f_c} \right)^2
\),~\cite{rappaport2024wireless}
where $d_{i,j}$ (km) is the distance between neighboring clients $i$ and $j$, 
and $\lambda$ is a constant.
According to Shannon's theorem, the transmission rate of a direct link can be evaluated as
\(v_{i,j} = B \log_2 \left( 1 + \gamma_{i,j} \right),\)
where $\gamma_{i,j}$ is the signal-to-noise ratio (SNR) between neighboring clients $i$ and $j$, and is given by
\(
\gamma_{i,j} = \frac{h_{\alpha,(i,j)}^2 P}{N_0 B}.
\) 

\vspace{2mm}
We consider two image classification tasks: \textbf{F-CIFAR10 with ResNet-18} and \textbf{F-CIFAR100 with ResNet-34}.
The CIFAR-10 and CIFAR-100 datasets consist of 10 and 100 classes, respectively, with each class containing 500 training samples and 100 testing samples. To satisfy the IID assumption, the training data are randomly shuffled before being distributed to the clients. Regarding the model architectures, ResNet-18 consists of 17 convolutional layers and one fully-connected layer (totaling $K=11.69 \times 10^6$ parameters), while ResNet-34 comprises 33 convolutional layers and one fully-connected layer (totaling $K=21.8 \times 10^6$ parameters). For both tasks, each model parameter is represented as a 32-bit floating-point number (4 bytes). Local training is performed using the SGD optimizer with momentum. The key hyperparameters are kept consistent across both scenarios: the learning rate is set to $\mu = 0.1$, the batch size is 16, and each client performs 3 local training epochs per communication round. The entire training process consists of 2000 communication rounds to ensure convergence.

\vspace{2mm}
We compare the following two D-FL frameworks with different aggregation strategies:

$\bullet$ \textbf{P2P-based D-FL}~\cite{10304208}: Clients exchange and aggregate models only with single-hop neighbors. 

$\bullet$ \textbf{Multi-hop routing (This study)}: Clients exchange and aggregate models via multi-hop routing.

\vspace{2mm}
To validate the effectiveness of the proposed pruning method, we use Algorithm~\ref{alg:Simplified Optimal Routing} to determine the model transmission path for each client (denoted as $\mathcal{T}_m^*$ and $\mathcal{I}_m^*$), and compare our approach with two commonly used pruning strategies:

\begin{itemize}
  \item \textbf{Optimal Pruning (Proposed)}:Both the model retention rate $r_m^*$ and routing are optimized jointly as described in Section \ref{sec:pruning and routing}.

  \item \textbf{Fixed Pruning} ~\cite{9713700}: A fixed model retention rate is preset. The performance is evaluated under three retention rates: $r_m = 0.60$, $0.85$, and $0.95$.
  
  \item \textbf{No Pruning} ~\cite{10965802}: The full model is retained, i.e., $r_m = 1$.
\end{itemize}

To more comprehensively evaluate the proposed routing algorithm (P\_CLT) under the optimal pruning strategy, we introduce a stricter baseline in which each client’s model retention rate is adaptively determined based on its routing path and link rate, while conventional routing protocols are employed.

Specifically, we consider the following representative routing schemes:


\begin{itemize}
\item \textbf{Flood Fill} \cite{9734208}: Broadcasts the model to all neighboring clients.

\item \textbf{Bellman} \cite{10760521}: Uses the Bellman-Ford algorithm to construct a shortestinging-path tree for model transmission.

\item \textbf{Kruskal} \cite{10726020}: Connects clients via a minimum spanning tree to minimize total link weight.
\end{itemize}

\noindent Additionally, we design five variants of the P\_CLT algorithm for ablation studies:

\begin{itemize}
\item \textbf{NP\_CLT}: Without node priority design.

\item \textbf{P\_NCLT}: Without link threshold design.

\item \textbf{NP\_NCLT}: Without both node priority and link threshold designs.

\item \textbf{P\_CLT with \eqref{eq:condition_theta}}: Only uses condition \eqref{eq:condition_theta}.

\item \textbf{P\_CLT with \eqref{eq:condition_w}}: Only uses condition \eqref{eq:condition_w}.
\end{itemize}

\begin{figure}[ht]
    \centering
    \includegraphics[width=0.98\linewidth]{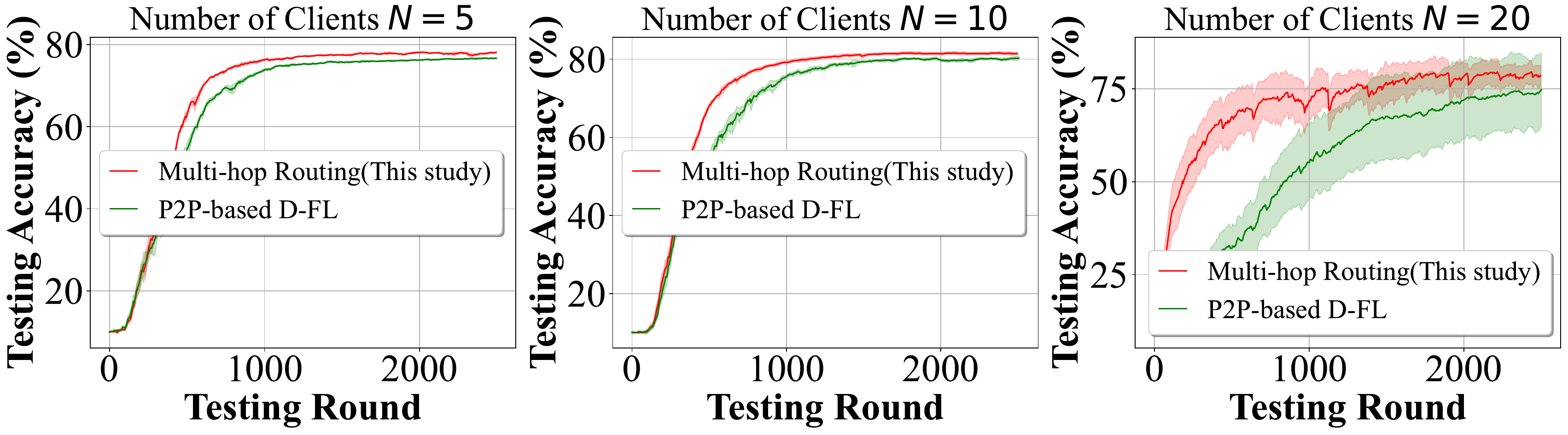} 
    \caption{Testing accuracy comparison of multi-hop and P2P aggregation strategies on \textbf{F-CIFAR10 \& ResNet-18} across different numbers of clients.}
    \label{fig:combined}
\end{figure}

To evaluate the performance of the proposed \textbf{Enhanced Routing} scheme ( see Section \ref{sec:Bottlenecks}) under communication-constrained scenarios, experiments are conducted in a multi-hop network consisting of bandwidth-limited nodes ( $\mathcal{V}_{\mathrm{bw}}=\{0,17\}$ ) and forwarding-capacity-limited nodes ( $\mathcal{V}_{\mathrm{obs}}=\{2,5,16\}$ ). Each forwarding-capacity-limited node participates in at most six model forwarding rounds, while each bandwidth-limited node can transmit up to $0.8 \times K$ model parameters per round.
The proposed scheme is compared with three baselines: 
\begin{itemize}

  \item \textbf{Bandwidth-Only}: optimizes only the bandwidth-limited nodes (see Section~\ref{sec:Bottlenecks}(a)).
    \item \textbf{Forwarding-Only}: optimizes only the forwarding-capacity-limited nodes (see Section~\ref{sec:Bottlenecks}(b)).
  
    \item \textbf{Baseline (P\_CLT)}: no optimization is applied.
\end{itemize}

\begin{figure}[ht]
  \centering
  \begin{subfigure}[b]{0.11\textwidth}
    \centering
    \includegraphics[height=2cm,width=\textwidth]{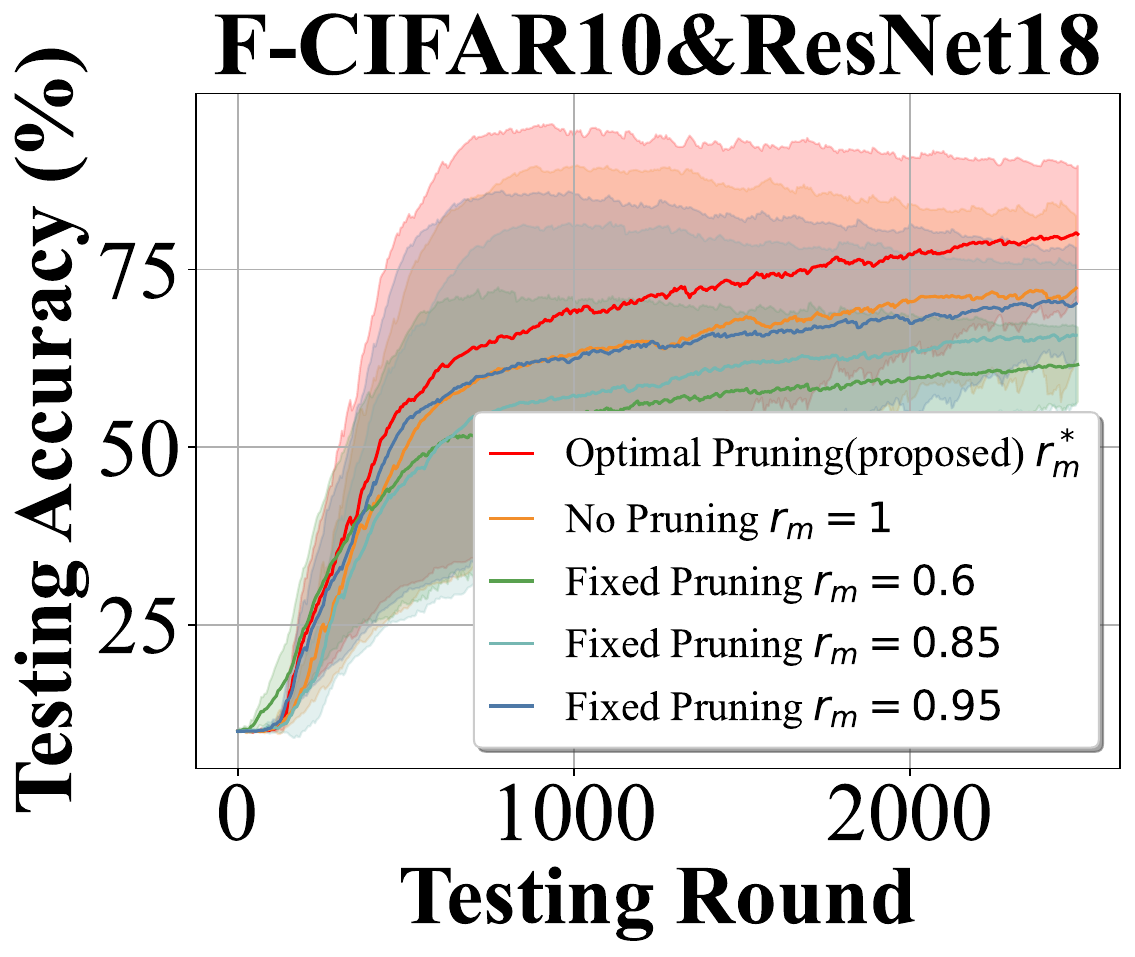}
    \caption{}
  \end{subfigure}
  \hfill
  \begin{subfigure}[b]{0.11\textwidth}
    \centering
    \includegraphics[height=2cm,width=\textwidth]{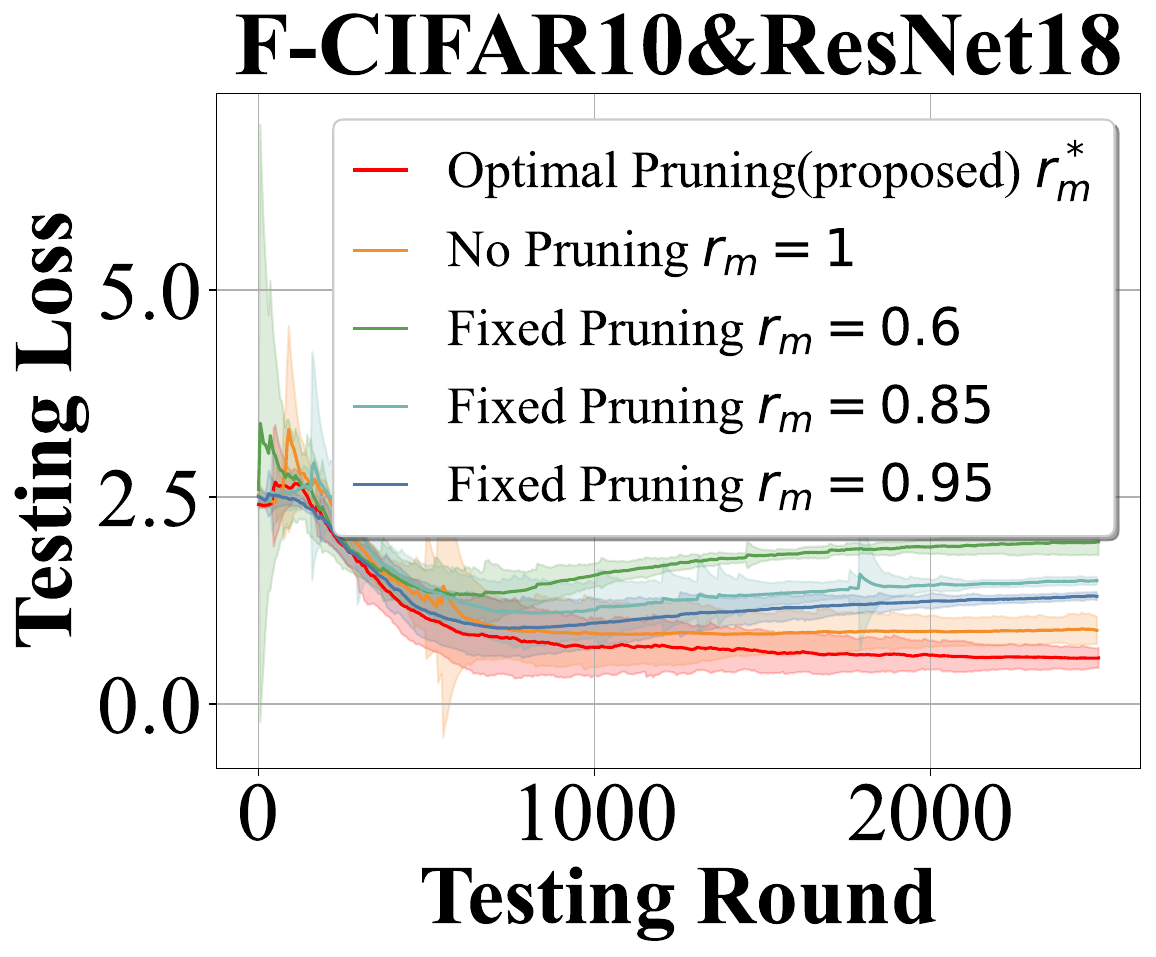}
    \caption{}
  \end{subfigure}
  \hfill
  \begin{subfigure}[b]{0.11\textwidth}
    \centering
    \includegraphics[height=2cm,width=\textwidth]{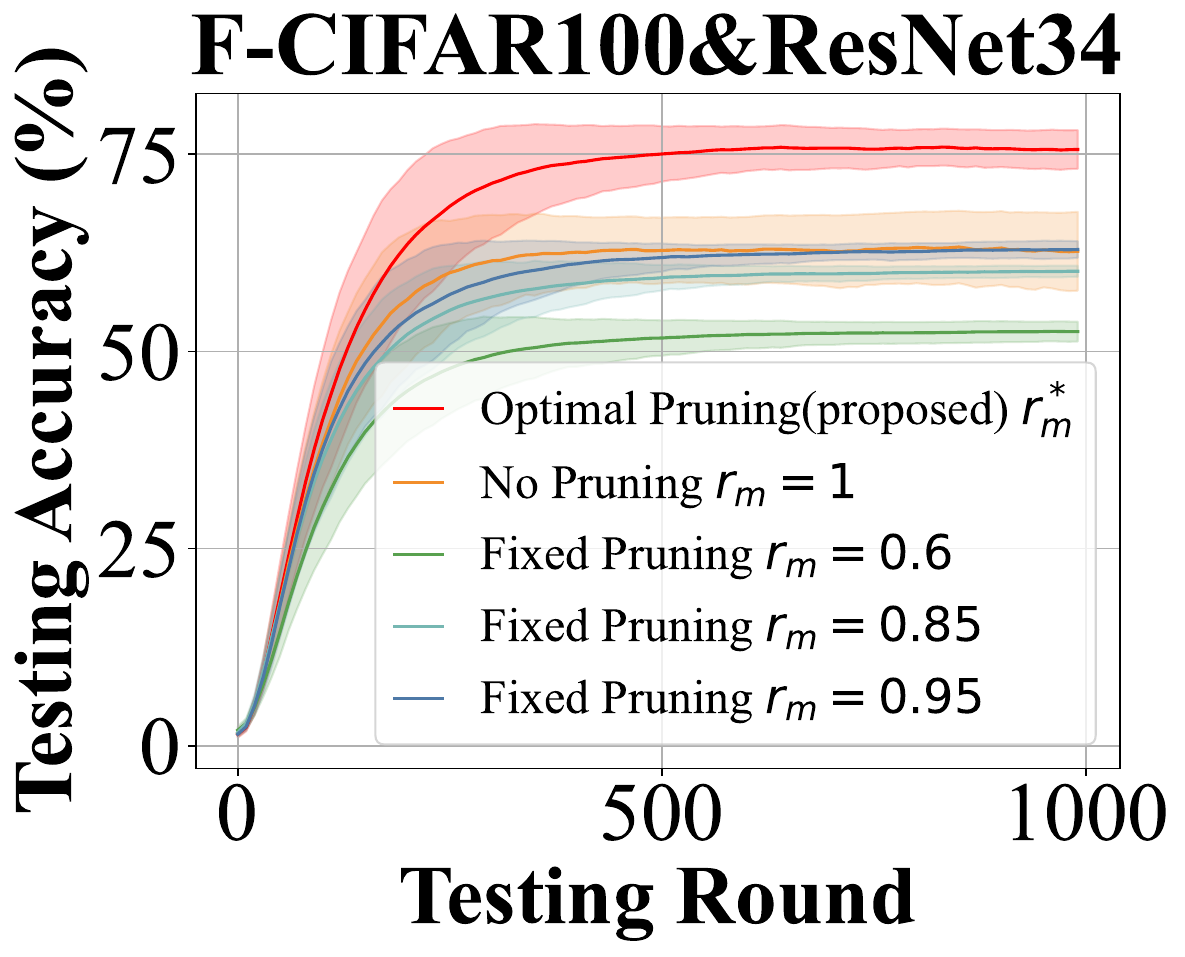}
    \caption{}
  \end{subfigure}
  \hfill
  \begin{subfigure}[b]{0.11\textwidth}
    \centering
    \includegraphics[height=2cm,width=\textwidth]{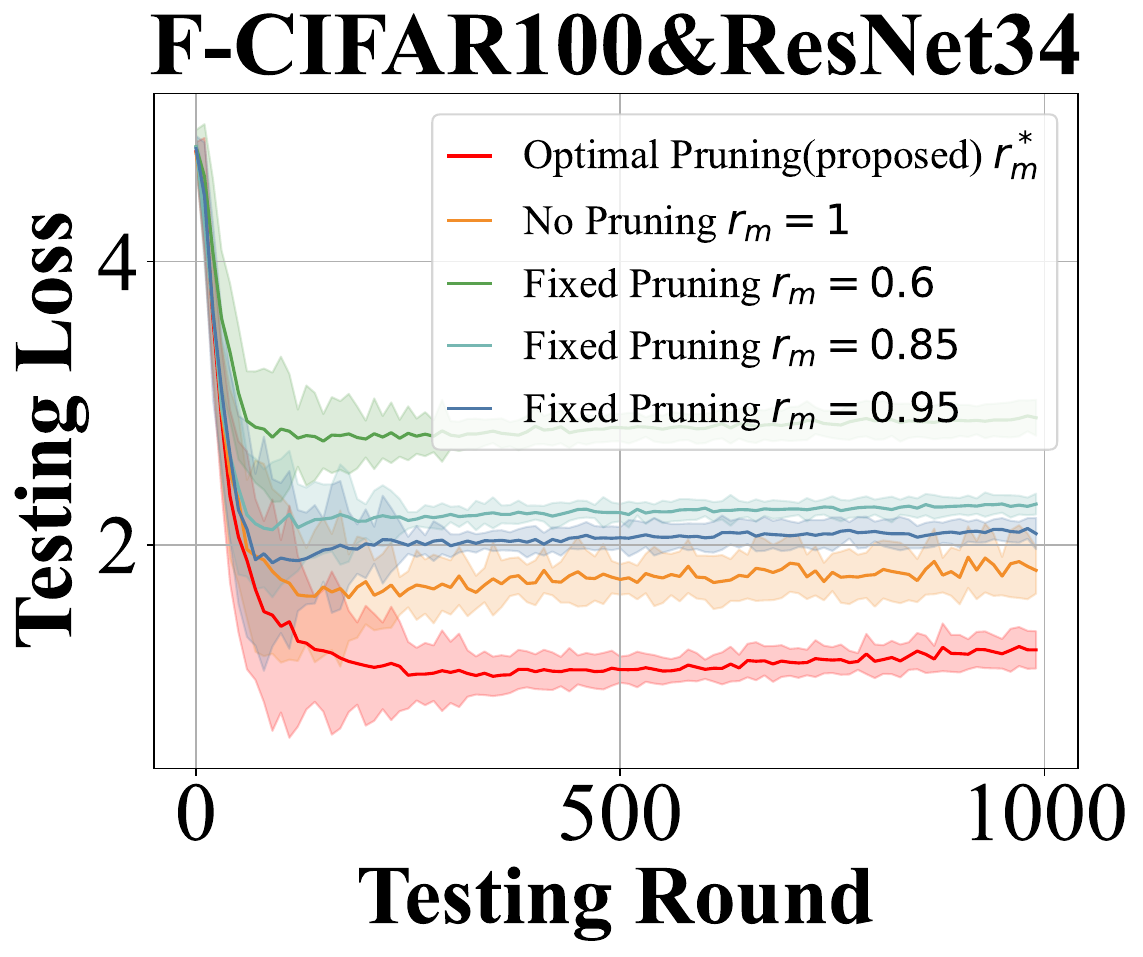}
    \caption{}
  \end{subfigure}

  \begin{subfigure}[h]{0.5\textwidth}
    \includegraphics[width=\linewidth,valign=t]{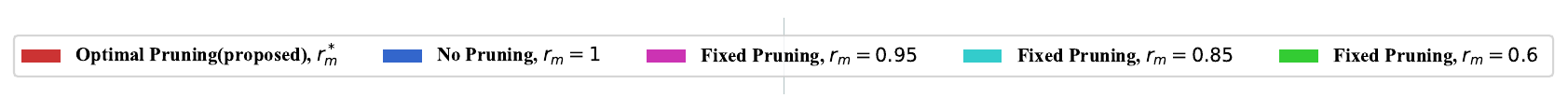}
    \label{subfigure1}
  \end{subfigure}
  \\
  \begin{subfigure}[b]{0.5\textwidth}  
    \includegraphics[width=\textwidth]{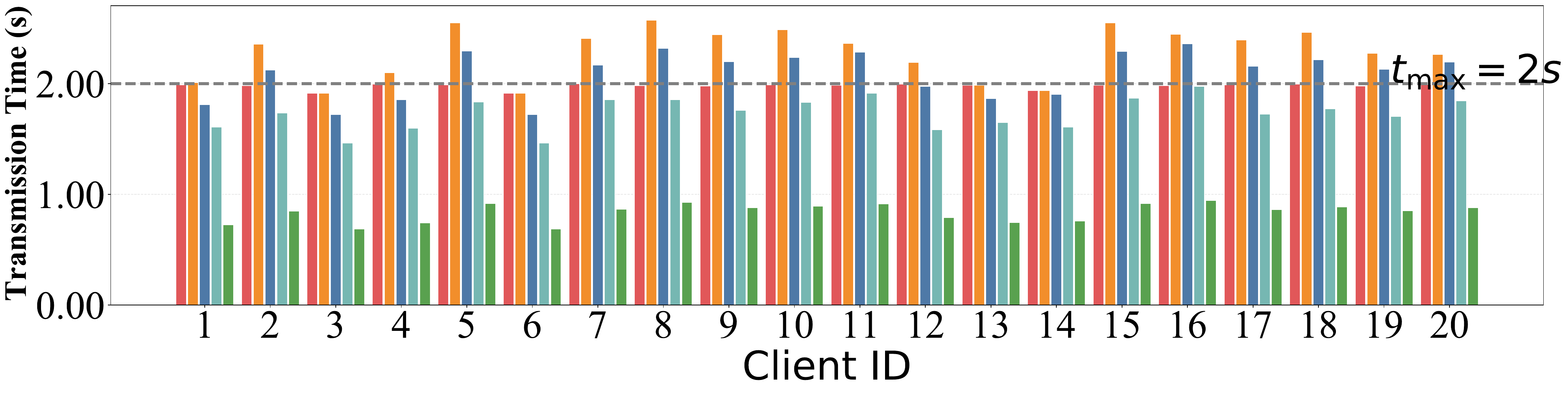}
    \caption{}
  \end{subfigure}

   \caption{(a)-(b) Testing accuracy and loss of  different pruning schemes on the F-CIFAR10 \& ResNet-18;  
(c)-(d)Testing accuracy and loss of different pruning schemes on the F-CIFAR100 \& ResNet-34;  
(e) Distribution of per-client model transmission time under different pruning schemes on F-CIFAR10 \& ResNet-18.
}
  \label{fig.pruning}
\end{figure}

\subsection{Evaluation of D-FL Model Aggregation Strategies}

Fig.~\ref{fig:combined} compares the testing accuracy of different aggregation strategies under varying numbers of clients on the F-CIFAR10 \& ResNet-18. The results indicate that the multi-hop aggregation strategy significantly outperforms the P2P aggregation strategy. This is because P2P allows information exchange only between directly connected clients, which can lead to model bias for clients with low connectivity. In contrast, the proposed multi-hop routing mechanism enables communication through predefined multi-hop paths, effectively expanding the model exchange range, enhancing cross-region information sharing, and mitigating local errors caused by sparse connectivity.

\subsection{Evaluation of the Proposed  Model Pruning Scheme}

\begin{figure}[t]  
  \centering
 
  \begin{subfigure}[b]{0.11\textwidth}
    \centering
    \includegraphics[height=2cm,width=\textwidth]{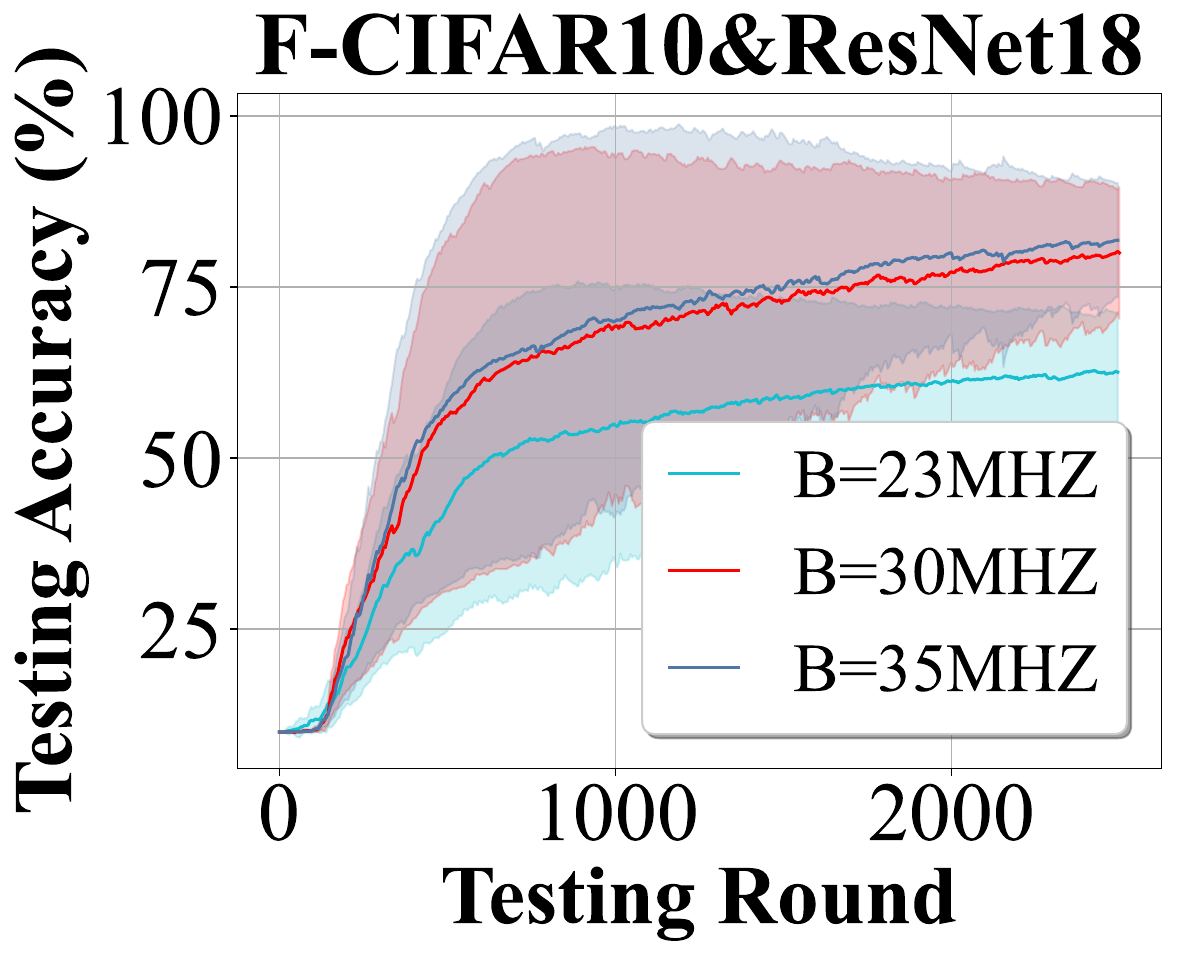}
    \caption{}
  \end{subfigure}
  \begin{subfigure}[b]{0.11\textwidth}
    \centering
    \includegraphics[height=2cm,width=\textwidth]{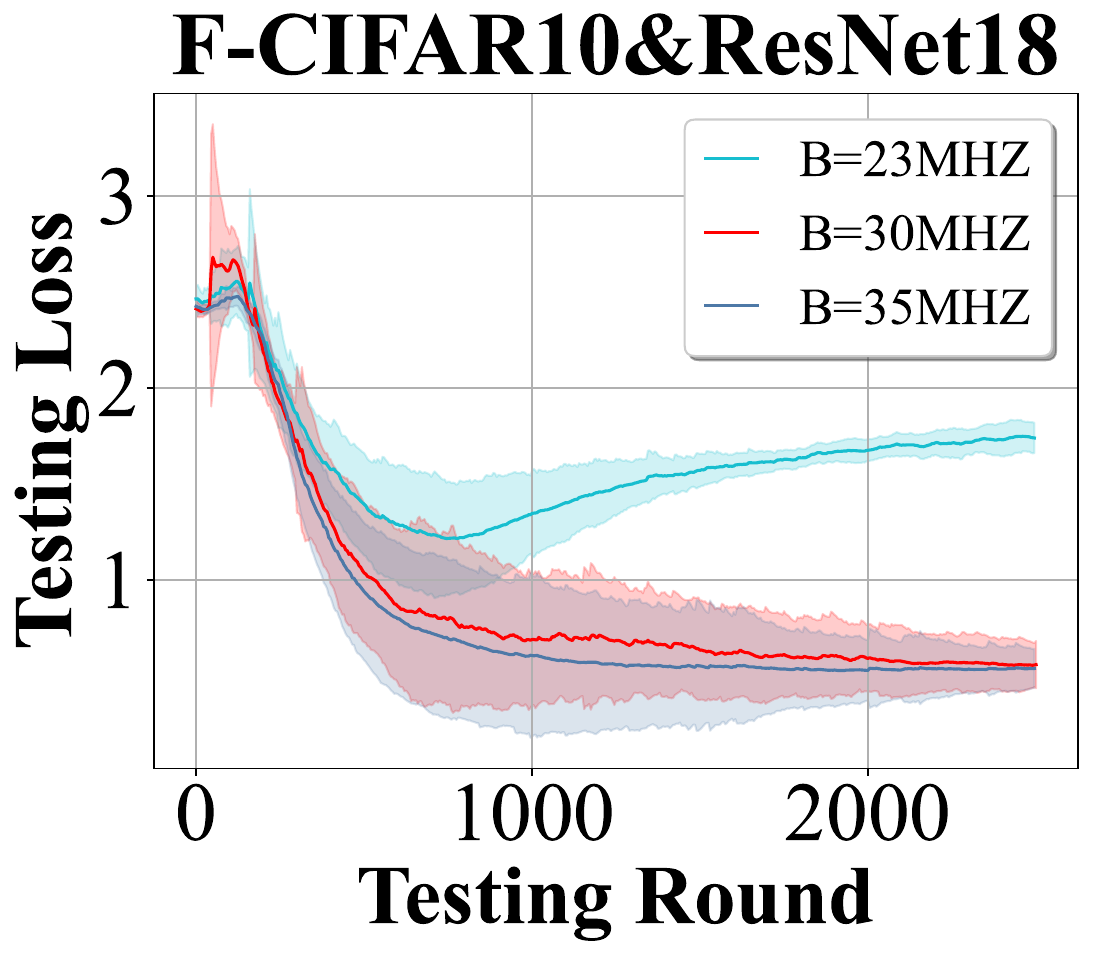}
\caption{}
  \end{subfigure}
  \begin{subfigure}[b]{0.11\textwidth}
    \centering
    \includegraphics[height=2cm,width=\textwidth]{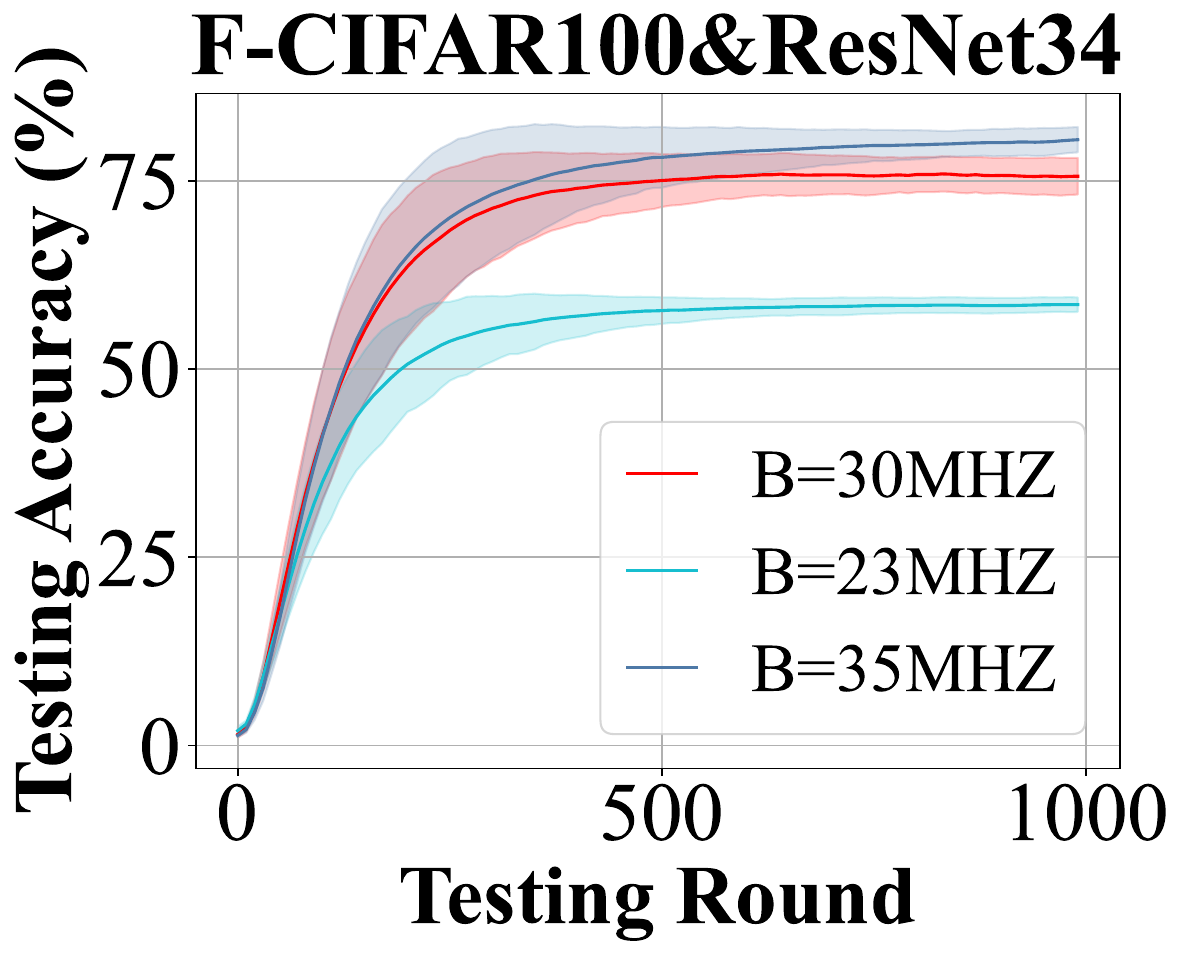}
  \caption{}
  \end{subfigure}
  \begin{subfigure}[b]{0.11\textwidth}
    \centering
    \includegraphics[height=2cm,width=\textwidth]{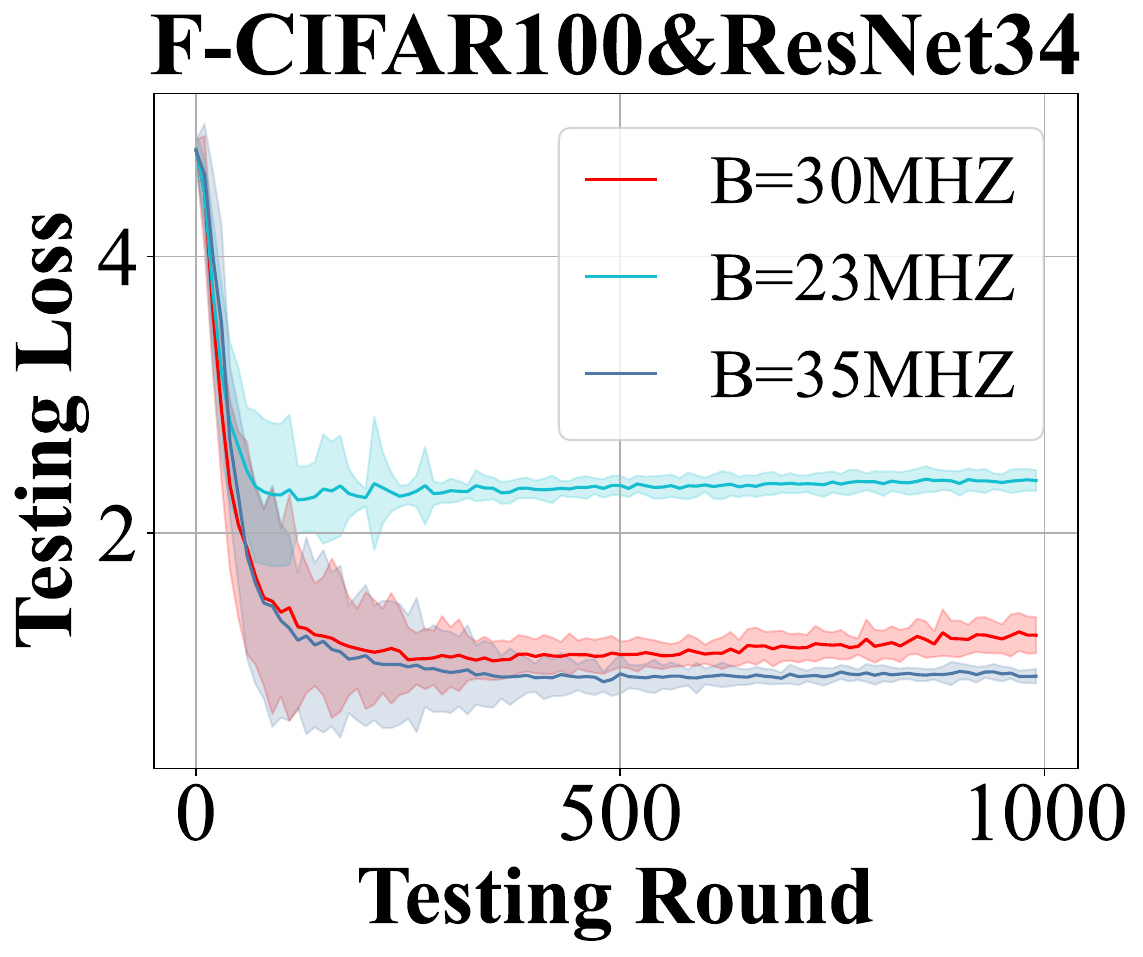}
\caption{}
  \end{subfigure}
    \begin{subfigure}[h]{0.46\textwidth}
    \includegraphics[width=\linewidth,valign=t]{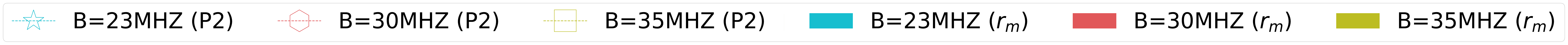}
\end{subfigure}\\
\begin{subfigure}[b]{0.46\textwidth} 
\includegraphics[width=\textwidth]{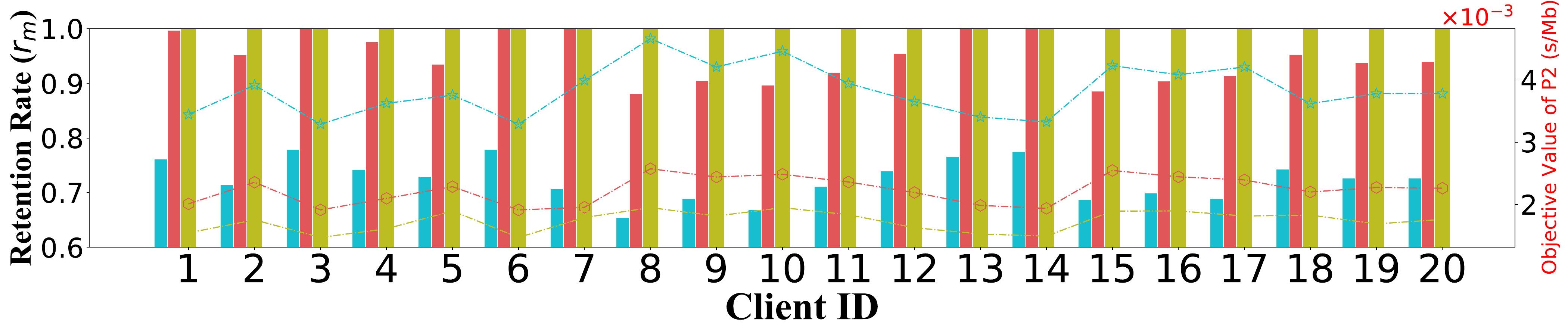} 
  \caption{}
  \end{subfigure}
\caption{(a)--(b) Testing accuracy and loss of the proposed joint pruning and routing method on F-CIFAR10 \& ResNet-18 under varying bandwidths;  
(c)--(d) Testing accuracy and loss on F-CIFAR100 \& ResNet-34 under varying bandwidths;  
(e) Distribution of model retention rates and cumulative maximum link weights per hop across 20 clients under varying bandwidths.}
\label{fig.B}
\end{figure}

\begin{figure}[ht] 
  \centering
  \begin{subfigure}[b]{0.11\textwidth}
    \centering
    \includegraphics[height=2cm,width=\textwidth]{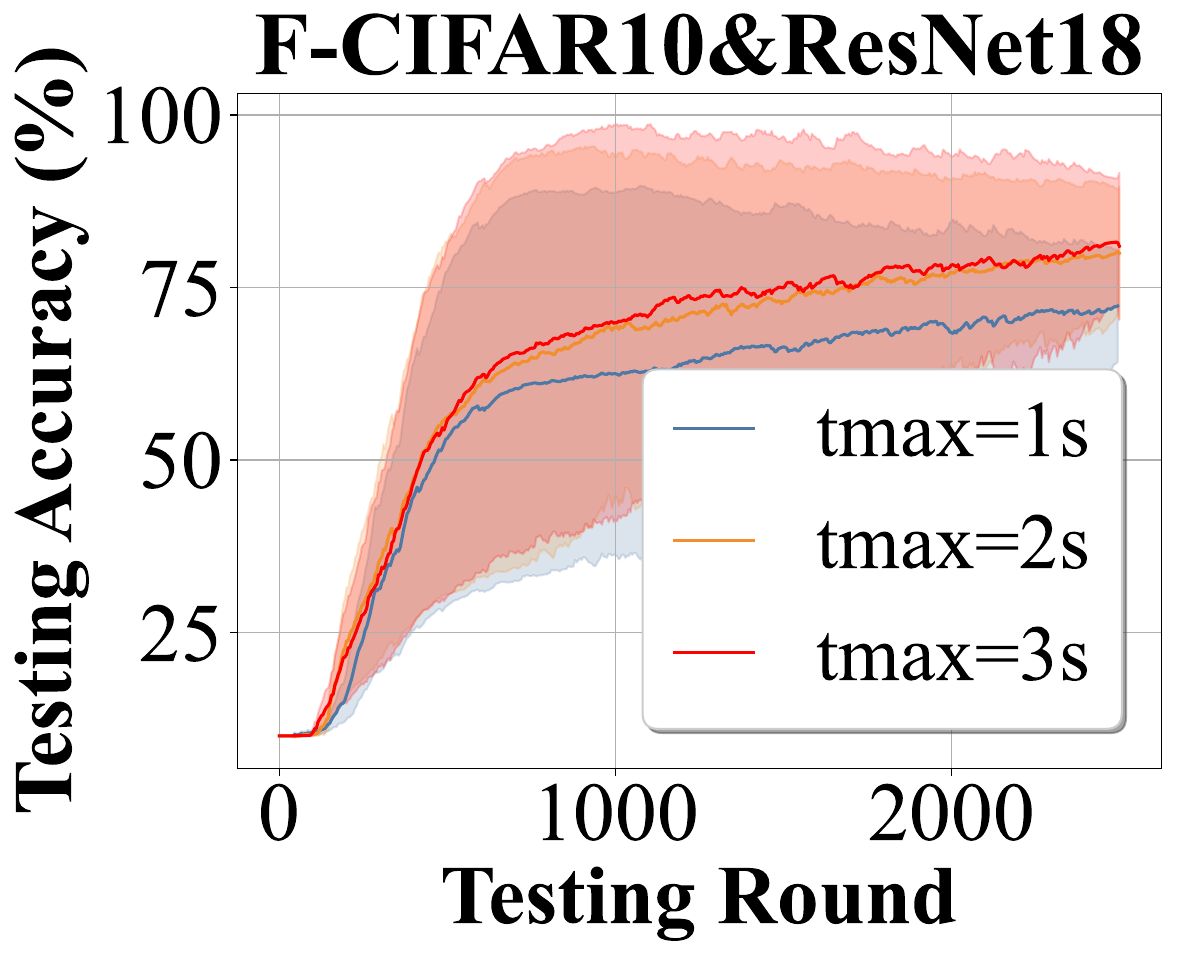}
    \caption{}
  \end{subfigure}
  \begin{subfigure}[b]{0.11\textwidth}
    \centering
    \includegraphics[height=2cm,width=\textwidth]{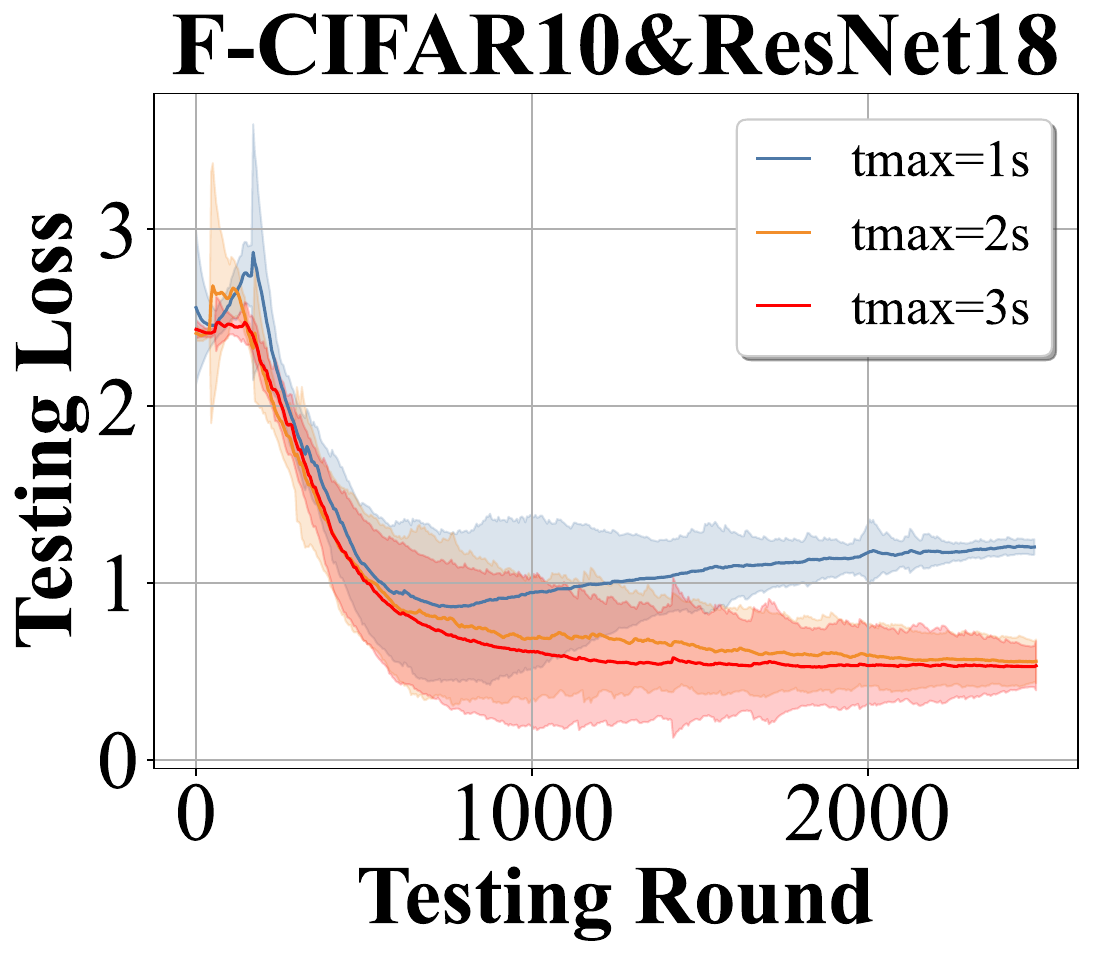}
    \caption{}
  \end{subfigure}
  \begin{subfigure}[b]{0.11\textwidth}
    \centering
    \includegraphics[height=2cm,width=\textwidth]{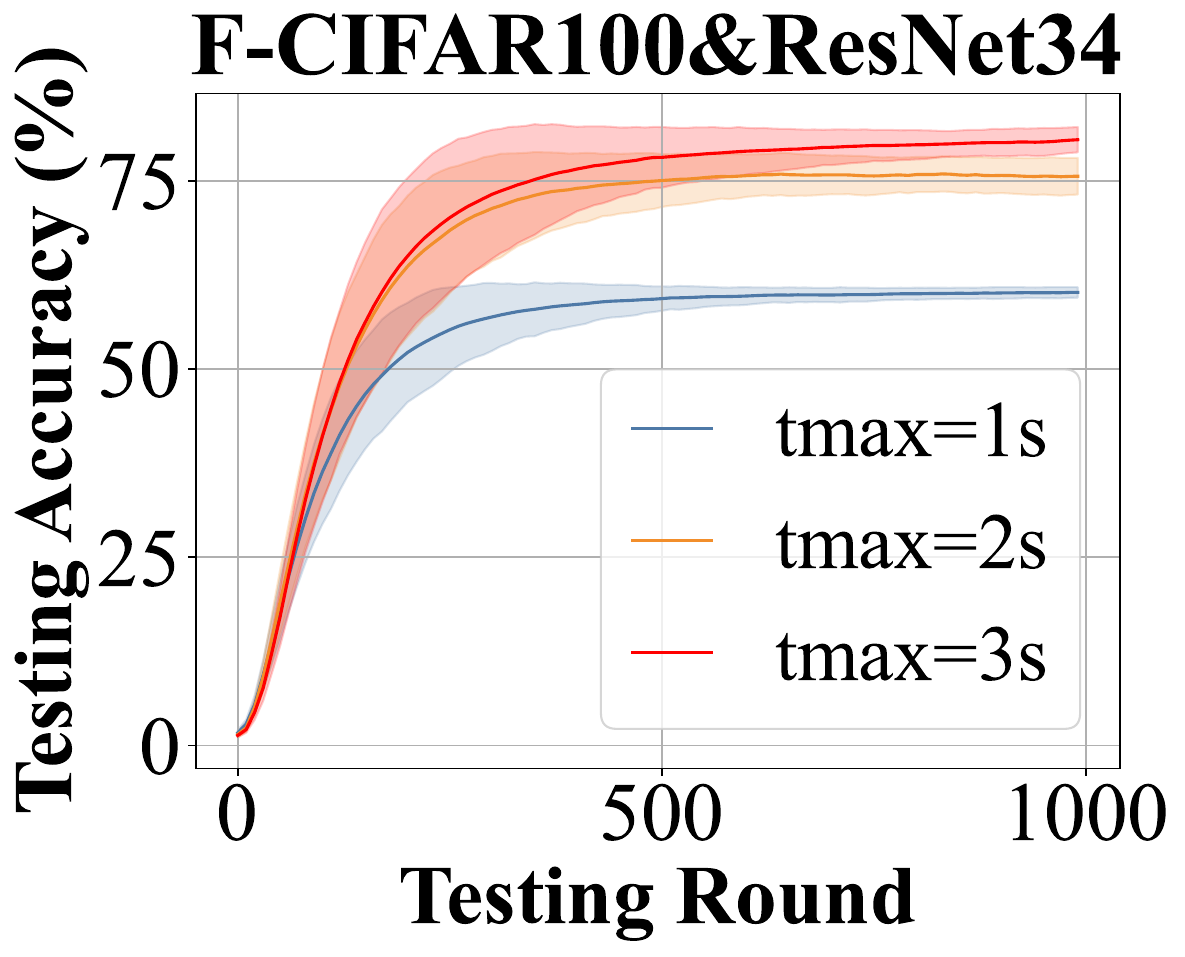}
    \caption{}
  \end{subfigure}
  \begin{subfigure}[b]{0.11\textwidth}
    \centering
    \includegraphics[height=2cm,width=\textwidth]{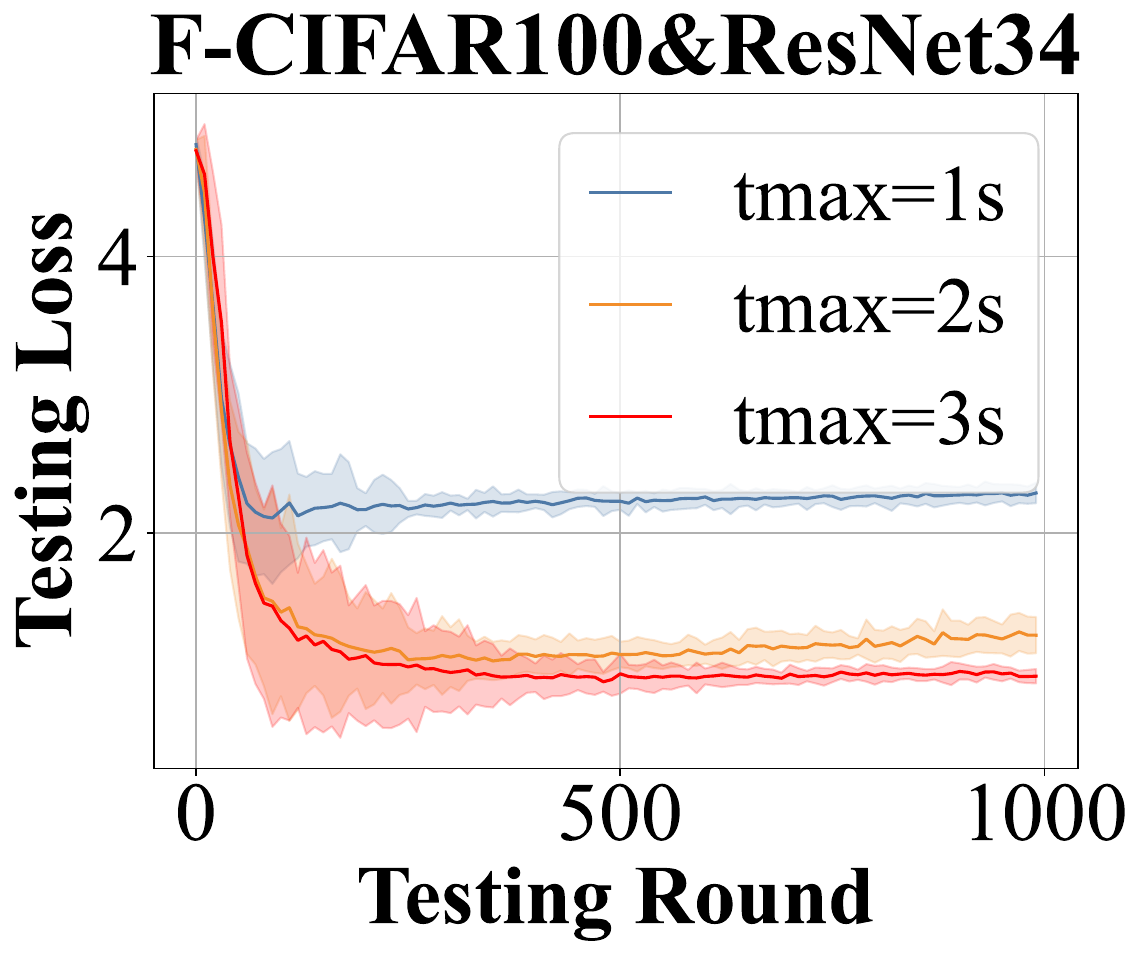}
    \caption{}
  \end{subfigure}

      \begin{subfigure}[h]{0.5\textwidth}
    \includegraphics[width=\linewidth,valign=t]{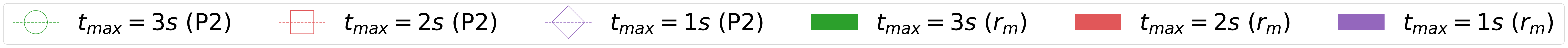}
    \label{subfigure1}
\end{subfigure}
\begin{subfigure}[b]{0.5\textwidth} 
    \centering
    \includegraphics[width=\textwidth]{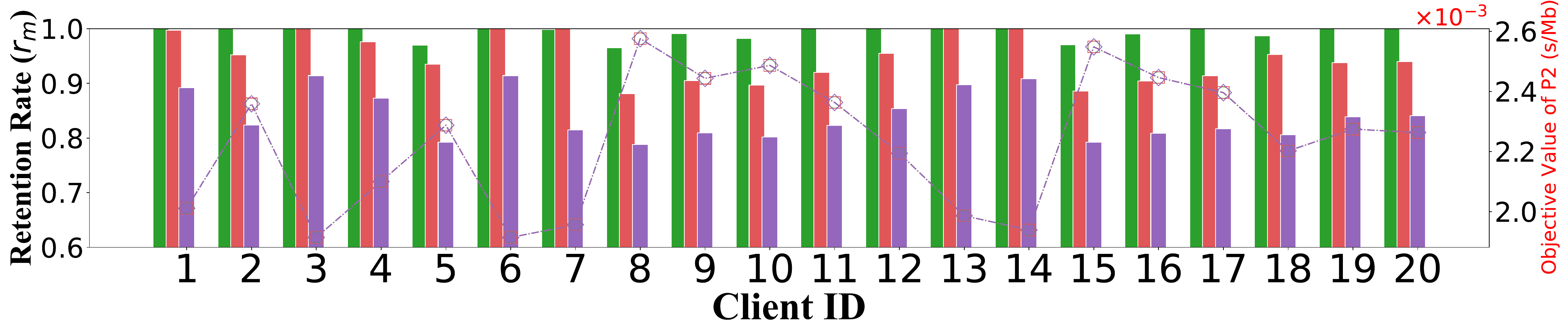}  
    \caption{}
  \end{subfigure}

\caption{(a)--(b) Testing accuracy and loss of the proposed joint pruning and routing method on F-CIFAR10 \& ResNet-18 under different latency thresholds;  
(c)--(d) Testing accuracy and loss on F-CIFAR100 \& ResNet-34 under different latency thresholds;  
(e) Distribution of model retention rates and cumulative maximum link weights per hop across 20 clients under different latency thresholds.}

 \label{fig.FL_T}
\end{figure}

\begin{figure}[t]
    \centering
    \begin{subfigure}[b]{0.32\linewidth}
        \centering
        \includegraphics[width=\linewidth]{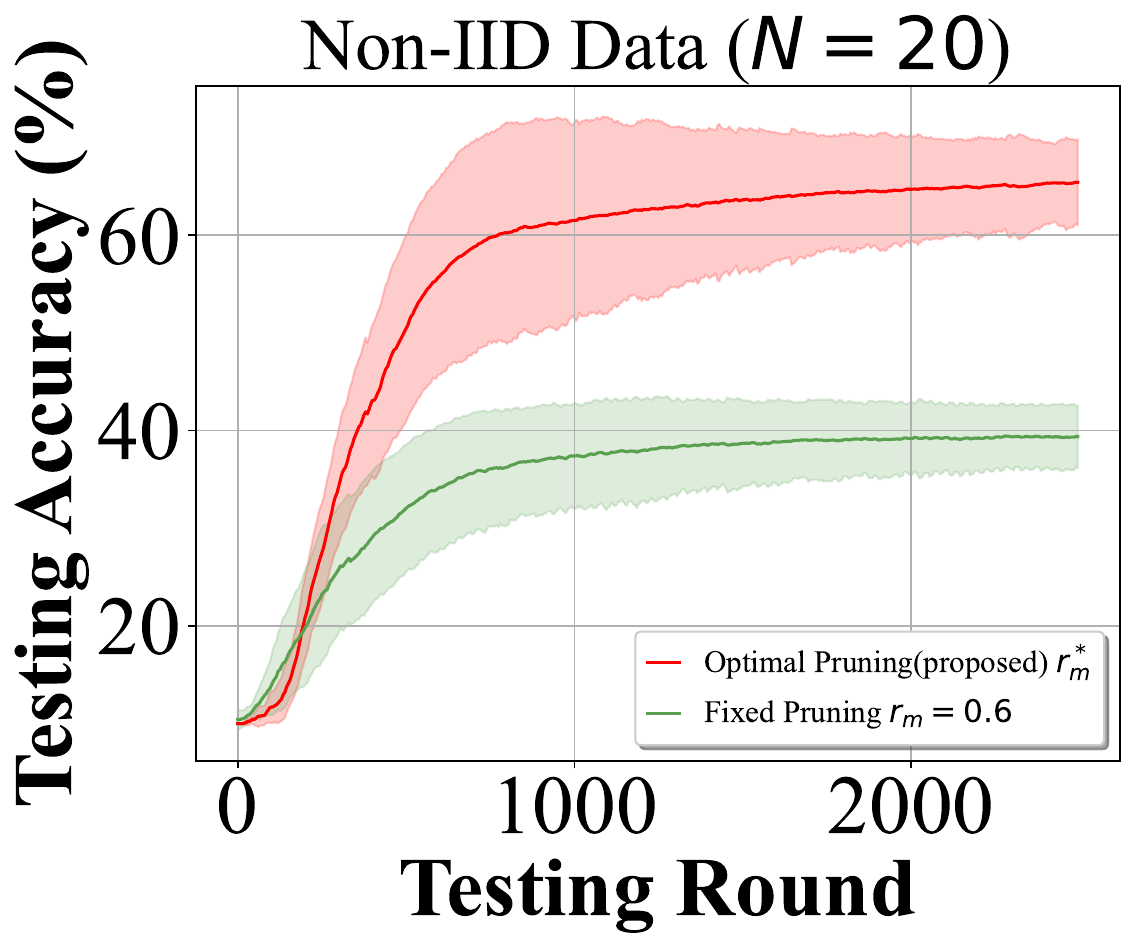}
        \caption{}
    \end{subfigure}
    \hfill
    \begin{subfigure}[b]{0.32\linewidth}
        \centering
        \includegraphics[width=\linewidth]{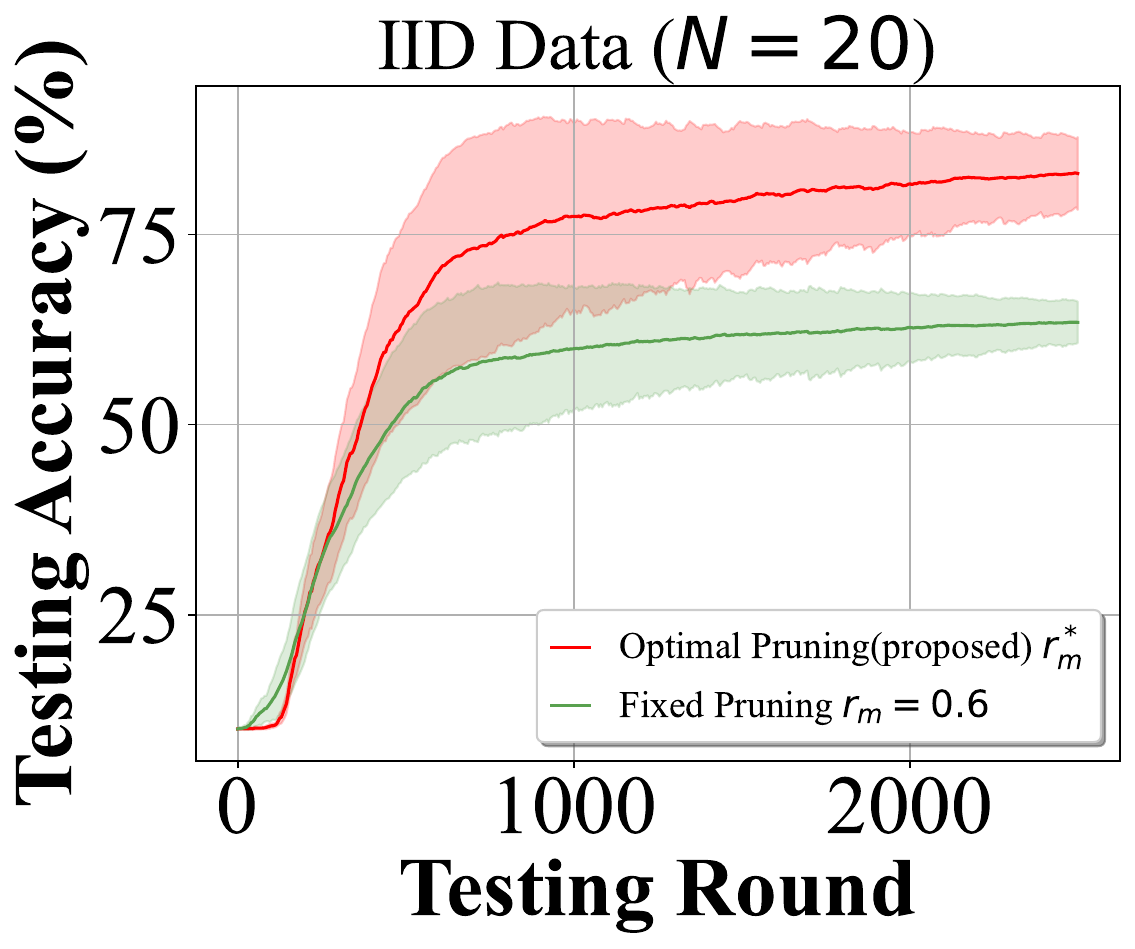}
        \caption{}
    \end{subfigure}
    \hfill
    \begin{subfigure}[b]{0.32\linewidth}
        \centering
        \includegraphics[width=\linewidth]{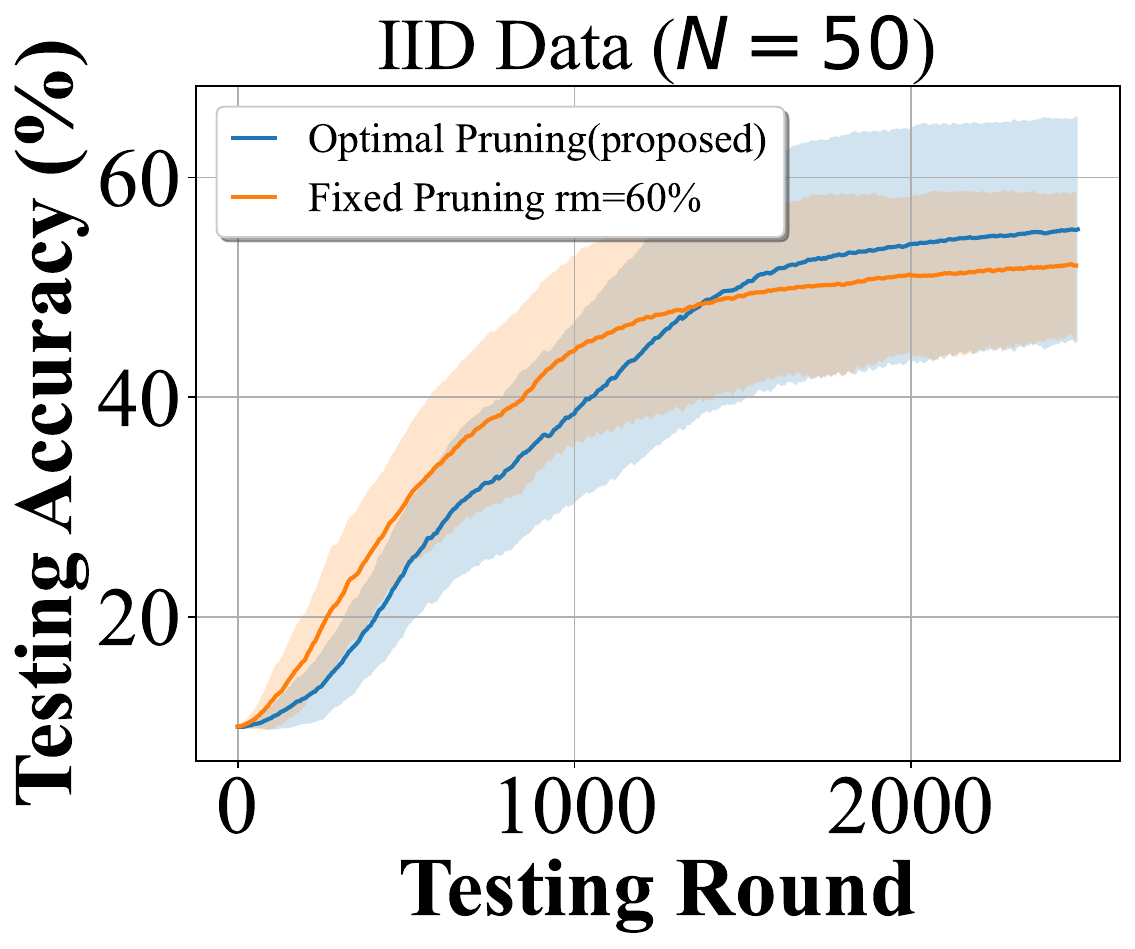}
        \caption{}
    \end{subfigure}

    \caption{Comparison of the different pruning strategies in terms of test accuracy of ResNet18 \& F-CIFAR10 under different client numbers and data distributions: (a) Non-IID with 20 clients; (b) IID with 20 clients; (b) IID with 50 clients
    }
    \label{fig:FL_IID}
\end{figure}

Figs.~\ref{fig.pruning}(a)--(b) and \ref{fig.pruning}(c)--(d) illustrate the testing accuracy and loss under different pruning schemes on the F-CIFAR10 \& ResNet-18 and the F-CIFAR100 \& ResNet-34, respectively. Fig.~\ref{fig.pruning}(e) shows the time required for each client to complete model transmission tasks on F-CIFAR10 with ResNet-18. The experimental results demonstrate that the proposed optimal pruning scheme consistently outperforms other pruning schemes in terms of testing accuracy.  Compared with the no-pruning scheme, the optimal pruning scheme improves testing accuracy by approximately 12\% and reduces the average model transmission time per client by about 27.8\%. This is because the optimal pruning strategy trims the model based on the maximum number of parameters that can be transmitted along the model's communication path, allowing more parameters to be transmitted within the time constraint and effectively improving model performance. In contrast, a high retention rate or no pruning requires transmitting more model parameters, which can easily lead to communication timeouts and missed model aggregation, while a too low retention rate introduces excessive information loss, degrading model performance.

Figs.~\ref{fig.B}(a)--(b) and ~\ref{fig.B}(c)--(d) illustrate the testing accuracy and loss of the proposed joint pruning and routing method on F-CIFAR10 \& ResNet18 and F-CIFAR100 \& ResNet34, respectively, as well as the distribution of model retention rates and cumulative maximum link weights per hop across 20 clients under different network bandwidths (23~MHz, 30~MHz, and 35~MHz). The results show that, as bandwidth increases, testing loss decreases, accuracy improves, cumulative maximum link weights per hop reduce, and model retention rates rise. This indicates that higher bandwidth enhances model transmission efficiency, allowing clients to transmit more parameters under latency constraints, thereby increasing model capacity and expressiveness and ultimately improving the overall performance of the decentralized federated learning system.

Figs.~\ref{fig.FL_T}(a)--(b) and \ref{fig.FL_T}(c)--(d) illustrate the testing accuracy and loss of the proposed joint pruning and routing method on F-CIFAR10 \& ResNet18 and F-CIFAR100 \& ResNet34, respectively, as well as the distribution of model retention rates and cumulative maximum link weights per hop across 20 clients under different latency thresholds. As the latency threshold increases (1~s, 2~s, and 3~s), testing loss decreases, accuracy improves, and client model retention rates rise, while the cumulative maximum link weights per hop remain unchanged. This is because a larger latency threshold allows clients to select higher model retention rates along the given paths, retaining more parameters to enhance performance. Since the model transmission paths and bandwidth remain fixed, the cumulative maximum link weights per hop stay constant.

Figs.~\ref{fig:FL_IID}(a)--(c) compare the performance of different pruning strategies on the \textbf{F-CIFAR10 \& ResNet-18} task under \textbf{non-IID data distribution} and \textbf{IID data distribution with varying client scales ($N=20$ and $N=50$)}. The results show that the proposed \textbf{optimal pruning strategy} consistently outperforms the baseline methods in terms of both test accuracy and loss convergence. Even under highly heterogeneous data distributions and large-scale network settings, the proposed approach maintains superior robustness and scalability in heterogeneous federated learning environments.

\begin{figure}
      \centering
\begin{subfigure}{0.15\textwidth}
        \centering
        \includegraphics[width=\linewidth]{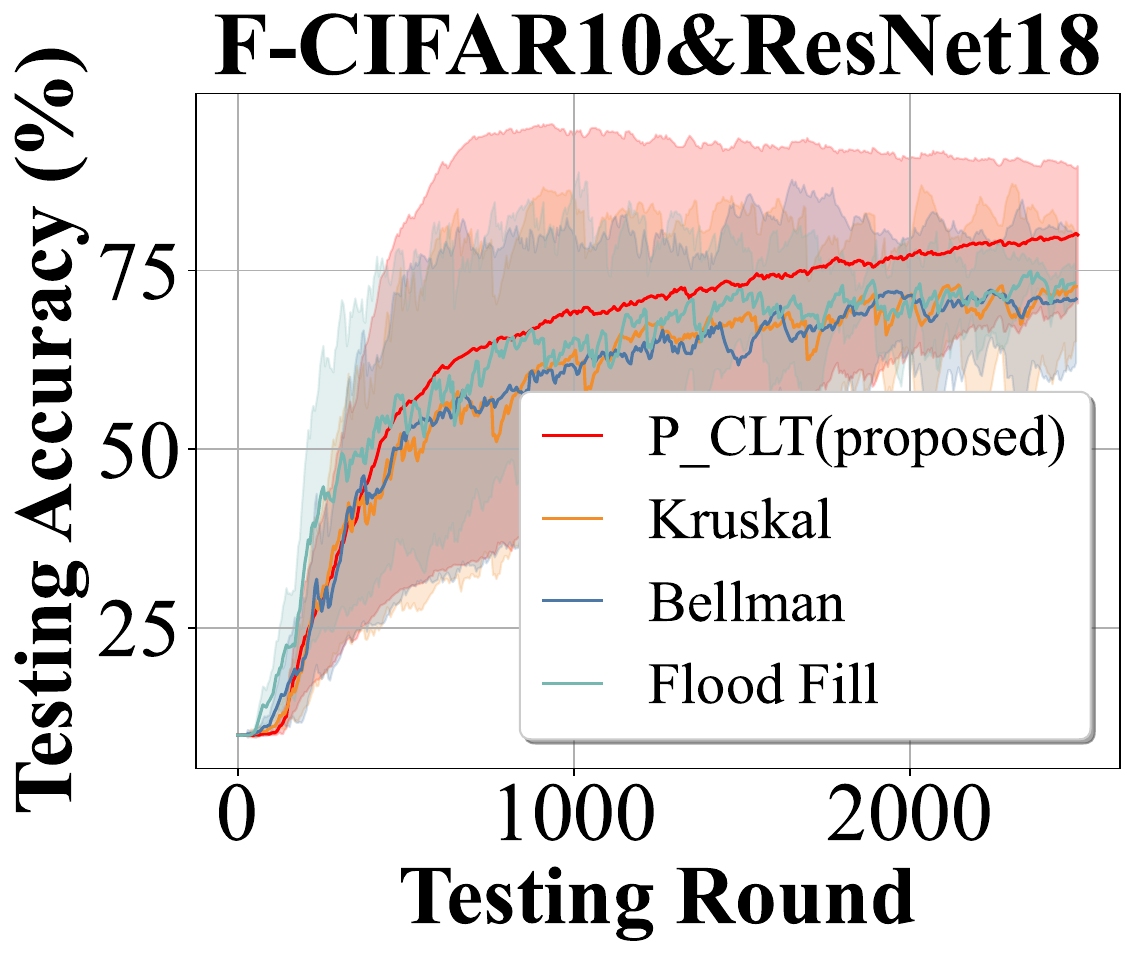}
        \caption{}
    \end{subfigure}%
    \begin{subfigure}{0.15\textwidth}
        \centering
        \includegraphics[width=\linewidth]{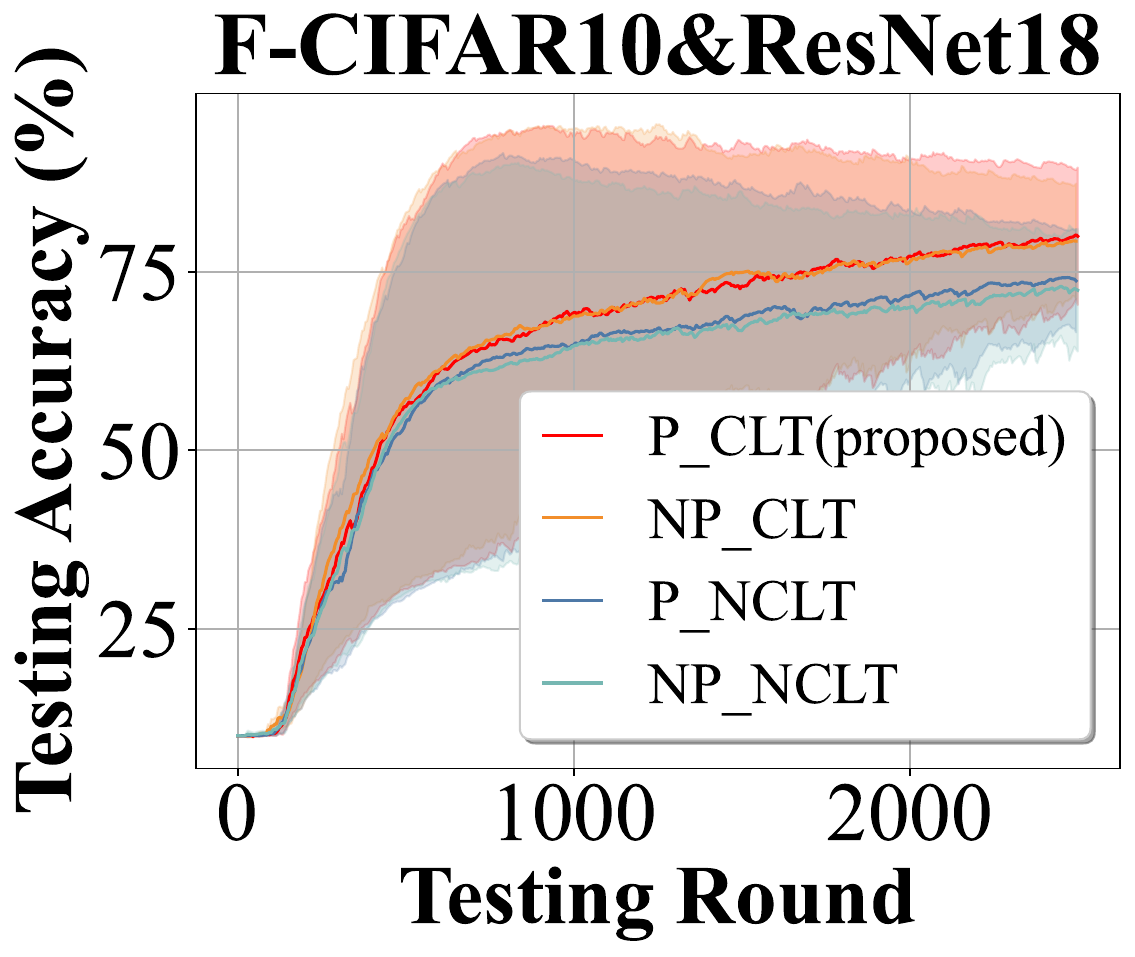}
        \caption{}
    \end{subfigure}%
    \begin{subfigure}{0.15\textwidth}
        \centering
        \includegraphics[width=\linewidth]{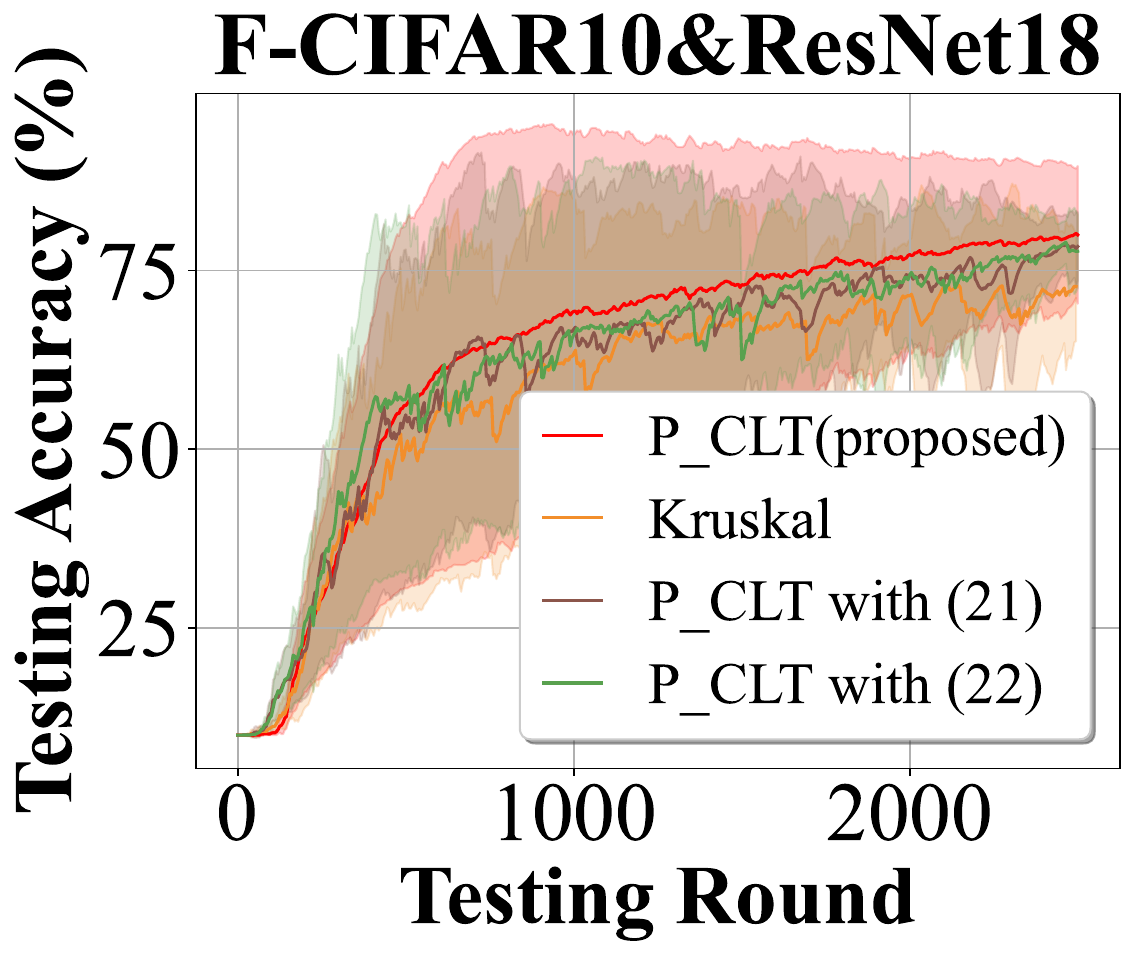}
        \caption{}
    \end{subfigure}

\begin{subfigure}[h]{0.5\textwidth}
    \includegraphics[width=\linewidth,valign=t]{"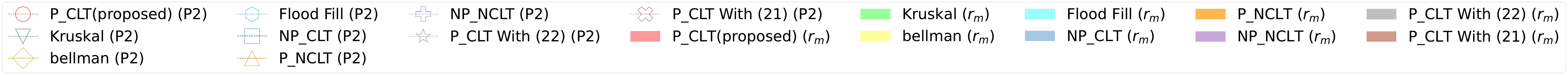"}
   
\end{subfigure}\\

\begin{subfigure}[h]{0.5\textwidth}
    \includegraphics[width=\linewidth,valign=t]{"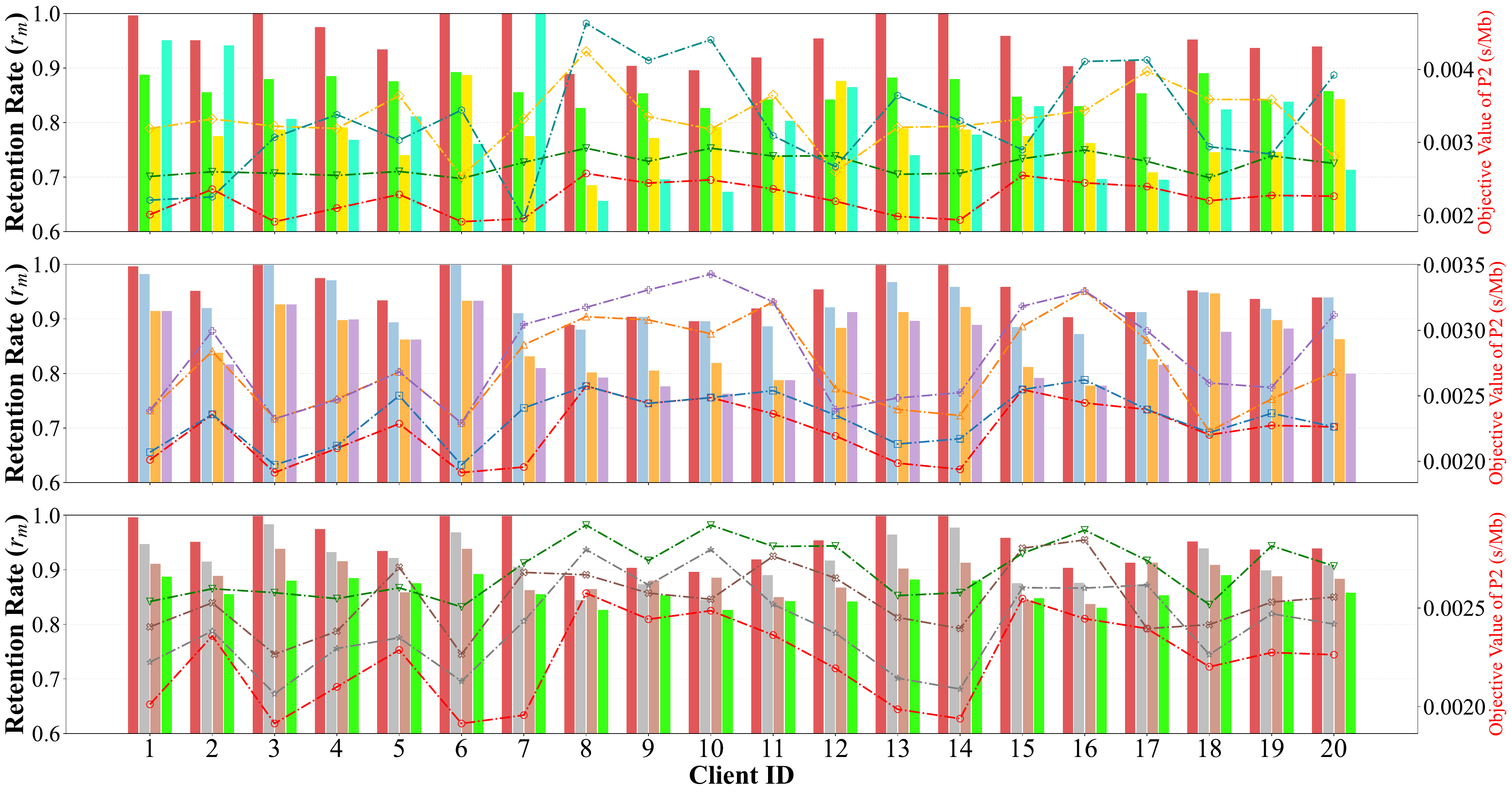"}
        \caption{}
\end{subfigure}\\

\caption{(a) Comparison of the proposed routing algorithm and traditional algorithms in testing accuracy on F-CIFAR10 \& ResNet18; (b)-(c) Comparison with benchmark methods in testing accuracy; (d) Distribution of model retention rates and cumulative maximum link weights per hop for P\_CLT and benchmark methods across 20 clients.
}\label{fig:overall_rm}
\end{figure}

\subsection{Evaluation of the Proposed Routing Algorithm}
 Figs.~\ref{fig:overall_rm}(a)--(c) illustrate the testing accuracy of various routing algorithms evaluated on the F-CIFAR10 \& ResNet-18 task. Fig.~\ref{fig:overall_rm}(d) unveil the intrinsic correlation between the model retention rate and the multi-hop path structure on the F-CIFAR10 \& ResNet-18 task: the lower the cumulative maximum link weight per hop along a multi-hop path, the more model parameters a client can retain and successfully upload.  The results indicate that the proposed P\_CLT algorithm outperforms other routing algorithms in both testing accuracy and model retention rate, achieving an improvement of approximately 8\% in testing accuracy compared with traditional methods.

This improvement is primarily attributed to the algorithm's ability to dynamically adjust routing paths while integrating model pruning strategies, minimizing the cumulative maximum link weight per hop in multi-hop communication, and thereby transmitting more model parameters under given time and bandwidth constraints. In contrast, the Kruskal algorithm enforces a tree topology that limits parallel transmissions and increases hop counts; the Bellman-Ford algorithm lacks global optimization capability, making it difficult to reduce the cumulative maximum link weight effectively; and the Flood Fill algorithm cannot mitigate bottlenecks caused by high-weight links. Moreover, other benchmark routing algorithms that rely solely on a single link addition criterion struggle to achieve globally optimal paths, and the absence of a link-threshold or node-priority mechanism may lead to suboptimal link selection, further reducing communication performance.

\begin{figure}[h]

    \centering
    \begin{subfigure}[b]{0.32\linewidth}  
        \centering
        \includegraphics[width=\linewidth]{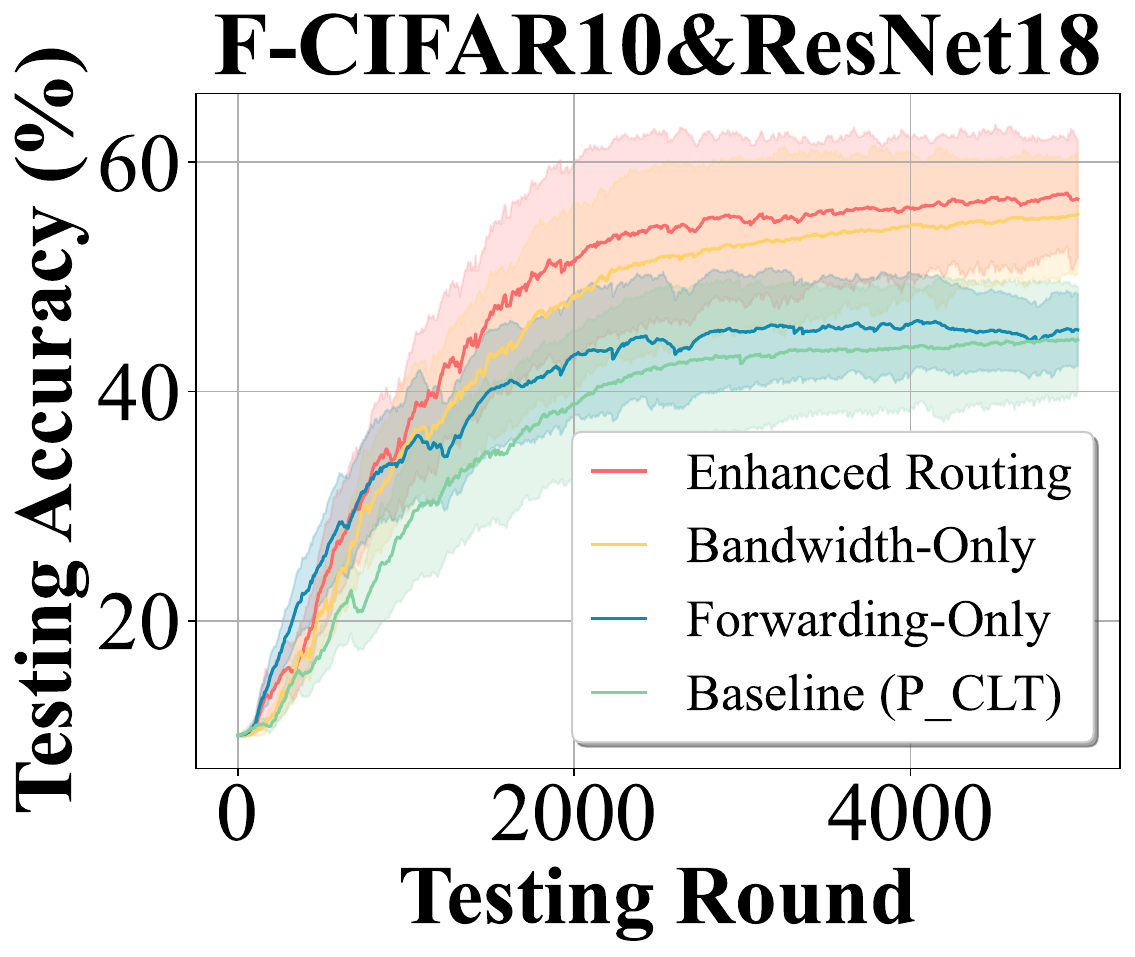}
        \caption{}
        \label{fig:sub1}
    \end{subfigure}
    \hfill
    \begin{subfigure}[b]{0.66\linewidth}
        \centering
        \includegraphics[width=\linewidth]{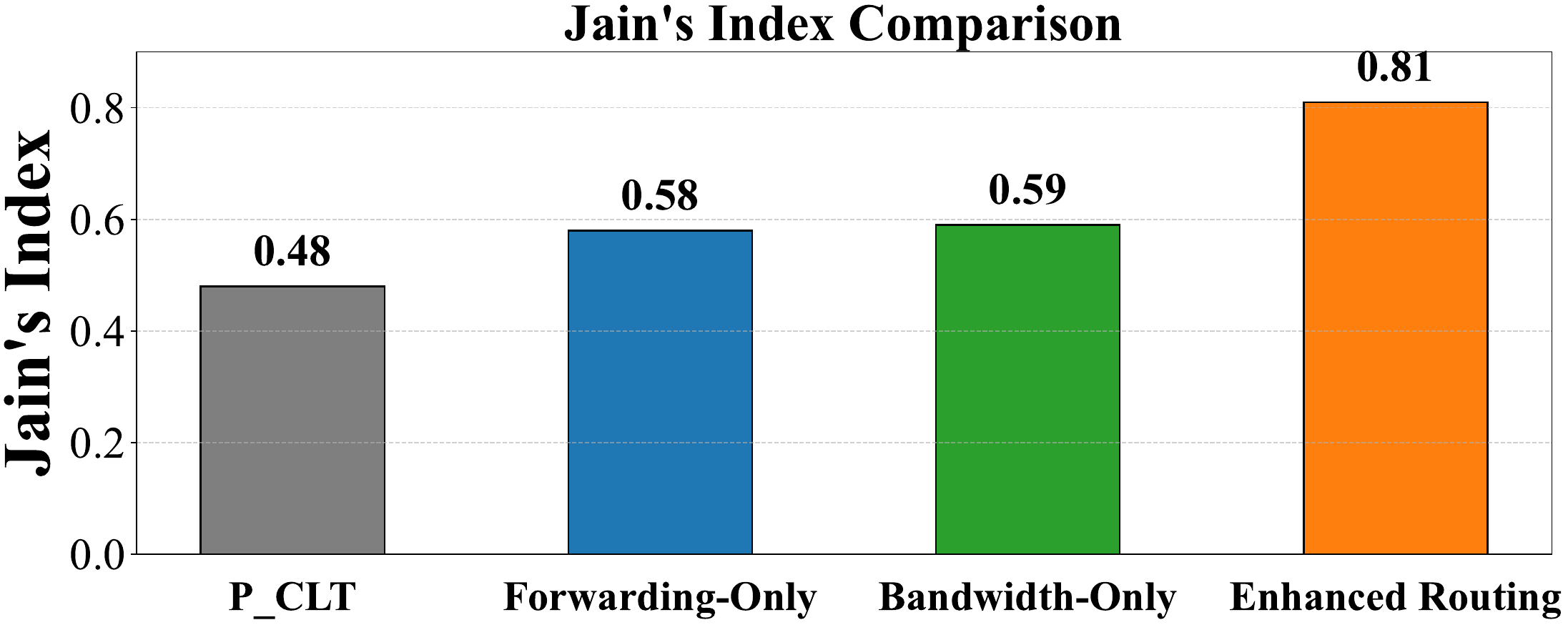}
        \caption{}
        \label{fig:sub2}
    \end{subfigure}
    \hfill

\caption{Performance comparison between the proposed enhanced routing method and the baseline under communication bottlenecks: (a) Testing accuracy on F-CIFAR10 \& ResNet-18; (b) Comparison of fairness metrics.}
    \label{fig:FL_CONF}
\end{figure}


Figs.~\ref{fig:FL_CONF}(a) and (b) illustrate the testing accuracy and Jain's fairness index for the F-CIFAR10 \& ResNet-18 task under representative communication bottlenecks. The results demonstrate that the proposed \textbf{enhanced routing} strategy consistently outperforms the baselines in both performance and fairness metrics. This improvement is primarily due to the fact that network congestion and stochastic node failures often cause transmission interruptions, which severely hinder the synchronization of model updates and lead to biased aggregation. By integrating the CAM and FPSR into the P\_CLT framework (see Section~\ref{sec:Bottlenecks}), our approach proactively alleviates congestion and dynamically bypasses failure nodes. This not only accelerates global convergence but also significantly enhances the operational fairness and overall stability of the D-FL system.

In the above plots, each solid line represents the average testing accuracy across all clients, while the shaded area reflects the variation in testing accuracy among clients. It is evident that, as training converges, the testing accuracy across clients becomes increasingly consistent, as indicated by the progressively shrinking shaded area.

\section{Conclusion and Future Work}\label{sec:conclusion}  This paper has proposed a joint model pruning and routing framework for D-FL to ensure timely model transmission and improve system convergence. We have derived the convergence upper bound of D-FL under joint routing and model pruning, and found that model retention plays a key role in enhancing convergence. Further analysis indicates that model retention is closely linked to clients' transmission paths, based on which we have designed an efficient routing optimization algorithm. Simulation results show that, compared with unpruned systems, the proposed joint routing and pruning method reduces average model transmission latency by about 27.8\% and improves testing accuracy by around 12\%. In addition, the proposed routing algorithm further increases testing accuracy by roughly 8\% compared with conventional and benchmark methods.

Future work will study deployment challenges in realistic heterogeneous scenarios, focusing on non-ideal channels and device heterogeneity. We plan to use non-cooperative game theory to design a multi-objective optimization framework balancing transmission reliability, computational cost, and system performance, aiming to improve convergence efficiency and robustness for practical distributed deployment.

\appendix
\subsection{Proof of \textbf{Lemma \ref{lemma1}}}\label{proof:lemma1}
We expand $\Vert\boldsymbol{\bar{\omega}}_{\alpha}\!\!-\!\!\boldsymbol{\omega}^*\Vert^2$, and then apply the inequality of arithmetic and geometric means, yielding
\begin{align}\label{delta_r1}\small
\Vert\boldsymbol{\bar{\omega}}_{\alpha}\!\!-\!\!\boldsymbol{\omega}^*\Vert^2
\leq&(1\!\!+\!\!\frac{1}{\tau_{\varrho}})\left\Vert \boldsymbol{\bar{\omega}}_{\alpha}\!\!-\!\!\boldsymbol{\bar{\varphi}}_{\alpha}\right\Vert ^{2}\!\!+\!\!
(1\!\!+\!\!\tau_{\varrho})\left\Vert \boldsymbol{\bar{\varphi}}_{\alpha}\!\!-\!\!\boldsymbol{\omega}^{*}\right\Vert ^{2},
\end{align} 
When the client's model transmission path is optimal, most of the model parameters are transmitted within the given time constraint, and subsequently \(\tau_{\rho} \to 0\).

Define the globally aggregated model at the $\alpha$-th training round as
\begin{align}\label{yr1}
\boldsymbol{\bar{\varphi}}_{\alpha}={\sum}_{ n \in \mathcal{V}}p_n\left(\boldsymbol{\bar{\omega}}_{\alpha\!-\!1}-\eta\nabla F_{n}(\boldsymbol{\bar{\omega}}_{\alpha\!-\!1},\xi)\right).
\end{align}
With \eqref{eq:epoch_imperfect}, \eqref{accurate_model} and \eqref{yr1}, ${\left\Vert \bar{\boldsymbol{\omega}}_{\alpha}-\boldsymbol{\bar{\varphi}}_{\alpha}\right\Vert ^{2}}$ is upper bounded by%
 \begin{subequations}\label{a}%
  \footnotesize \begin{align}	  
&\!\!\!\!\!\big\Vert\bar{\boldsymbol{\omega}}_{\alpha}\!-\!\boldsymbol{\bar{\varphi}}_{\alpha}\!\big\Vert^{2}\!\!\!\!=\!\!\Big\Vert\!{\sum}_{ n \in \mathcal{V}}\!p_{n}\!\left(\hat{\boldsymbol{\omega}}_{\alpha\!-\!1,n}\!\!-\!\!\eta\nabla F_{n}(\hat{\boldsymbol{\omega}}_{\alpha\!-\!1,n},\xi)\!\!-\!\!\bar{\boldsymbol{\omega}}_{\alpha\!-\!1}\!\!+\!\!\eta\nabla F_{n}(\boldsymbol{\bar{\omega}}_{\alpha\!-\!1},\xi)\!\right)\!\!\Big\Vert^{2}\notag\\
&\!\!\!\leq\!\!(1\!\!+\!\!\eta L)\!\Big\Vert{\sum}_{n\in\mathcal{V}}p_{n}(\hat{\boldsymbol{\omega}}_{\alpha\!-\!1,n}\!-\!\!\bar{\boldsymbol{\omega}}_{\alpha-1})\Big\Vert^{2}\!\!\\
&+\!\!(1\!\!+\!\!\frac{1}{\eta L})\!\Big\Vert\!{\sum}_{n\in\mathcal{V}}p_{n}\eta\big(\nabla F_{n}(\hat{\boldsymbol{\omega}}_{\alpha\!-\!1,n},\xi)\!\!-\!\!\nabla F_{n}(\boldsymbol{\bar{\omega}}_{\alpha\!-\!1},\xi)\big)\!\Big\Vert^{2}\label{a_bound:c}\\
&\!\!\leq\!(1\!\!+\!\!\eta L)\!\Big[\Big\Vert\!\!{\sum}_{n\in\mathcal{V}}\!p_{n}\hat{\boldsymbol{\omega}}_{\alpha\!-\!1,n}\!-\!\bar{\boldsymbol{\omega}}_{\alpha\!-\!1}\!\Big\Vert^{2}\!\!\!+\!\!\eta L\!\!{\sum}_{n\in\mathcal{V}}\!p_{n}\!\Vert\hat{\boldsymbol{\omega}}_{\alpha\!-\!1,n}\!\!-\!\!\bar{\boldsymbol{\omega}}_{\alpha\!-\!1}\Vert\!^{2}\Big]\label{a_bound:eqb}\\
&\!\!\!\leq\!\!(1\!\!+\!\!\eta L)\!\Big[\!\!\left({\sum}_{n'\in \mathcal{V}}p_{n'}^{2}+\!\!\eta L p_{\max}\right)\!{\sum}_{n\in\mathcal{V}}\!\Vert\hat{\boldsymbol{\omega}}_{\alpha\!-\!1,n}\!\!-\!\!\bar{\boldsymbol{\omega}}_{\alpha\!-\!1}\Vert\!^{2}\Big].\label{a_bound:eqc}
		\end{align}%
  \end{subequations}
Here, \eqref{a_bound:c} uses $\Vert\mathbf{a}+\mathbf{b}\Vert^{2}\leq(1+\eta L)\Vert\mathbf{a}\Vert^{2}+(1+\frac{1}{\eta L})\Vert\mathbf{b}\Vert^{2}$. \eqref{a_bound:eqb} is obtained due to the $L$-smoothness of $F_n$ and Cauchy's inequality with ${\sum}_{ n\in\mathcal{V}}p_{n}=1$. \eqref{a_bound:eqc} is based on $p_n\leq p_{\max}$ with $p_{\max}=\underset{\forall n}{\max}\; p_n$, and the fact that
$\!\big\Vert\!\sum_{n\in\mathcal{V}}\!p_{n}\hat{\boldsymbol{\omega}}_{\alpha\!-\!1,n}\!\!-\!\!\boldsymbol{\bar{\omega}}_{\alpha\!-\!1}\!\big\Vert^{2}\!\le   \sum_{n'\in \mathcal{V} }p_{n'}^{2}\sum_{n\in \mathcal{V}}\big\Vert\!\hat{\boldsymbol{\omega}}_{\alpha\!-\!1,n}\!\!-\!\!\boldsymbol{\bar{\omega}}_{\alpha\!-\!1}\!\big\Vert^{2}$.

By substituting \eqref{yr1}, $\left\Vert \boldsymbol{\bar{\varphi}}_{\alpha}-\boldsymbol{\omega}^{*}\right\Vert^2$ is upper bounded by
\begin{subequations}\label{b}\small
\begin{align}
	&\left\Vert \boldsymbol{\bar{\varphi}}_{\alpha} - \boldsymbol{\omega}^{*} \right\Vert^2
	= \Big\Vert \sum_{n \in \mathcal{V}} p_n \left( \boldsymbol{\bar{\omega}}_{\alpha-1} - \eta \nabla F_n(\boldsymbol{\bar{\omega}}_{\alpha-1}, \xi) \right) - \boldsymbol{\omega}^{*} \Big\Vert^2 \notag \\
	&= \sum_{n \in \mathcal{V}} p_n \left\Vert \boldsymbol{\bar{\omega}}_{\alpha-1} - \boldsymbol{\omega}^{*} \right\Vert^2 + \Big\Vert \eta \sum_{n \in \mathcal{V}} p_n \nabla F_n(\boldsymbol{\bar{\omega}}_{\alpha-1}, \xi) \Big\Vert^2 \!\\
    &- \!2 \Big\langle \sum_{n \in \mathcal{V}} p_n (\boldsymbol{\bar{\omega}}_{\alpha-1} \!- \!\boldsymbol{\omega}^{*}), \eta \sum_{n \in \mathcal{V}} p_n \nabla F_n(\boldsymbol{\bar{\omega}}_{\alpha-1}, \xi) \Big\rangle \label{b_bound:eqb} \\
	&\leq \left( 1 - 2\mu\eta + L^2\eta^2 \right) \sum_{n \in \mathcal{V}} p_n \left\Vert \boldsymbol{\bar{\omega}}_{\alpha-1} - \boldsymbol{\omega}^{*} \right\Vert^2 \label{b_bound:eqc} \\
	&= \left( 1 - 2\mu\eta + L^2\eta^2 \right) \left\Vert \boldsymbol{\bar{\omega}}_{\alpha-1} - \boldsymbol{\omega}^{*} \right\Vert^2, \label{b_bound:eqd}
\end{align}
\end{subequations}
where 
\(\eqref{b_bound:eqb}\) uses \(\left\Vert a - b \right\Vert^2 = \left\Vert a \right\Vert^2 + \left\Vert b \right\Vert^2 - 2 \langle a, b \rangle\); \eqref{b_bound:eqc} is due to the $L$-smoothness of $F(\cdot)$, i.e., $\left\Vert\sum_{n \in \mathcal{V}} p_n\nabla F_{n}(\boldsymbol{\bar{\omega}}_{\alpha-1},\xi)\right\Vert^{2} \leq L\sum_{n \in \mathcal{V}} p_n\left\Vert  \boldsymbol{\bar{\omega}}_{\alpha-1} - \boldsymbol{\omega}^{*}\right\Vert^{2}$, and the $\mu$-strong convexity, i.e., $\!\!\!\!\!\left\langle  \boldsymbol{\bar{\omega}}_{\alpha-1}-\boldsymbol{\omega}^{*}, \sum_{n \in \mathcal{V}} p_n \nabla F_{n}(\boldsymbol{\bar{\omega}}_{\alpha-1},\xi)\right\rangle \geq \mu \sum_{n \in \mathcal{V}} p_n \left\Vert \boldsymbol{\bar{\omega}}_{\alpha-1} - \boldsymbol{\omega}^{*}\right\Vert^2$; \(\eqref{b_bound:eqd}\) is based on \(\sum_{n \in \mathcal{V}} p_n = 1\).
Substituting \eqref{a} and \eqref{b} into \eqref{delta_r1}, we obtain (\ref{delta_1}).

\subsection{Proof of \textbf{Lemma \ref{lemma2}}}\label{proof_lemma2}

Let $\mathcal{V}_k$ denote the set of clients for which the \( k \)-th model parameter of $\boldsymbol{\hat{\omega}}_{\alpha ,n}$ is transmitted. $\mathcal{V}_{k,n}=\mathcal{V}_{k}\cup\{n\}, n \in \mathcal{V}$. The upper bound of $\sum_{k=1}^{K}\sum_{m\in\mathcal{V}}\lambda_{\alpha-1,(m,n),k}^{2}$ is given by
\begin{subequations} \footnotesize 
    \begin{align}
&\!\sum_{k=1}^{K}\!\sum_{m\in\mathcal{V}}\!\lambda_{\alpha-1,(m,n),k}^{2}\!=\!\sum_{k=1}^{K}\!\Big(\big(\!\!\sum_{m\in\mathcal{V}_{k,n}}\!\!\!\!p_{m}^{2}\!\big)\Big(\frac{\sum_{n'\in\mathcal{V}-\mathcal{V}_{k,n}}p_{n'}}{\sum_{m'\in\mathcal{V}_{k,n}}p_{m'}}\Big)^{2}\!\!+\!\!\!\!\!
\sum_{l\in\mathcal{V}-\mathcal{V}_{k,n}}\!\!\!p_{l}^{2}\Big)
    \label{eq:lambda_bound_c}\\&\!\!\leq\! \sum_{k\!=\!1}^{K}\Big(\big({\sum_{m\in\mathcal{V}\!\!-\!\mathcal{V}_{k,n}}\!\!p_{m}}\big)^{2}\!\!+\!\!\!\sum_{l\!\in\!\mathcal{V}\!-\!\!\mathcal{V}_{k,n}}\!\!p_{l}^{2}\Big)\label{eq:lambda_bound_d}\\
    &\leq \sum_{k=1}^{K}\big(1+\sum_{m\in\mathcal{V}-\mathcal{V}_{k}}1\big)\sum_{l\in\mathcal{V}-\mathcal{V}_{k,n}}p_{l}^{2}\label{eq:lambda_bound_e}\\
    &\!= \!\!\sum_{k\!=\!1}^{K}\big(1\!\!+\!\sum_{m\in\mathcal{V}}(1\!\!-\!{e}_{\alpha,(m,n),k}))\big)\sum_{l\in\mathcal{V}}(1\!\!-\!{e}_{\alpha,(l,n),k})\, p_{l}^{2}\label{eq:lambda_bound_f}\\
    &\!\leq\! \sum_{k=1}^{K}\sum_{l\in\mathcal{V}}(1\!\!-\!{e}_{\alpha,(l,n),k})\, p_{l}^{2}\!+\!\Big(\sum_{k=1}^{K}\sum_{m\in\mathcal{V}}(1\!\!-\!{e}_{\alpha,(m,n),k})\Big)^2\label{eq:lambda_bound_h}  \\
    &\!=\!\sum_{l\in\mathcal{V}\setminus\{n\}} (K\!-\!\lfloor r_l K\rfloor) p_l^2 \!+\!(\sum_{l\in\mathcal{V}\setminus\{n\}} (K\!-\!\lfloor r_l K\rfloor) )^2,\label{eq:lambda_bound_j}
\end{align}
\end{subequations}
where 
\eqref{eq:lambda_bound_c} expands $\lambda_{\alpha-1,(m,n),k}^{2}$; 
\eqref{eq:lambda_bound_d} is based on $\big(\sum_{m\in\mathcal{V}_{k,n}}p_{m}^{2}\big)\big/\big(\sum_{m'\in\mathcal{V}_{k,n}}p_{m'}\big)^2\leq1$;
\eqref{eq:lambda_bound_e} is due to $\big(\sum_{m\in\mathcal{V}-\mathcal{V}_{k}}p_{m}\big)^{2}\leq\left(\sum_{m\in\mathcal{V}-\mathcal{V}_{k}}p_{m}^{2}\right)\left(\sum_{m\in\mathcal{V}-\mathcal{V}_{k}}1\right)$; 
\eqref{eq:lambda_bound_f} is based on $\sum_{m \in \mathcal{V}_{k}}1 +\sum_{m \in \mathcal{V}-\mathcal{V}_{k}}1 =\sum_{m \in \mathcal{V}}1$, with $\sum_{m \in \mathcal{V}_{k}}1=\sum_{m \in \mathcal{V}}{e}_{\alpha,(m,n),k}$; 
\eqref{eq:lambda_bound_h} is due to the fact that $0\leq p_l \leq 1$ and Jensen's inequality;
\eqref{eq:lambda_bound_j} stems from \eqref{eq:prune_indicator}.

\subsection{Derivation of Sorting Complexity} \label{appendix:complexity_proof}

This section provides the detailed derivation of the sorting complexity estimation for all nodes.
Let the number of nodes be \(N\), the number of edges be \(E\), and the degree of node \(i\) be \(d_i\). The sorting complexity of node \(i\)'s adjacent links is \(\mathcal{O}(d_i \log d_i)\), so the total sorting complexity for all nodes is $\sum_{i=1}^N d_i \log d_i$.
Consider \(f(x) = x \log x\), which is convex for \(x > 0\). By Jensen’s inequality, we have $\frac{1}{N} \sum_{i=1}^N f(d_i) \geq f\left( \frac{1}{N} \sum_{i=1}^N d_i \right).$ Multiplying both sides by \(N\) yields $\sum_{i=1}^N d_i \log d_i \geq E \log \frac{E}{N}$, where \(E = \sum_{i=1}^N d_i\). This provides a lower bound on the total complexity, which is widely used as a good upper-bound estimate of the sorting cost in practice. Thus, the total complexity is $\mathcal{O}\left(  E \log \frac{E}{N}\right).$

{\scriptsize
\bibliographystyle{IEEEtran}
\bibliography{citations}

\begin{IEEEbiography}[{\includegraphics[width=1in,height=1.25in,clip,keepaspectratio]{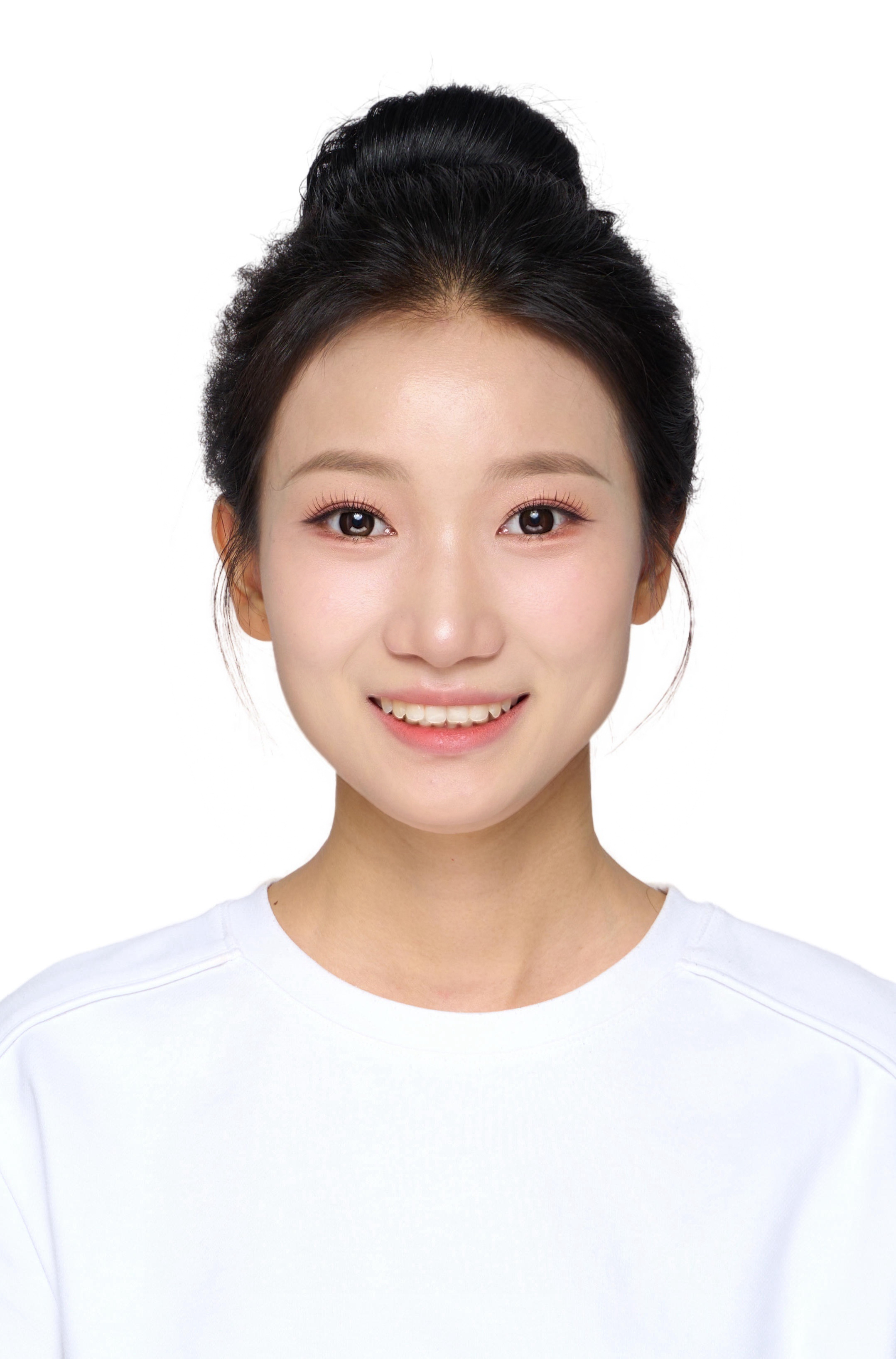}}]{Xiaoyu He} is currently pursuing the Ph.D. degree in Information and Communication Engineering at Beijing University of Posts and Telecommunications (BUPT), Beijing, China. Her research interests include decentralized federated learning (D-FL), adaptive model pruning, reinforcement learning, non-cooperative game theory, personalization, multi-hop routing optimization, and resource-constrained wireless networks.

\end{IEEEbiography}

\begin{IEEEbiography}[{\includegraphics[width=1in,height=1.25in,clip,keepaspectratio]{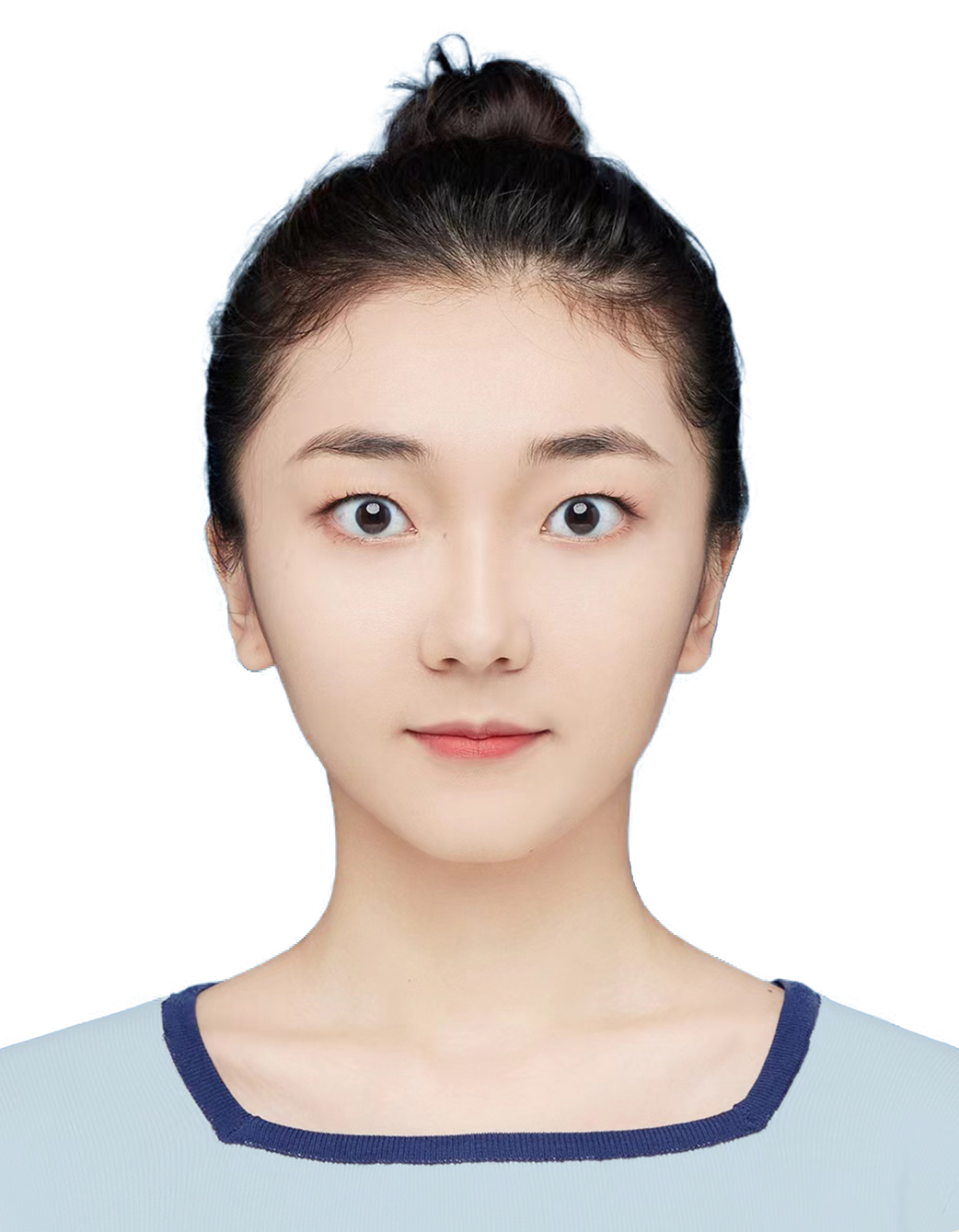}}]{Weicai Li} 
 received the B.E. and Ph.D. degrees from the School of Information and Communication Engineering, Beijing University of Posts and Telecommunications (BUPT), China, in 2020 and 2025, respectively. From December 2022 to December 2023, she was a Visiting Student at the University of Technology Sydney, Australia. She is currently with the Center for Target Cognition Information Processing Science and Technology and the Key Laboratory of Modern Measurement and Control Technology, Ministry of Education, both at Beijing Information Science and Technology University, Beijing, China (e-mail: liweicai@bupt.edu.cn). Her research interests include integrated sensing and communications, wireless federated learning, distributed computing, and privacy preservation.
\end{IEEEbiography}

\begin{IEEEbiography}[{\includegraphics[width=1in,height=1.25in,clip,keepaspectratio]{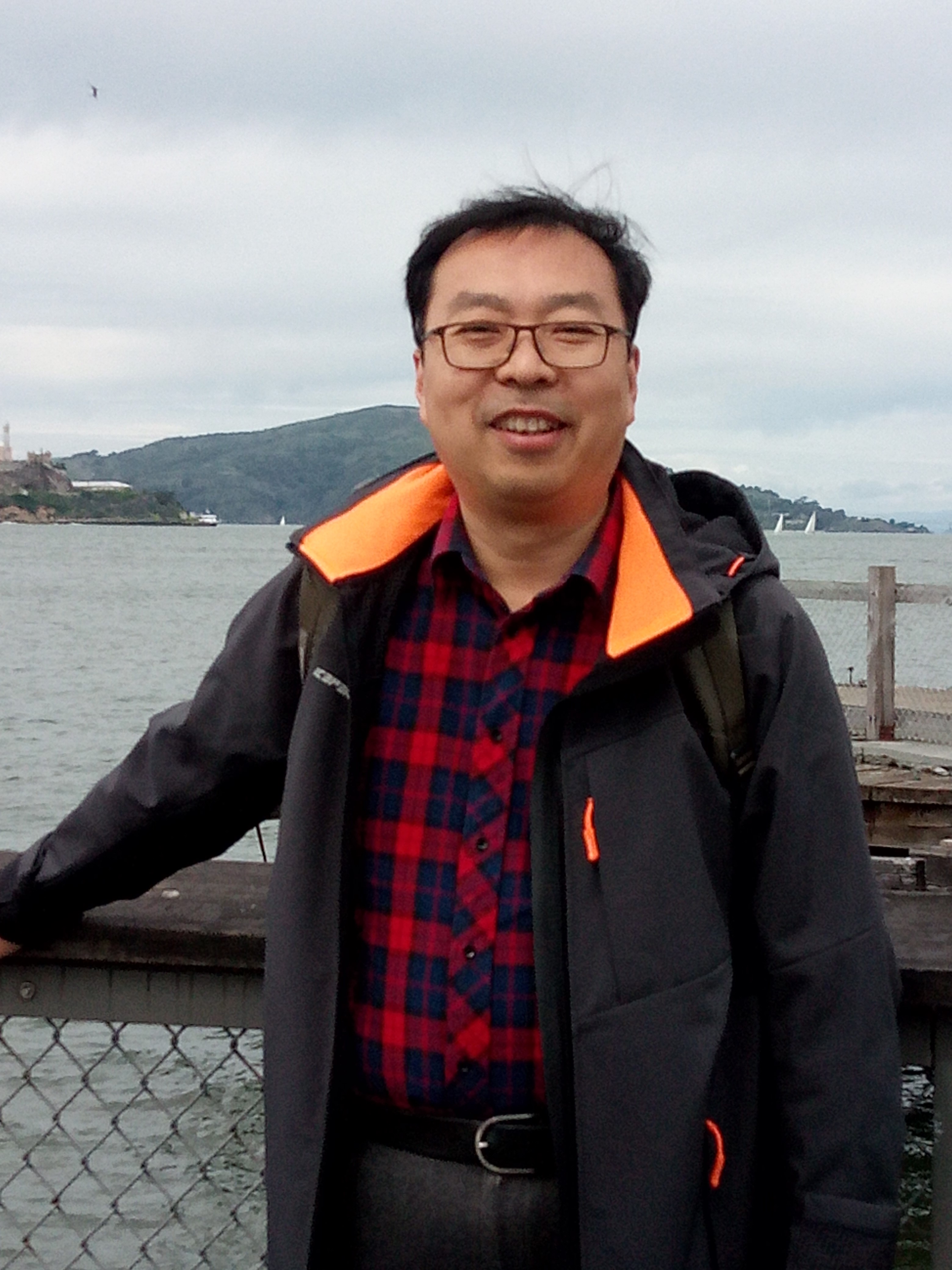}}]{Tiejun Lv} (Senior Member, IEEE) received the M.S. and Ph.D. degrees in electronic engineering from the University of Electronic Science and Technology of China (UESTC), Chengdu, China, in 1997 and 2000, respectively. From January 2001 to January 2003, he was a Post-Doctoral Fellow at Tsinghua University, Beijing, China. In 2005, he was promoted to a Full Professor at the School of Information and Communication Engineering, Beijing University of Posts and Telecommunications (BUPT). From September 2008 to March 2009, he was a Visiting Professor with the Department of Electrical Engineering, Stanford University, Stanford, CA, USA. He is currently the author of four books, one book chapter, and more than 170 published journal articles and 200 conference papers on the physical layer of wireless mobile communications. His current research interests include signal processing, communications theory, and networking. He was a recipient of the Program for New Century Excellent Talents in University Award from the Ministry of Education, China, in 2006. He received the Nature Science Award from the Ministry of Education of China for the hierarchical cooperative communication theory and technologies in 2015 and Shaanxi Higher Education Institutions Outstanding Scientific Research Achievement Award in 2025.

\end{IEEEbiography}

\begin{IEEEbiography}[{\includegraphics[width=1in,height=1.25in,clip,keepaspectratio]{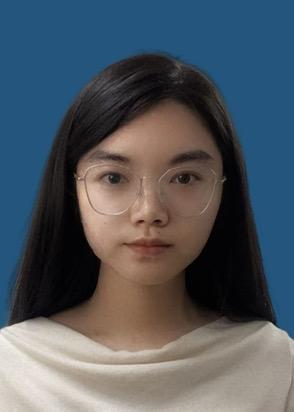}}]{Xi Yu}(Graduate Student Member, IEEE) 
received the B.E. degree in communication engineering from Beijing University of Posts and Telecommunications (BUPT), China, in 2020. She is pursuing her Ph.D. with the School of Information and Communication Engineering at BUPT. Her research interests include multi-task semantic communication and privacy-preserving techniques.

\end{IEEEbiography}


}

\end{document}